\documentclass[10pt,onecolumn,letterpaper]{article}
\pdfoutput=1
\usepackage{cvpr}
\usepackage{times}
\usepackage{subfigure}
\usepackage{epsfig}
\usepackage{graphicx}
\usepackage{amsthm}
\usepackage{amsmath}
\usepackage{amssymb}
\usepackage{algorithm,algpseudocode}

\newtheorem{theorem}{Theorem}
\newtheorem{remark}{Remark}
\newtheorem{lemma}{Lemma}
\newtheorem{definition}{Definition}
\newtheorem{assumption}{Assumption}
\newtheorem{corollary}{Corollary}

\usepackage{datetime}

\usepackage{amsthm}

\usepackage{appendix}

\cvprfinalcopy 
\def\cvprPaperID{****} 
\allowdisplaybreaks

\usepackage{authblk}

\usepackage{datetime}
\newdate{date}{25}{10}{2017}

\begin{document}
\pdfoutput=1	
	\title{Duality-free  Methods for Stochastic Composition  Optimization}
	
\author[1]{ Liu Liu\thanks{lliu8101@uni.sydney.edu.au}}
\author[2]{ Ji Liu\thanks{ji.liu.uwisc@gmail.com}}
\author[1]{ Dacheng Tao\thanks{dacheng.tao@sydney.edu.au}}
\affil[1]{UBTECH Sydney AI Centre and SIT, FEIT, The University of Sydney}
\affil[2]{Department of Computer Science, University of Rochester}

	\maketitle

	\begin{abstract}
		We consider the composition optimization  with two expected-value functions in the form of $\frac{1}{n}\sum\nolimits_{i = 1}^n F_i(\frac{1}{m}\sum\nolimits_{j = 1}^m G_j(x))+R(x)$, { which formulates many important problems in statistical learning and machine learning such as solving Bellman equations in reinforcement learning and nonlinear embedding}. Full Gradient or classical stochastic gradient descent based optimization algorithms are unsuitable or computationally expensive to solve this problem due to the inner expectation $\frac{1}{m}\sum\nolimits_{j = 1}^m G_j(x)$. We propose a duality-free based stochastic composition method that combines variance reduction methods to address the stochastic composition problem. We apply SVRG and SAGA based methods to estimate the inner function, and duality-free method to estimate the outer function. We prove the linear convergence rate not only for the convex composition problem, but also for the case that the individual outer functions are non-convex while the objective function is strongly-convex. We also provide the results of experiments that show the effectiveness of our proposed methods.	
	\end{abstract}
	\section{Introduction}
{ Many important machine learning and statistical learning problems can be formulated into the following composition minimization}:
\begin{align}\label{SCDF:ProblemMainCompositionminimization}
\mathop {\min }\limits_{x \in {\mathbb{R}^N}} \left\{ {P(x)\mathop  = \limits^{def} \frac{1}{n}\sum\limits_{i = 1}^n {{F_i}} (\frac{1}{m}\sum\limits_{j = 1}^m {{G_j}} (x)) + R(x)} \right\},
\end{align}
where each $F_i$: ${\mathbb{R}^M} \to \mathbb{R}$ is a smooth function, each $G_i$: ${\mathbb{R}^N} \to {\mathbb{R}^M}$ is a mapping function, and $R(x)$ is a proper and relatively simple convex function. We call $ G( x )$:$  =  {\frac{1}{m}\sum\nolimits_{j = 1}^m {{G_j}( x )} } $ the inner function, and  $F( {G( x )} )$:$ = \frac{1}{n}\sum\nolimits_{i = 1}^n {{F_i}(G(x))} $ the outer function. 
{ The composition optimization problem arises in large-scale machine learning and reinforcement learning tasks \cite{wang2017stochastic,dai2016learning}, such as solving Bellman equations in reinforcement learning \cite{sutton1998reinforcement}:}
{ \begin{align*}
	\mathop {\min }\limits_x {\| {\mathbb{E}[ B ]x - \mathbb{E}[ b ]} \|^2},
	\end{align*}
	where $\mathbb{E}[ B ] = I - \gamma {P^{\pi} }$, $\gamma\in(0,1)$ is a discount factor, $P^{\pi}$ is the transition probability,	$\mathbb{E}[ b ] = {r^{\pi} }$, and $r^{\pi}$ is the expected state transition reward.} Another example is the mean-variance in risk-averse learning:
\begin{align*}
\mathrm{min}_x\, \mathbb{E}_{a,b}[ h( {x;a,b} ) ] + \lambda \mathrm{Var}_{{a,b}}[ {h( {x;a,b} )} ],
\end{align*}
where $h(x;a,b)$ is the  loss function with random variables $a$ and $b$. $\lambda>0$ is a regularization parameter.

{ The commonly used gradient or stochastic gradient descent based optimization algorithms are unsuitable or too computationally expensive to solve this problem} due to the inner expectation ${\frac{1}{m}\sum\nolimits_{j = 1}^m {{G_j}(x)} }$.
Recently, \cite{wang2017stochastic} provided two plausible schemes for the composition problem. The first is based on the stochastic composition gradient method (SCGD), which adopt a quasi-gradient approach and sample method to approximate $G$ and  estimate the gradient of $F(G(x))$. The other is the Fenchel's transform approach, which is analogous to the stochastic primal-dual coordinate (SPDC) \cite{zhang2015stochastic} method. This approach is based on the primal-dual algorithm to  solve the convex-concave saddle problem, in which  problem (\ref{SCDF:ProblemMainCompositionminimization}) can be reformulated as  
\begin{align}\label{SCDF:ProblemConvexConcave}
\mathop {\min }\limits_x \mathop {\max }\limits_z \left\{ {\left\langle {z,G\left( x \right)} \right\rangle  - {F^*}\left( z \right) + R\left( x \right)} \right\},
\end{align}
where ${F^*}( z ) = \mathop {\max }\nolimits_{G( x )} \{ {\langle {z,G( x )} \rangle  - F( {G( x )} )} \}$.
However, the above { reformulation (\ref{SCDF:ProblemConvexConcave}) destroys the convexity of the original problem, since the reformulation does not necessarily result in a convex-concave structure even if the original problem is convex. This means that we lose global optimality.} Specifically, when using the cross-iteration method to minimize $\langle {z,G(x)} \rangle  + R(x)$ with respect to $x$ while fixing $z$, it may not converge to the optimal point since the subproblem is not necessarily convex. In such cases, the dual problem becomes meaningless.

In this paper, we propose the  stochastic composition duality-free (SCDF) method. The SCDF method belongs to the family of stochastic gradient descent (SGD) methods and, while based on the gradient estimation, is different to the vanilla SGD. 
Variance reduction method have become very popular for estimating the gradient and are investigated in stochastic variance reduction gradient (SVRG) \cite{johnson2013accelerating}, SAGA \cite{defazio2014saga}, stochastic dual coordinate ascent (SDCA) \cite{shalev2013stochastic} and duality-free SDCA \cite{shalev2016sdca}. However, these methods only consider  one finite-sum function. The  Composition-SVRG1 and Composition-SVRG2 \cite{lian2016finite} methods apply variance reduced technology to the two finite-sum functions that estimate the gradient of $(\partial G( x ))^\mathsf{T}\nabla F( G( x ) )$, the inner function $G$ and the corresponding partial derivative $\partial G$. However, SVRG-based { methods cannot} directly deal with the dual problem.  Here we design a new algorithm that not only disposes of the dual function, but also reduces the gradient variance. The main contributions of this paper are three-fold:
\begin{itemize}
	\item We apply the duality-free based method to the composition of two finite-sum functions. Even though the gradient estimation 
	$(\partial G( x ))^\mathsf{T}\nabla F( G( x ) )$ is biased using the SVRG-based method to estimate the inner function $G$, we obtain the linear convergence rate. 		 
	\item Besides the SVRG-based method to estimate the inner function $G$ and the  partial gradient $\partial G$, we also provide the SAGA-based method to estimate $G$ and $\partial G$ and provide the corresponding convergence analysis.
	\item Our proposed SCDF method  also deals with the scenario that the  individual function $F_i(\cdot)$ is non-convex but the function $F$ is strongly convex. We also proof the linear convergence rate for such case. 
\end{itemize}

\subsection{Related work}
Stochastic gradient methods have often been used to minimize the large-scale finite-sum problem.  However, stochastic gradient methods are unsuitable for the  family of nonlinear functions with two finite-sum structures.  \cite{wang2017stochastic} first proposed the first-order stochastic method SCGD  to solve  such  problems, which used two steps to alternately update the variable and inner function. SCGD achieved a convergence rate of $O(K^{-2/7})$ for the general function and $O(K^{-4/5})$ for the strongly convex function, where $K$ is the number of queries to the stochastic first-order oracle. Furthermore, in the special case that the inner function $G$ is a linear mapping, \cite{wang2016accelerating} also proposed an accelerated stochastic composition proximal gradient method with a convergence rate of  $O(K^{-1})$.

Recently, variance-reduced  stochastic gradient methods have attracted attention due to their fast convergence. \cite{roux2012stochastic} \cite{schmidt2017minimizing} proposed a stochastic average gradient method with a sublinear convergence rates. Two popular gradient estimator methods, SVRG \cite{johnson2013accelerating} and SAGA \cite{defazio2014saga}, were later introduced, both of which have linear convergence rates. \cite{xiao2014proximal} went on to introduce the proximal-SVRG method to the regularization problem and in doing so provided a more succinct convergence analysis. Other related SVRG-based or SGAG-based methods have also been proposed, including \cite{liu2017accelerated} who applied SVRG to the ADMM method. \cite{harikandeh2015stopwasting} reported  practical SVRG to improve the performance of SVRG , \cite{allen2016katyusha} introduced the  Katyusha  method to  accelerate the variance-reduction based algorithm,   and  \cite{allen2016improved} used the SVRG-based algorithm to explore the non-strongly convex objective and the sum-of-non-convex objective. Moreover,  \cite{lian2016finite} first applied the SVRG-based method to the stochastic composition optimization and obtained a linear convergence rate.

Dual stochastic and primal-dual stochastic methods have also been  proposed, and these also included  "variance reduction" procedure. SDCA \cite{shalev2013stochastic}   randomly selected the coordinate of the dual variable to maximize the dual function and  performed the update between the dual and primal variables. Accelerated SDCA \cite{shalev2014accelerated}  dealt with the ill-conditioned  problem by adding a quadratic term to the objective problem, such that it could be conducted on the modified strongly convex subproblem. Accelerated randomized proximal coordinate (APCG) \cite{lin2014accelerated} \cite{lin2014acceleratedSIAM} was also based on SDCA but used a  different accelerated method.  Duality-free SDCA  \cite{shalev2016sdca} exploited  the primal and dual variable relationship to approximately reduce the gradient variance. SPDC \cite{zhang2015stochastic} is based on the primal-dual algorithm, which alternately  updates the primal and dual variables. However, these methods can only be applied to the single  finite-sum structure problem.  \cite{dai2016learning} proposed the dual-based method for stochastic composition problem but with additional assumptions that limited the general composition function to two finite-sum structures. 

Finally, \cite{wang2016stochastic} considered  corrupted samples with Markov noise and proved that SCGD could almost always converge  to an optimal solution.  \cite{yu2017fast}  applied the  ADMM-based method to the stochastic composition optimization problem and provide an analysis of the convex function without requiring Lipschitz smoothness.
\section{Preliminaries}
In this paper, we denote the Euclidean norm with $\left\|  \cdot  \right\|$.  $i \in [ n ]$ and $j \in [ m ]$  denote that $i$ and $j$ are generated uniformly at random from $[ n ] = \{ {1,2,...,n} \}$ and $[ m ] = \{ {1,2,...,m} \}$.
$(\partial G( x ))^\mathsf{T}\nabla F( {G( x )} )$ denotes the full gradient of function $F( {G( x )} )$, where $\partial G$ is the partial gradient of $G$. We first revisit some basic definitions on conjugate, strongly convexity and smoothness, and then provide  assumptions about the composition of the two expected-value functions.
\begin{definition}\label{DefinitionConvex}
	For a function f: ${\mathbb{R}^M} \to \mathbb{R}$,
	\begin{itemize}
		\item $f^*$ is the conjugate of function $f(x)$ if  $\forall x,y \in {\mathbb{R}^M}$, it satisfies ${f^*}(y) = {\max _x}(\langle x,y\rangle  - f(x))$.
		\item $f$ is $\lambda$-strongly convex if  $\forall x,y \in {\mathbb{R}^M}$, it satisfies $f( x ) \ge f( y ) + \langle {\nabla f( y ),x - y} \rangle  + {\lambda }/{2}\| {x - y} \|^2$. For $\forall a\in [0,1]$, it also satisfies
		$
		f( {a x + ( {1 - a } )y} ) \le a f( x ) + ( {1 - a } )f( y ) - a ( {1 - a } ){\lambda }/{2}{\| {x - y} \|^2}.
		$
		\item $f$ is L-smooth function if $\forall x,y \in \mathbb{R}^M$, it satisfies $f( x ) \le f( y ) + \langle {\nabla f( y ),x - y} \rangle  + {L}/{2}\| {x - y} \|^2$. If $f$ is convex, it also satisfies $f( x ) \ge f( y ) + \langle {\nabla f( y ),x - y} \rangle  + {1}/{(2L)}{\| {\nabla f( x ) - \nabla f( y )} \|^2}$.
	\end{itemize}
\end{definition}

\begin{assumption}\label{Assumption1}	The random variables $\left( {i,j} \right)$ are independent and identically distributed, $i \in [ n ], j \in [ m ], \forall x \in {\mathbb{R}^M}$
	\begin{align*}
	E[ {\partial {G_j}( x )^\mathsf{T}\nabla {F_i}( {G( x )} )} ] = \partial G( x )^\mathsf{T}\nabla F( {G( x )} ).
	\end{align*}
\end{assumption}

\begin{assumption}\label{Assumption2} For function  $\frac{1}{n}\sum\nolimits_{i = 1}^n {{F_i}( {\frac{1}{m}\sum\nolimits_{j = 1}^m {{G_j}( x )} } )} $,  we assume that 
	\begin{itemize}
		\item $F_i$ has the bounded gradient and Lipschitz continuous gradient, $i \in [ n ]$, 
		\begin{align}
		\label{InequationAssumptionF1}
		\| {\nabla {F_i}( y )} \| \le& {B_F},\forall y \in  \mathbb{R}{^M},\\
		\label{InequationAssumptionF2}
		\| {\nabla {F_i}( x ) - \nabla {F_i}( y )} \| \le& {L_F}\| {x - y} \|,\forall x,y \in \mathbb{R} {^M}.
		\end{align}
		\item $G_j$ has the bounded Jacobian and Lipschitz continuous gradient, $j \in [ m ]$
		\begin{align}
		\label{InequationAssumptionG1}
		\| {\partial {G_j}( x )} \| \le& {B_G},\forall x \in \mathbb{R} {^N},\\
		\label{InequationAssumptionG2}
		\| {{G_j}( x ) - {G_j}( y )} \| \le& {B_G}\| {x - y} \|,\forall x,y \in  \mathbb{R} {^N},\\
		\label{InequationAssumptionG3}
		\| {\partial {G_j}( x ) - \partial {G_j}( y )} \| \le& {L_G}\| {x - y} \|,\forall x,y \in  \mathbb{R} {^N}.
		\end{align}
	\end{itemize}
\end{assumption}
\begin{assumption}\label{Assumption3}
	For function  $\frac{1}{n}\sum\nolimits_{i = 1}^n {{F_i}( {G(x) } )} $,  we assume that $F_i$ is $L_f$-smoothness and convex, then,
	\begin{align}
	\label{InequationAssumption3}
	\| (\partial G(x))^\mathsf{T}\nabla {F_i}(G(x)) - (\partial G(y))^\mathsf{T}\nabla {F_i}(G(y)) \|^2 /(2L_f)
	\le {F_i}(G(x)) - \nabla {F_i}(G(y)) - \langle (\partial G(y))^\mathsf{T}\nabla {F_i}(G(y)),x - y \rangle .
	\end{align}
\end{assumption}


\section{The duality-Free method for Stochastic Composition}
Here we introduce the duality-free method for stochastic composition. This method is a natural extension of duality-free SDCA: at each iteration, the dual variable and the primal variable are alternately updated, where the estimated gradient  satisfies $E[ {{( {\partial {G}( x )} )^\mathsf{T}}\nabla {f_i}( {G( x )} ) + \nabla R( x )} ] = {( {\partial G( x )} )^\mathsf{T}}\nabla f( {G( x )} ) + \nabla R( x )$. Note that the query complexity for computing the estimated gradient is $O(2+2m)$.  We first describe the relationship between the primal and dual variable and derive the estimated gradient that satisfies the unbiased estimate for the composition problem. Algorithm \ref{algorithmSCDF} shows the duality-free process. Note that  partial gradient $\partial G_j(x)$ and  inner function $G(x)$ are computed directly. In our proposed method, both function $G$  and its partial gradient can be estimated using  variance reduction approaches.

\begin{algorithm}[h]
	\caption{Dual-Free for composition function }\label{algorithmSCDF}
	\begin{algorithmic}[1]
		\Require $\beta ^0 = {\left( {\nabla G\left( {{x_0}} \right)} \right)^\mathsf{T}}\alpha^0$
		\Ensure 
		\For{$t$=1 to T }
		\State Randomly select $i\in [n]$ and $j\in[m]$
		\State 
		$\beta _i^{t + 1} = \beta _i^t - \lambda n{\eta}( {{( {\partial  G( {{x_t}} )} )^\mathsf{T}}\nabla {F_i}( {G( {{x_t}} )} ) + \beta _i^t} )$\label{algorithmSCDFStep1}
		\State
		${x_{t+1}} = {x_{t}} - {\eta}( {{( {\partial G( {{x_{t}}} )} )^\mathsf{T}}\nabla {F_i}( {G( {{x_{t}}} )} ) + \beta _i^{t}} )$
		\EndFor
	\end{algorithmic}
\end{algorithm}

To obtain the dual function,  we adopt the Fenchel duality method \cite{bertsekas1999nonlinear}, which is derived by converting the original problem  (\ref{SCDF:ProblemMainCompositionminimization}) to  the equation equality optimization problem in variables $y_i$, $i\in[n]$,
\begin{align*}
\mathop {\min }\limits_{x \in {\mathbb{R}^N},{y_i} \in {\mathbb{R}^M}}\,\, \frac{1}{n}\sum\limits_{i = 1}^n {{F_i}( {{y_i}} )}  + R( x ),\,\,\,\,
s.t.\,\,{y_i} = \frac{1}{m}\sum\limits_{j = 1}^m {{G_j}( x )}.
\end{align*}
Its corresponding Lagrange function is 
\begin{align*}
L( {x,y,\alpha } ) =& \frac{1}{n}\sum\limits_{i = 1}^n {{F_i}( {{y_i}} )}  + R( x ) + \frac{1}{n}\sum\limits_{i = 1}^n {\langle {{\alpha _i},{y_i} - \frac{1}{m}\sum\limits_{j = 1}^m {{G_j}( x )} } \rangle } \\
=&  - \frac{1}{n}\sum\limits_{i = 1}^n {( {\langle { - {\alpha _i},{y_i}} \rangle  - {F_i}( {{y_i}} )} )}  - ( {\langle {\frac{1}{n}\sum\limits_{i = 1}^n {{\alpha _i}} ,G( x )} \rangle  - R( x )} ),
\end{align*}
where $\alpha_i\in \mathbb{R}^M$ is the Lagrange multiplier. Through minimizing the Lagrange function with respect to $x$ and $y$, respectively, we have
\begin{align*}
\begin{array}{*{20}{l}}
{D( \alpha  ) = \mathop {\min }\limits_{x,y} L( {x,y,\alpha } ) =  - \frac{1}{n}\sum\limits_{i = 1}^n {F_i^*( { - {\alpha _i}} )} }
\end{array} - {{\tilde R}^*}( \alpha  ),
\end{align*}
where ${F_i^*\left( { - {\alpha _i}} \right)}$ is the conjugate function of $F_i$,  and ${{\tilde R}^*}\left( \alpha  \right)$ is the function with respect to $\alpha$,
\begin{align*}
F_i^*( { - {\alpha _i}} ) =& \mathop {\max }\limits_{{y_i}} \{ {\langle { - {\alpha _i},{y_i}} \rangle  - {F_i}( {{y_i}} )} \},\\
{{\tilde R}^*}( \alpha  ) = &\mathop {\max }\limits_x \{ {\langle {\frac{1}{n}\sum\nolimits_{i = 1}^n {{\alpha _i}} ,G( x )} \rangle  - R( x )} \}.
\end{align*}
Based on the convexity definition, we can see that ${{\tilde R}^*}\left( \alpha  \right)$ is convex function but not the conjugate of $R(x)$ if $G(x)$ is not affine. Furthermore,  ${{\tilde R}^*}\left( \alpha  \right)$ is not easily computed if $G(x)$ is complicated. However,  the dual problem is concave problem, and the relationship between primal variable and dual variable can be obtained through keeping the gradient of ${\left\langle {\frac{1}{n}\sum\nolimits_{i = 1}^n {{\alpha _i}} ,G\left( x \right)} \right\rangle  - R\left( x \right)}$ w.r.t. $x$ to zero,
\begin{align}\label{EquationGradientR}
{( {\partial G( x )} )^\mathsf{T}}\frac{1}{n}\sum\nolimits_{i = 1}^n {{\alpha _i}}  = \nabla R( x ).
\end{align}
We observe that the update of x can be written as
\begin{align*}
{x_{t+1}} = {x_{t}} - {\eta}( {{( {\partial {G}( {{x_{t }}} )} )^\mathsf{T}}\nabla {f_i}( {G( {{x_{t }}} )} ) + {{( {\nabla G( {{x_{t }}} )} )}^\mathsf{T}}\alpha _i^{t}} ).
\end{align*}
Based on the expectation of gradient, we have
\begin{align*}
E[ {{x_{t + 1}}} ] 
=& E[ {{x_t}} ] - {\eta}E[ {{( {\partial {G}( {{x_t}} )} )^\mathsf{T}}\nabla {f_i}( {G( {{x_t}} )} ) + {( {\nabla G( {{x_t}} )} )^\mathsf{T}}\alpha _i^t} ]\\
=& E[ {{x_t}} ] - {\eta}\nabla P( {{x_t}} ),
\end{align*}
where the gradient is 
\begin{align}\label{EquationGradientP}
\nabla P( x ) = {( {\partial G( x )} )^\mathsf{T}}\nabla f( {G( x )} ) + \nabla R( x ).
\end{align}

For the case of $l_2$ norm, that is $R( x ) = \frac{1}{2}\lambda \| x \|^2$, from (\ref{EquationGradientR}), we have $\lambda x = {( {\nabla G( x )} )^\mathsf{T}}\frac{1}{{n }}\sum\nolimits_{i = 1}^n {\alpha _i^{}}$. Let $\beta _i^t = {( {\nabla G( {{x_t}} )} )^\mathsf{T}}\alpha _i^t$, we observe that
\begin{align*}
E[ {\beta _i^{t+1}} ] - E[ {\beta _i^{t}} ] 
=& E[ {{( {\nabla G( {{x_{t+1}}} )} )^\mathsf{T}}\alpha _i^{t+1}} ] - E[ {{( {\nabla G( {{x_{t}}} )} )^\mathsf{T}}\alpha _i^{t }} ]\\
=& \lambda n( {{x_{t+1}} - {x_{t }}} )\\
=&\lambda n {\eta}E[ {{( {\nabla G( {{x_{t }}} )} )^\mathsf{T}}\nabla {f_i}( {G( {{x_{t }}} )} ) + {( {\nabla G( {{x_{t }}} )} )^\mathsf{T}}\alpha _i^{t }} ].
\end{align*} 
Then, the update of $w_t$ becomes
\begin{align*}
\beta _i^{t+1} = \beta _i^{t } - \lambda n{\eta}( {{( {\nabla G( {{x_{t }}} )} )^\mathsf{T}}\nabla {f_i}( {G( {{x_{t }}} )} ) + \beta _i^{t}} ).
\end{align*}

Let $x^*$ be the optimal primal solution and $\alpha^*$ be the optimal dual solution. Combining equations (\ref{EquationGradientP}) and (\ref{EquationGradientR}), their relationship  is 
\begin{align*}
\frac{1}{n}\sum\limits_{i = 1}^n {{( {\nabla G( {{x^*}} )} )^\mathsf{T}}\alpha _i^*}  =  - \frac{1}{n}\sum\limits_{i = 1}^n {( {{( {\nabla G( {{x^*}} )} )^\mathsf{T}}\nabla {f_i}( {\nabla G( {{x^*}} )} )} )}. 
\end{align*}
Through the relationship, we can see that according to the theorem in \cite{shalev2016sdca}, the primal and dual solutions converge to the optimal point at the linear convergence rate. Furthermore, as the iterations increase, the gradient variance  asymptotically approaches zero as $x$ and $\alpha$ go to the optimal solution. Note that the inner function $G$ is fully computed.

In  Algorithm \ref{algorithmSCDF}, each iteration requires computing  function $G$ and its partial gradient $\partial G$, which has $O(2+2m)$ query complexity. In the next section, we provide the variance reduction method to estimate  function $G$ and partial gradient $\partial G$.
\section{The duality-free and variance-reduced method for stochastic composition optimization}

To reduce query complexity, we follow the variance reduction method in \cite{lian2016finite} to estimate  $G$ and  $\partial G$. In doing so, we propose SVRG- and SGAG-based  SCDF methods, referred to here as SCDF-SVRG and SCDF-SAGA. These two methods not only include gradient estimations  but also  estimate the inner function $G$ and corresponding partial gradient:
\begin{itemize}
	\item In SCDF-SVRG, we divide iterations into epochs, each with a snapshot point $\tilde{x}$. For the finite-sum structure function $G$, we follow the SVRG-based method  in \cite{lian2016finite} to estimate the full function and full partial gradient at the snapshot point. In the inner iteration, composition-SVRG2 defines the function estimator  $G_j( x ) - {G_j}( {\tilde x} ) + G( {\tilde x} )$  and the partial gradient estimator $\partial  {G_j}( x ) - \partial {G_j}( {\tilde x} ) + \partial G( {\tilde x} )$. Then, we  use the estimated $G$ and its partial gradient to define a new gradient estimation  of function $F(G(x))$. We extend the dual-free SDCA method using the estimated gradient to tackle the formed convex-concave problem.  Pseudocode can be found in Algorithm \ref{AlgorithmSDFCVRG1}
	\item 	In  SCDF-SAGA, we replace the estimation method for inner function $G$ with the SAGA-based method. They are the function estimator  $\partial {G_j}( x ) - \partial {G_j}( {{\phi _j}} ) + \frac{1}{m}\sum\nolimits_{j = 1}^m {\partial {G_j}( {{\phi _j}} )} $ and the partial gradient estimator ${G_j}( x ) - {G_j}( {{\phi _j}} ) + \frac{1}{m}\sum\nolimits_{j = 1}^m {{G_j}( {{\phi _j}} )} $. Thus, we can also obtain the new estimator of full gradient $F(G(x))$, which can be applied to the dual-free SDCA method. SCDF-SVRG differs in that there is no epoch to maintain a snapshot point. Pseudocode can be found in Algorithm \ref{AlgorithmSCDFSAGA}
\end{itemize}
\subsection{Estimating the function $G$ based on SVRG}
Specifically, we describe SCDF-SVRG method. Because $G(x)$ function is also sums of  function $G_i$.
For each epoch, the estimated function and the corresponding estimated partial gradient of $G(x)$ are,
\begin{small}
	\begin{align}
	\label{SCDF:SVRG:DefinitionSVRGFunctionG}
	{{\hat G}_k} &= \frac{1}{A}\sum\limits_{1 \le j \le A}^{} {( {{G_{{{\cal A}_k}[j]}}( {{x_k}} ) - {G_{{{\cal A}_k}[j]}}( {{{\tilde x}_s}} )} )}  + G( \tilde x_s ),\\
	\label{SCDF:SVRG:DefinitionSVRGEstimateG}
	\partial {{\hat G}_k} &= \frac{1}{A}\sum\limits_{1 \le j \le A}^{} {( {\partial {G_{{{\cal A}_k}[j]}}( {{x_k}} ) - \partial {G_{{{\cal A}_k}[j]}}( {{{\tilde x}_s}} )} )}  + \partial G( {{{\tilde x}_s}} ),
	\end{align}
\end{small}
where $\tilde{x}_s$ is the current outer iteration, $x_k$ is the current inner iteration, $\mathcal{A}$ is the mini-batch multiset and $A$ is the sample times from $\forall i \in[n]$ to form $\mathcal{A}$. Taking expectation with respect to $i$, we have 
\begin{align*}
E[ {{{\hat G}_k}} ] = G( {{x_k}} ),E[ {\partial {{\hat G}_k}} ] = \partial G( {{x_k}} ).
\end{align*}
Furthermore, we assume $i$ and $j$ are independent with each other, that is $E[ {{( {\partial {G_j}( {{x_k}} )} )^\mathsf{T}}\nabla {F_i}( {{{\hat G}_k}} )} ] = {( {\partial G( {{x_k}} )} )^\mathsf{T}}\nabla F( {{{\hat G}_k}} )$.	Then the step \ref{algorithmSCDFStep1} in algorithm \ref{algorithmSCDF}, can be replaced by
\begin{align*}
x_{k+1} = x_{k } - \eta ( {{( {\partial {{\hat G}_{k }}} )^\mathsf{T}}\nabla {F_i}( {{{\hat G}_{k }}} ) + \beta _i^{k }} ).
\end{align*}

However, because the inner function ${{\hat G}_k}$ is also estimated, $E[( \partial {G_j}( {{x_k}} ) )^\mathsf{T}\nabla {F_i}({{\hat G}_k})] \ne ( \partial G( x_k ) )^\mathsf{T}\nabla F(G({x_k}))$, Even though the biased of the estimated gradient $E[{( {\partial {G_j}( {{x_k}} )} )^\mathsf{T}}\nabla {F_i}({{\hat G}_k})]$, we give the following lemma to show that the variance between ${( {\partial {G_j}( {{x_k}} )} )^\mathsf{T}}\nabla {F_i}({{\hat G}_k})$ and ${( {\partial G( {{x_k}} )} )^\mathsf{T}}\nabla F(G({x_k}))$ decrease as the variable $x_k$ and $\tilde{x}_s$ close to the optimal solution,
\begin{lemma}\label{SCDF:LemmBoundSVRGEstimateFullGradientF}
	Suppose Assumption \ref{Assumption2} holds, in algorithm \ref{AlgorithmSDFCVRG1}, for the intermediated iteration at $x_k$ and $\tilde{x}_s$, and $ \hat{G}_k$ and  $\partial \hat{G}_k$ defined in (\ref{SCDF:SVRG:DefinitionSVRGEstimateG}) and (\ref{SCDF:SVRG:DefinitionSVRGFunctionG}),  we have
	\begin{align*}
	E[ \| ( \partial \hat G_k  )^\mathsf{T}\nabla {F_i}( \hat G_k  ) - ( \partial \hat G_k  )^\mathsf{T}\nabla {F_i}( G(x_{k}) ) \|^2 ]
	\le& B_G^4L_F^2\frac{1}{A}E[\| x_k -  x^* \|^2]+B_G^4L_F^2\frac{1}{A}E[\| \tilde x_s -x^*\|^2],
	\end{align*}
	where $L_F$ and $B_G$ are the parameters in (\ref{InequationAssumptionF2}) and (\ref{InequationAssumptionG1}).
\end{lemma}
\begin{remark}
	The mini-batch $\mathcal{A}_k$ is obtain by sampling from $[m]$ for $A$ times, if the number of $A$ is infinite, then we can see that ${{\hat G}_k} \approx G\left( {{x_k}} \right)$, the difference between ${( {\partial {G_j}( {{x_k}} )} )^\mathsf{T}}\nabla {F_i}({{\hat G}_k})$ and ${( {\partial G( {{x_k}} )} )^\mathsf{T}}\nabla F(G({x_k}))$ is also approximating to zero. This is verified by Lemma \ref{SCDF:LemmBoundSVRGEstimateFullGradientF} that the difference is bounded by $O({1}/{A})$ (assume $E[ {\| {{x_{k }} - {{\tilde x}_s}} \|^2} ]$ is a bound sequence) that as $A$ increase, the upper bound approximate to zero.
\end{remark}
\begin{algorithm}[h]
	\caption{SCDF-SVRG}\label{AlgorithmSDFCVRG1}
	\begin{algorithmic}[1]
		\State Initialize: ${x_0} = \frac{1}{n}\sum\nolimits_{i = 1}^n {\beta _i^0} $, ${{\tilde x}_0} = {x_0},$
		\For{$s$=0,1,2...S-1}
		\State $G({\tilde x_s}) = \frac{1}{m}\sum\nolimits_{j = 1}^m {{G_j}({{\tilde x}_s})} $\Comment{m Queries}
		\State $\partial  G( {{{\tilde x}_s}} ) =\frac{1}{m}\sum\nolimits_{j = 1}^m {{\partial G_j}({{\tilde x}_s})} $\Comment{m Queries}
		\State $x_0={\tilde x}_s$
		\For{$k$=0,2...K-1 }	
		\State Sample from $[m]$ for $A$ times to form the mini-batch  ${{{\cal A}_{k}}}$
		\State
		Update ${{\hat G}_{k }} $ from (\ref{SCDF:SVRG:DefinitionSVRGFunctionG})\Comment{2A Queries}
		\State 
		Update $\partial {{\hat G}_{k}}$ from (\ref{SCDF:SVRG:DefinitionSVRGEstimateG})\Comment{2A Queries}
		\State Randomly select $i \in [n]$
		\State 
		$\beta _i^{k+1 }= \beta _i^{k} - \lambda n\eta ( {{( {\partial {{\hat G}_{k}}} )^\mathsf{T}}\nabla {F_i}( {{{\hat G}_{k }}} ) + \beta _i^{k }} )$
		\State
		${x_{k+1}} = {x_{k}} - \eta ( {{( {\partial {{\hat G}_{k}}} )^\mathsf{T}}\nabla {F_i}( {{{\hat G}_{k}}} ) + \beta _i^{k }} )$
		\EndFor
		\State
		${{\tilde x}_{s + 1}} = \frac{1}{K}\sum\limits_{k = 1}^K {{x_k}} ,\tilde \beta _i^{s + 1} = \frac{1}{K}\sum\limits_{k = 1}^K {{\beta _{i}^{k}}}$, $i\in[n]$
		\EndFor	
	\end{algorithmic}
\end{algorithm}
\subsubsection{Convergence analysis}
Here we provide two different convergence analyses for the cases that the individual function  $F_i$ is convex and non-convex, respectively. Theorem \ref{SCDF:SVRG:TheoremSVRGMainConvergenceNonconvex} gives the convergence analysis without  Assumption \ref{Assumption3} that function $F_i$ can be non-convex but $P(x)$ is convex.  Theorem \ref{SCDF:SVRG:TheoremSVRGMainConvergenceConvex} gives the convergence rate under Assumption \ref{Assumption3}. Both  convergence rates are linear.
\begin{theorem}\label{SCDF:SVRG:TheoremSVRGMainConvergenceNonconvex}
	Suppose Assumption \ref{Assumption1} and \ref{Assumption2} hold,  $P(x)$ is $\lambda$-strongly convex, in algorithm \ref{AlgorithmSDFCVRG1}, let ${{\tilde A}_s} = \left\| {{{\tilde x}_s} - {x^*}} \right\|^2$, ${{\tilde B}_s} = \frac{1}{n}\sum\nolimits_{i = 1}^n {\| {\tilde \beta _i^s - \beta _i^*} \|^2}$, ${{\tilde C}_s} = aE[ {{{\tilde A}_s}} ] + bE[ {{{\tilde B}_s}} ]$. Define $\lambda {R_x} = {\max _x}\{ {{{\| {{x^*} - x} \|}^2}:F(G(x)) \le F(G({x_0}))} \}$, the SCDF-SVRG method has geometric convergence:
	\begin{align*}
	{\tilde C_{s}} \le {\left( {\frac{1}{{\eta \lambda K}} + \frac{{{d_2}}}{{a\eta \lambda }}} \right)^s}{{\tilde C}_0},
	\end{align*}
	where the parameters $a,b, d_2$ and $\eta$ satisfy
	\begin{align*}
	\eta  &\le \frac{{\frac{1}{2}{\lambda ^2} - 4B_G^4L_F^2\frac{1}{A}}}{{2\left( {4B_F^2L_G^2\frac{1}{A} + 4B_G^4L_F^2\frac{1}{A}} \right)\lambda  + \frac{1}{2}{\lambda ^3}n - 4\lambda B_G^4L_F^2\frac{1}{A}n}},\\
	{d_2} &= 2\left( {a\eta qB_G^4L_F^2\frac{1}{A} + b\lambda \eta \left( {4B_F^2L_G^2\frac{1}{A} + 4B_G^4L_F^2\frac{1}{A}} \right)} \right) + b\lambda \eta \left( {4B_F^2L_G^2 + 4B_G^4L_F^2} \right),\\
	&\frac{{2\lambda \left( {4B_F^2L_G^2\frac{1}{A} + 4B_G^4L_F^2\frac{1}{A}} \right)}}{{\lambda  - \frac{1}{q} - 2qB_G^4L_F^2\frac{1}{A}}} \le \frac{a}{b} \le \frac{{\left( {1 - n\lambda \eta } \right)\lambda }}{\eta },\\
	q &= {{A\lambda }}/({{4B_G^4L_F^2}}).
	\end{align*}
\end{theorem}

\begin{remark}
	The convergence analysis does not need the convexity of individual function $F_i$ but requires function $P(x)$ to be strongly convex.
\end{remark}
The following theorem also gives the geometric convergence in the case that $F_i$ is convex. Even though the proof method is similar to Theorem \ref{SCDF:SVRG:TheoremSVRGMainConvergenceNonconvex}, the inner convergence analyses is different such that it lead to different convergence.
\begin{theorem}\label{SCDF:SVRG:TheoremSVRGMainConvergenceConvex}
	Suppose Assumption \ref{Assumption1}, \ref{Assumption2} and \ref{Assumption3} hold,  $F_i$ is convex function, and $P(x)$ is $\lambda$-strongly convex, in algorithm \ref{AlgorithmSDFCVRG1}, let ${{\tilde A}_s} = \left\| {{{\tilde x}_s} - {x^*}} \right\|^2,{{\tilde B}_s} = \frac{1}{n}\sum\nolimits_{i = 1}^n {\| {\tilde \beta _i^s - \beta _i^*} \|^2}  ,{{\tilde C}_s} = aE[ {{{\tilde A}_s}} ] + bE[ {{{\tilde B}_s}} ]$. Define $\lambda {R_x} = {\max _x}\{ \| {{x^*} - x} \|^2:F(G(x)) \le F(G({x_0})) \}$, the SCDF-SVRG method has geometric convergence:
	\begin{align*}
	{\tilde C_{s}} \le {\left( {\frac{1}{{\eta \lambda K}} + \frac{{{e_2}}}{{a\eta \lambda }}} \right)^s}{{\tilde C}_0},
	\end{align*}
	where the parameters $a,b,d, e_2$ , $\eta$ and $A$ satisfy
	\begin{align*}
	&A \ge {{2{R_x}B_G^4L_F^2}}/{d},\\
	&\eta  \le ({{1 - d}})/({{2{L_f} + \lambda n\left( {1 - d} \right)}}),\\
	&e_2={2a\eta \lambda {R_x}B_G^4L_F^2\frac{1}{A} + 4b\lambda \eta \left( {B_F^2L_G^2 + B_G^4L_F^2} \right)/{A}},\\
	&\frac{{2\left( {2B_F^2L_G^2 + B_G^4L_F^2} \right)\frac{1}{A} - {L_f}\lambda }}{{d - 2{R_x}B_G^4L_F^2\frac{1}{A}}} \le \frac{a}{b} \le \frac{{(1 - n\lambda \eta )\lambda }}{\eta },\\
	&d \le \frac{{\left( {2B_F^2L_G^2 + B_G^4L_F^2} \right)\frac{1}{A} + \lambda {L_F}{R_x}B_G^4L_F^2\frac{1}{A}}}{{\left( {2B_F^2L_G^2 + B_G^4L_F^2} \right)\frac{1}{A} + \lambda {L_f}}}.
	\end{align*}
\end{theorem}

The variance bound of the modified estimate gradient is shown in the following corollary. Note that the inner function $\hat G$ is the estimated function of $G$. 
\begin{corollary}\label{CorollarySVRGGradient}
	Suppose Assumption \ref{Assumption2} holds, in algorithm \ref{AlgorithmSDFCVRG1}, for the intermediated iteration at $x_k$ and $\beta_k$, we have
	\begin{align*}
	E[ {\| {{( {\partial {{\hat G}_k}} )^\mathsf{T}}\nabla {F_i}( {{{\hat G}_k}} ) + \beta _i^k} \|^2} ]
	\le& \left( {4B_F^2L_G^2{1}/{A} + 4B_G^4L_F^2{1}/{A}} \right)E[ {\| {{x_k} - {{\tilde x}_s}} \|^2} ]\\
	&+ ( {4B_F^2L_G^2 + 4B_G^4L_F^2} )E[ {\| {{{\tilde x}_s} - {x^*}} \|^2} ] + E[ {\| {\beta _i^k - \beta _i^*} \|^2} ],
	\end{align*}
	where $B_F$, $B_G$ and $L_G$ are the parameters in (\ref{InequationAssumptionF1})- (\ref{InequationAssumptionG2}).
\end{corollary}
\begin{remark}
	From Corollary \ref{CorollarySVRGGradient}, the variance of the estimated gradient is bound by $O( {E[ {\| {{x_k} - {{\tilde x}_s}} \|^2} ]} )$ and  $O( {E[ {\| {\beta _i^k - \beta _i^*} \|^2} ]} )$. As $x_k$, $\tilde{x}_s$ and $\beta_k$ go to the optimal solution, the variance also approximates to zero.
\end{remark}
\subsection{SAGA-based method for estimating function $G$}
Extending SAGA such that the table elements are updated iteratively, we propose SAGA-based SCDF. In contrast to SCDF-SVRG, there is no need to compute the full function and full partial gradient of $G$. This approach is analogous to the duality-free method in that it can avoid computing the full gradient of function $F$. Following the variance reduction technology in SGAG, we replace step \ref{algorithmSCDFStep1} in Algorithm \ref{algorithmSCDF}  with
\begin{align*}
{x_{k + 1}} = {x_k} - \eta ( {{(\partial {{\hat G}_k})^\mathsf{T}}\nabla {F_i}({{\hat G}_k}) + \beta _i^k} ),
\end{align*}
where
\begin{align}
{{\hat G}_k} =& \frac{1}{A}\sum\limits_{1 \le j \le A}^{} {( {{G_{{\mathcal{A}_k}[j]}}\left( {{x_k}} \right) - {G_{{A_k}[j]}}( {\phi _{{{{\cal A}_k}[j]}}^k} )} )}
\label{DefinitionSAGAEstimateG}
  + \frac{1}{m}\sum\limits_{j = 1}^m {G_j( {\phi _{j}^k} )}, \\
\partial {{\hat G}_k} =& \frac{1}{A}\sum\limits_{1 \le j \le A}^{} {( {\partial {G_{{\mathcal{A}_k}[j]}}\left( {{x_k}} \right) - \partial {G_{{A_k}[j]}}}({\phi _{{{{\cal A}_k}[j]}}^k}) )}
\label{DefinitionSAGAEstimateGradientG}
  + \frac{1}{m}\sum\limits_{j = 1}^m {\partial G_j( {\phi _{{j}}^k} )}, 
\end{align}
$\mathcal{A}_k$ is the mini-batch formed by sampling $A$ times from $[n]$. $\mathcal{A}_k[j]$, $j \in {{\cal A}_k}$ indicates the $j$th element in the list $\mathcal{A}$.	${\phi _{{{\cal A}_k}[j]}^k}$, $j \in {{\cal A}_k}$ is stored in the variable table list. Taking expectation on above estimated function $G$ and partial gradient of $G$, we have $E[{{{\hat G}_k}}]= G({{x_k}})$ and $E[{\partial {{\hat G}_k}} ] = \partial G({{x_k}} )$. But the same problem as in SCDF-SVRG, the estimated gradient  is not unbiased estimation, because $E[ {{(\partial {{\hat G}_k})^\mathsf{T}}\nabla {F_i}({{\hat G}_k})} ] \ne  E[{{(\partial {G(x_k)})^\mathsf{T}}\nabla {F_i}({G}({x_k}))}] $. However, based on the above estimation about function $G$, we also give the upper bound of the difference between them,

\begin{algorithm}[h]
	\caption{SCDF-SAGA}\label{AlgorithmSCDFSAGA}
	\begin{algorithmic}[1]
		\State Initialize: ${x_0} = \frac{1}{n}\sum\limits_{i = 1}^n {\beta _i^0} $, ${x_0} = \phi _j^0,j \in [m]$,
		\For{$k$=0,2...K-1 }	
		\State Sample from $\left\{ {1,...,m} \right\}$ for $A$ times to form the mini-batch ${{{\cal A}_{k}}}$
		\State
		Update ${{\hat G}_k}$ by using (\ref{DefinitionSAGAEstimateG})
		\Comment{A Queries}
		\State 
		Update$\partial {{\hat G}_k}$ by using (\ref{DefinitionSAGAEstimateGradientG})
		\Comment{A Queries}
		\State Take $\phi _{{{\cal A}_k}[j]}^{k + 1}=x_k$ for $j\in \mathcal{A}_k$, and $\phi _{{{\cal A}_k}[j]}^{k + 1}=\phi _{{{\cal A}_k}[j]}^{k}$ for $j\in[m]$ but $j \notin  \mathcal{A}_k$
		\State
		Update ${{\tilde G}_{k + 1}}$ by using (\ref{DefinitionSAGAEstimateGExchange})			
		\State
		Update $\partial {{\tilde G}_{k + 1}}$ by using (\ref{DefinitionSAGAEstimateGradientGExchange})
		\State Randomly select $i \in [n]$
		\State 
		$\beta _i^{k + 1} = \beta _i^k - \lambda n\eta ( {{(\partial {{\hat G}_k})^\mathsf{T}}\nabla {F_i}({{\hat G}_k}) + \beta _i^k} )$
		\State
		${x_{k + 1}} = {x_k} - \eta ( {{(\partial {{\hat G}_k})^\mathsf{T}}\nabla {F_i}({{\hat G}_k}) + \beta _i^k} )$	
		\EndFor	
	\end{algorithmic}
\end{algorithm}

\begin{lemma}\label{LammaBoundSAGANormEsimateGradientShort}
	Assume Assumption \ref{Assumption2} holds, in algorithm \ref{AlgorithmSCDFSAGA}, for the intermediated iteration at $x_k$, $ \hat G$ defined in (\ref{DefinitionSAGAEstimateG}) and  $\partial \hat G_k$ defined in (\ref{DefinitionSAGAEstimateGradientG}), the following bound satisfies,
	\begin{align*}
	E[ {\| {{(\partial {{\hat G}_k})^\mathsf{T}}\nabla {F_i}({{\hat G}_k}) - {(\partial G({x_k}))^\mathsf{T}}\nabla {F_i}({G_k}({x_k}))} \|^2} ] \le& ( {2B_F^2L_G^2 + 2B_G^4L_F^2} ){1}/{{{A^2}}}\sum\limits_{1 \le j \le A}^{} {E[ {\| {{x_k} - \phi _{{{\cal A}_k}[j]}^k} \|_{}^2} ]},  
	\end{align*}
	where  $L_F$, $L_G$, $B_F$ and $B_G$ are the parameters in (\ref{InequationAssumptionF1})- (\ref{InequationAssumptionG3}).  
\end{lemma}
\begin{remark}
	As $x_k$ and $\phi^k$ go to the optimal solution, the expectation bound  approximates to zero. Furthermore, this lemma also shows that as $A$ increases,  the estimated ${{{\hat G}_k}}$ approaches the exact function $G$.
\end{remark}
At intermediated iteration $x_k$, define 
\begin{align*}
{{\tilde G}_k} = \frac{1}{m}\sum\nolimits_{j = 1}^m {{G_j}\left( {\phi _j^k} \right)} , \partial {{\tilde G}_k} =\frac{1}{m}\sum\nolimits_{j = 1}^m {\partial {G_j}\left( {\phi _j^k} \right)}.
\end{align*}
Note that for each time estimation for function $G$, the term ${{\tilde G}_k}$ and $\partial {{\tilde G}_k}$ can be iteratively updated without computing the full function and full partial gradient of function $G$,    
\begin{align}
{{\tilde G}_{k + 1}} =& \frac{A}{n}\sum\nolimits_{1 \le j \le A} {{( {{G_{{\mathcal{A}_k}[j]}}( {{\phi _{{{\cal A}_k}[j]}^{k + 1}}} ) - {G_{{\mathcal{A}_k}[j]}}( {\phi _{{{{\cal A}_k}[j]}}^k} )} )}}
\label{DefinitionSAGAEstimateGExchange}
+{{\tilde G}_k},  \\
\partial {{\tilde G}_{k + 1}} =&\frac{A}{n}\sum\nolimits_{1 \le j \le A} {( {\partial {G_{{\mathcal{A}_k}[j]}}( {{\phi _{{{\cal A}_k}[j]}^{k + 1}}} ) - \partial {G_{{\mathcal{A}_k}[j]}}(\phi _{{{{\cal A}_k}[j]}}^k)} )} 
\label{DefinitionSAGAEstimateGradientGExchange}
+ \partial {{\tilde G}_k}. 
\end{align}

\subsubsection{Convergence analysis}
Similar to the SVRG-based SCDF method, we also provide two convergence rates for the two cases that the individual function $F_i$ is convex or non-convex. In Theorem \ref{SCDF:SAGA:TheoremSAGAMainConvergenceNonconvex}, we provide the linear convergence rate for the non-convex case but $P(x)$ is strongly convex; in Theorem \ref{SCDF:SAGA:TheoremSAGAMainConvergenceConvex}, we also provide linear convergence rate for the convex case where $P(x)$ is strongly convex;
\begin{theorem}\label{SCDF:SAGA:TheoremSAGAMainConvergenceNonconvex}
	Suppose Assumption \ref{Assumption1} and \ref{Assumption2} hold,  and  P(x) is $\lambda$-strongly convex. Let  ${A_k} = \| {{x_k} - {x^*}} \|^2$, ${B_k} = \frac{1}{n}\sum\nolimits_{i = 1}^n {\| {\beta _i^k - \beta _i^*} \|^2}  $ and ${C_{k}} = \frac{1}{m}\sum\nolimits_{j = 1}^m {\| {\phi _j^k - {x^*}} \|^2} $, $x^*$ is the minimizer of $P(x)$ and $E\left[ {\beta _i^*} \right] = \lambda {x^*}$, $A$ is the sample times for forming mini-batch $\mathcal{A}$. Define $\lambda {R_x} = \max _x \{ \| {{x^*} - x} \|^2:F(G(x)) \le F(G({x_0})) \}$. As long as the sample times and  the step satisfy,
	\begin{align*}
	A \ge& ( {\lambda \eta n + 16{R_x}( {B_F^2L_G^2 + B_G^4L_F^2} )} )/2 
	+ \sqrt {{\lambda ^2}{\eta ^2}{n^2} + {{( {16{R_x}( {B_F^2L_G^2 + B_G^4L_F^2} )} )}^2}}/2 ;\\
	\eta  \le& \frac{\lambda }{{2{Y_2} + \frac{A}{{A - \lambda \eta n}}2{Y_3} + {\lambda ^2}n\left( {1 - 8\left( {1 + \frac{A}{{A - \lambda \eta n}}} \right){Y_1}} \right)}},
	\end{align*}
	where ${Y_1} = {R_x}( {B_F^2L_G^2 + B_G^4L_F^2} ){1}/{A},{Y_2} = B_F^2L_G^2{1}/{A} + B_G^4L_F^2,{Y_3} = B_F^2L_G^2{1}/{A}$. Then the SDCA-SAGA method has geometric convergence in expectation:
	\begin{align*}
	aE[ {{A_{k }}} ] + bE[ {{B_{k}}} ] + cnE[ {{C_{k }}} ] \le& {( {1 - \lambda \eta} )^k}( {aE[ {{A_0}} ] + bnE[ {{B_0}} ] + cnE[ {{C_0}} ]} ),
	\end{align*}
	where the parameters $a$, $b$ and $c$ satisfy,
	\begin{align*}
	&\frac{{2{Y_2} + \frac{A}{{A - \lambda \eta n}}2{Y_3}}}{{1 - 8\left( {1 + \frac{A}{{A - \lambda \eta n}}} \right){Y_1}}} \le \frac{a}{b} \le \frac{{(1 - \lambda n\eta )\lambda }}{\eta },\\
	&c \le \left( { - 8a\lambda \eta {Y_1} + a\eta \lambda  - 2b\lambda \eta {Y_2}} \right)/{A}.
	\end{align*}
\end{theorem}

\begin{remark}
	The convergence analysis does not need the convexity of the individual function $F_i$ but requires function $P(x)$ to be strongly convex.
\end{remark}
\begin{theorem}\label{SCDF:SAGA:TheoremSAGAMainConvergenceConvex}
	Suppose Assumption \ref{Assumption1}, \ref{Assumption2} and \ref{Assumption3} hold, $F_i(x)$ is convex, and  P(x) is $\lambda$-strongly convex. Let  ${A_k} = \| {{x_k} - {x^*}} \|^2$, $B_k = \frac{1}{n}\sum\nolimits_{i = 1}^n \| \beta _i^k - \beta _i^* \|^2  $ and ${C_{k}} = \frac{1}{m}\sum\nolimits_{j = 1}^m {\| {\phi _j^k - {x^*}} \|^2} $, $x^*$ is the minimizer of $P(x)$ and $E\left[ {\beta _i^*} \right] = \lambda {x^*}$, $A$ is the sample times for forming mini-batch $\mathcal{A}$. Define $\lambda {R_x} = \max _x \{ \| {{x^*} - x} \|^2:F(G(x)) \le F(G({x_0})) \}$. As long as the samle times $A$ and the step satisfies,
	\begin{align*}
	A \ge& ( {2 + \sqrt 2 } )( {\lambda \eta n + {{16{R_x}( {B_F^2L_G^2 + B_G^4L_f^2} )}}/{d}} ),\\
	\eta  \le& {1}/{{( {{{2{L_f}\lambda }}/( {1 - d} ) + \lambda n} )}},
	\end{align*}
	then the SDCA-SAGA method has geometric convergence in expectation:
	\begin{align*}
	aE[ {{A_{k }}} ] + bE[ {{B_{k}}} ] + cnE[ {{C_{k }}} ]
	\le &{( {1 - \eta \lambda} )^k}( {aE[ {{A_0}} ] + bnE[ {{B_0}} ] + cnE[ {{C_0}} ]} ),
	\end{align*}
	where the parameters $a$, $b$, $c$ ,$d$ and  $q$ satisfy,
	\begin{align*}
	&{{2{L_f}\lambda }}/{{( {1 - d} )}} \le \frac{a}{b} \le {{(1 - \lambda n\eta )\lambda }}/{\eta },\\
	&d \le \frac{{\left( {4Y + \frac{4YA}{{A - \lambda \eta n}} - 2{L_f}\lambda } \right) + 16\left( {1 + \frac{A}{{A - \lambda \eta n}}} \right){R_x}Y{L_f}\lambda }}{{4Y + \frac{A}{{A - \lambda \eta n}}4Y  }},\\
	&c \le ( { - 8a\eta \lambda {R_x}Y + a\eta d\lambda  - 4b\lambda \eta Y + 2b{L_f}{\lambda ^2}\eta } )/A,\\
	&Y = ( {B_F^2L_G^2 + B_G^4L_F^2} )/{A}.
	\end{align*}
\end{theorem}

\begin{remark}
	As parameter $d$ decreases, the lower bound number of sample times $A$ needs to increase, thus the estimated function $G$ and partial gradient of $G$ are well estimated. Furthermore, step $\eta$ can be larger than before. The opposite is also similar. This is verified in Theorem \ref{SCDF:SAGA:TheoremSAGAMainConvergenceConvex}.
\end{remark}

Note that as variable $x_k$ and $\beta^k$ go to the optimal solution, the variance of the gradient in the update iteration  approximates to zero. The following Corollary shows the bound of the estimated gradient variance.
\begin{corollary}\label{CorollarySAGAGradient}
	Suppose Assumption \ref{Assumption2} holds, in algorithm \ref{AlgorithmSCDFSAGA},  $ \hat{G}_k$ and  $\partial \hat{G}_k$ defined in (\ref{DefinitionSAGAEstimateG}) and (\ref{DefinitionSAGAEstimateGradientG}), we have,
	\begin{align*}
	E[ {\| {{( {\partial {{\hat G}_{k}}} )^\mathsf{T}}\nabla {F_i}( {{{\hat G}_{k }}} ) + \beta _i^k} \|^2} ]  
	\le & 4( {B_F^2L_G^2/{A} + B_G^4L_F^2} )E[ {\| {{x_k} - {x^*}} \|_{}^2} ]\\&+ 4B_F^2L_G^2\frac{1}{{{A^2}}}\sum\limits_{1 \le j \le A}^{} {E[ {\| {\phi _{{\mathcal{A}_k}[j]}^k - {x^*}} \|_{}^2} ]}+2E[\|\beta^k-\beta^*\|^2],
	\end{align*}
	where  $L_G$, $L_F$, $B_G$ and $B_F$ are parameters in (\ref{InequationAssumptionF1}) - (\ref{InequationAssumptionG3}).
\end{corollary}

\begin{remark}
	As the SCDF-SAGA method also shows geometric convergence, variables $x_k$ and $\beta_k$ both converge to the optimal solution iteratively. Since they control the upper bound of the  gradient as indicated in the Corollary, the gradient  variance decreases to zero.
\end{remark}

	\section{experiment}
\subsection{Portfolio management- Mean variance optimization}
\begin{figure*}
	\centering
	\subfigure[$\kappa=10$]{
		\begin{minipage}[b]{0.32\textwidth}
			\includegraphics[width=1.0\textwidth]{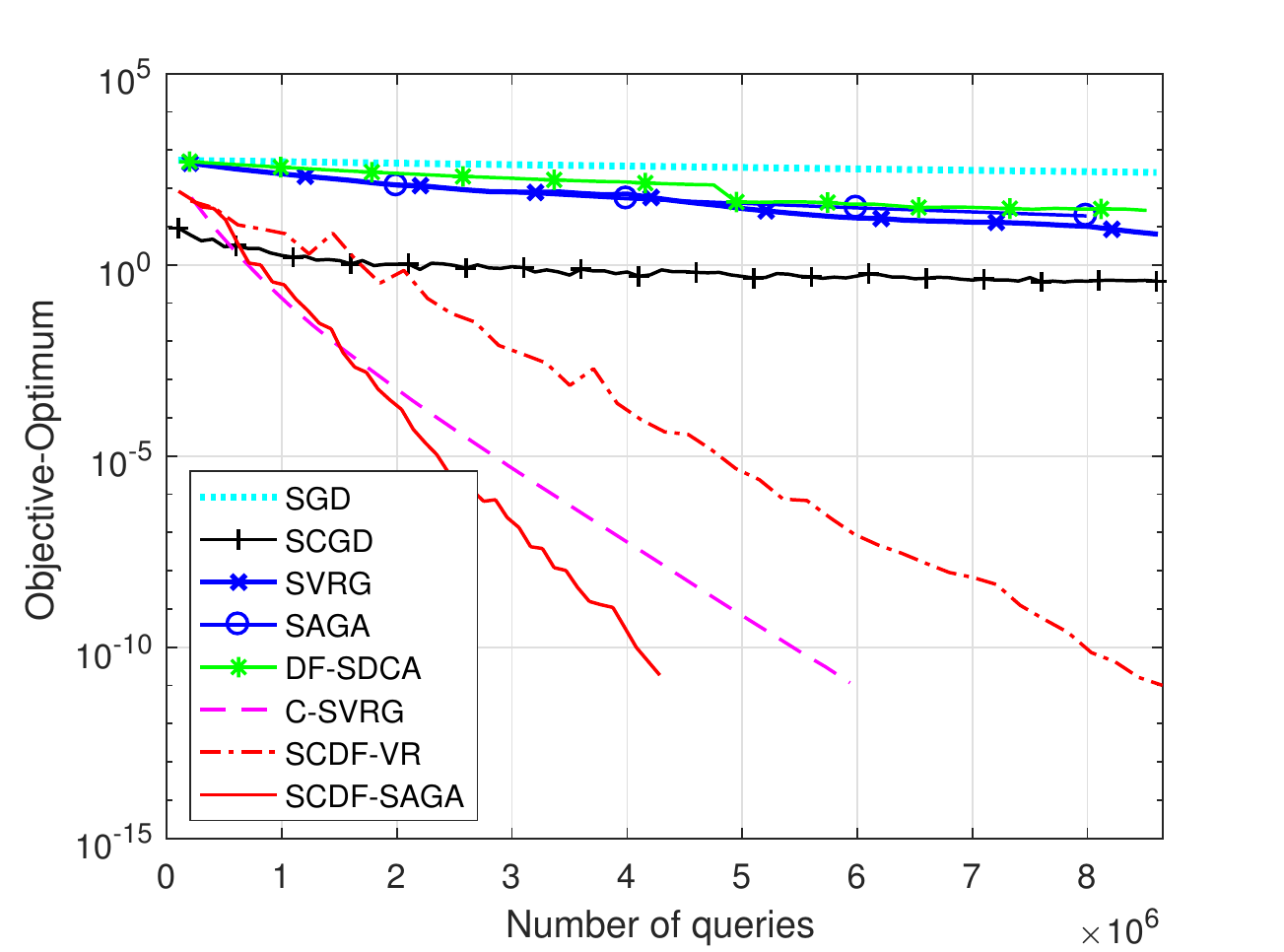}\\ 
			\includegraphics[width=1.0\textwidth]{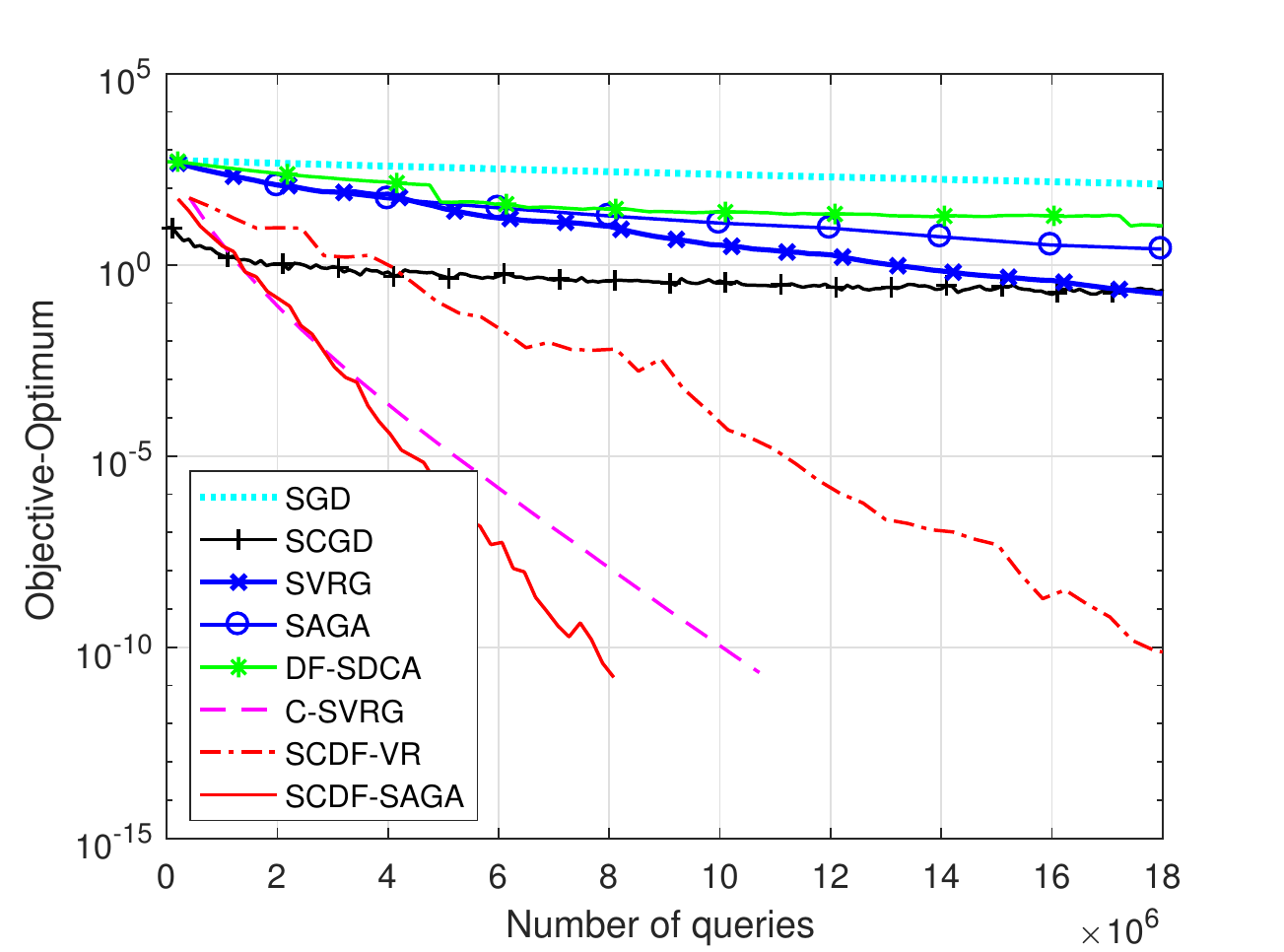}\\ 
			\includegraphics[width=1.0\textwidth]{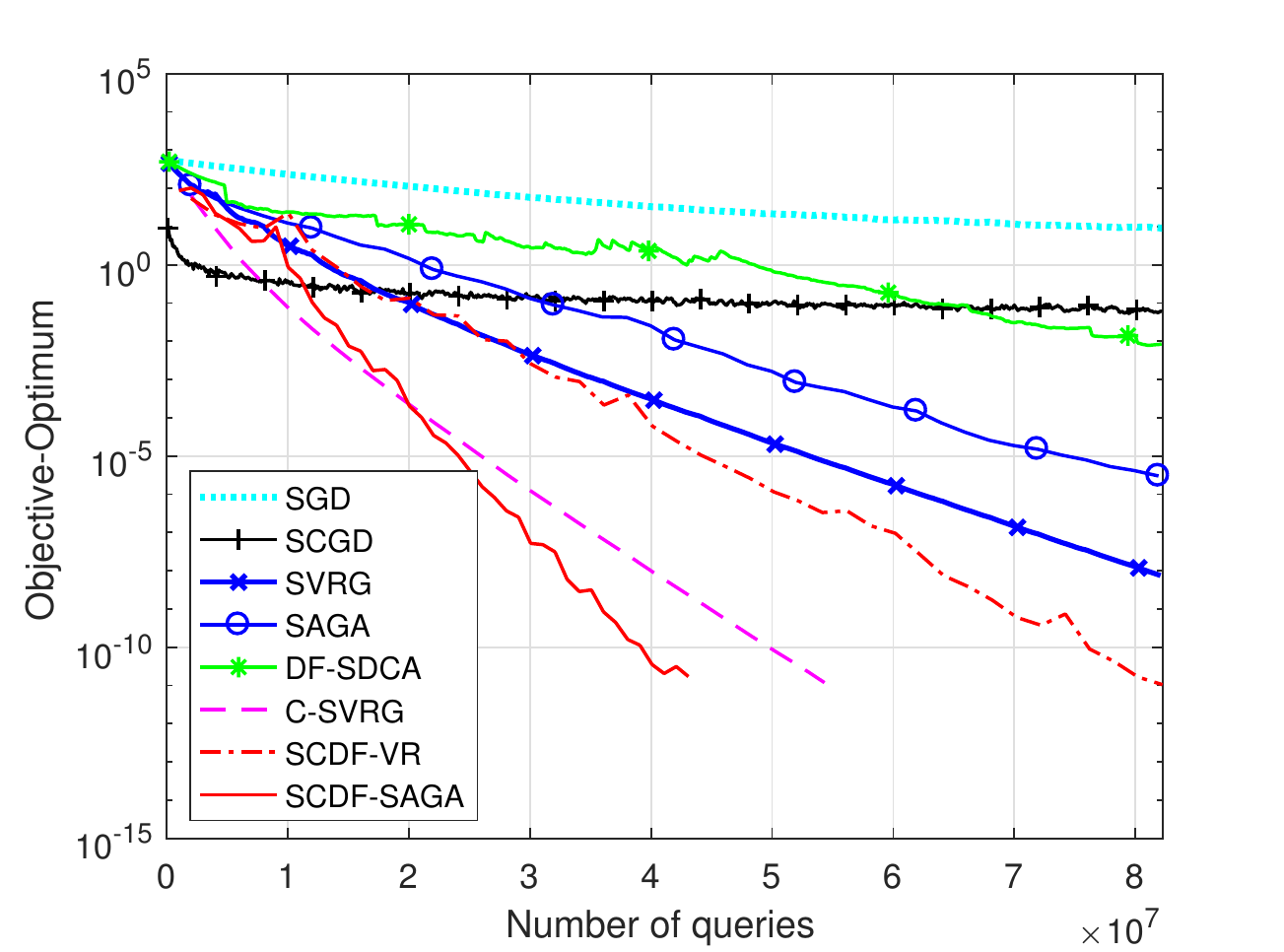} 
		\end{minipage}
	}
	\subfigure[$\kappa=30$]{
		\begin{minipage}[b]{0.32\textwidth}
			\includegraphics[width=1.0\textwidth]{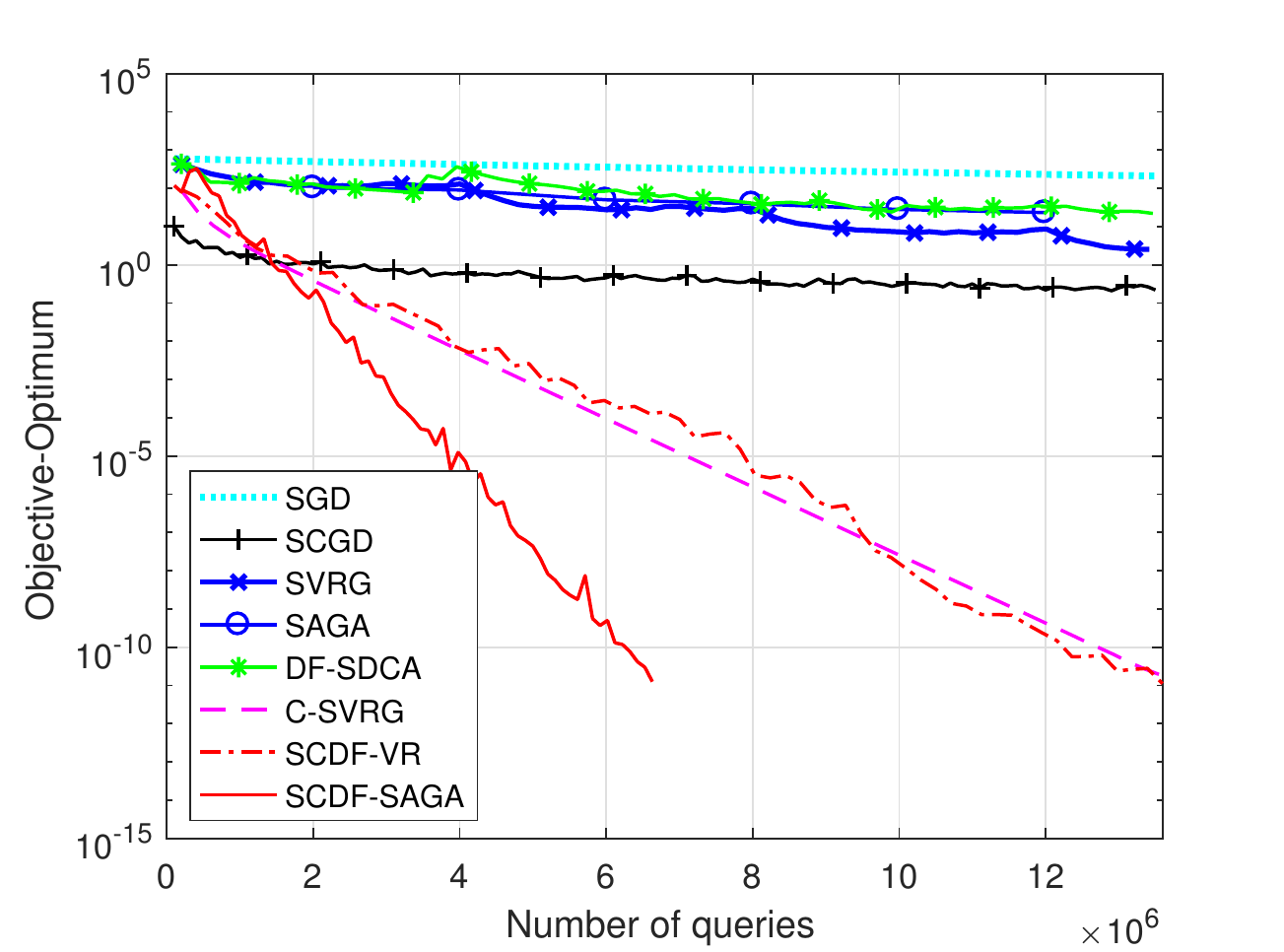}\\
			\includegraphics[width=1.0\textwidth]{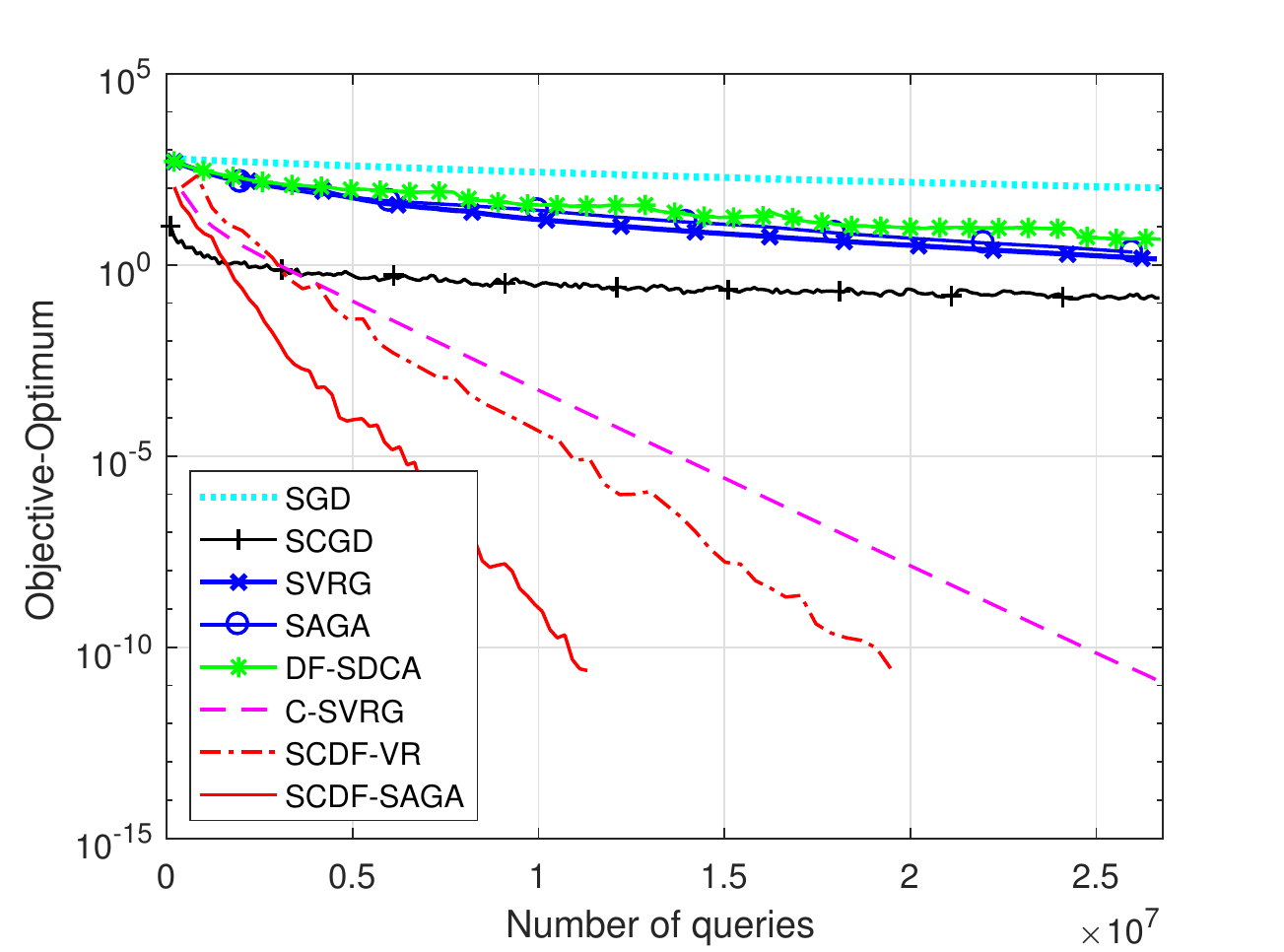}\\ 
			\includegraphics[width=1.0\textwidth]{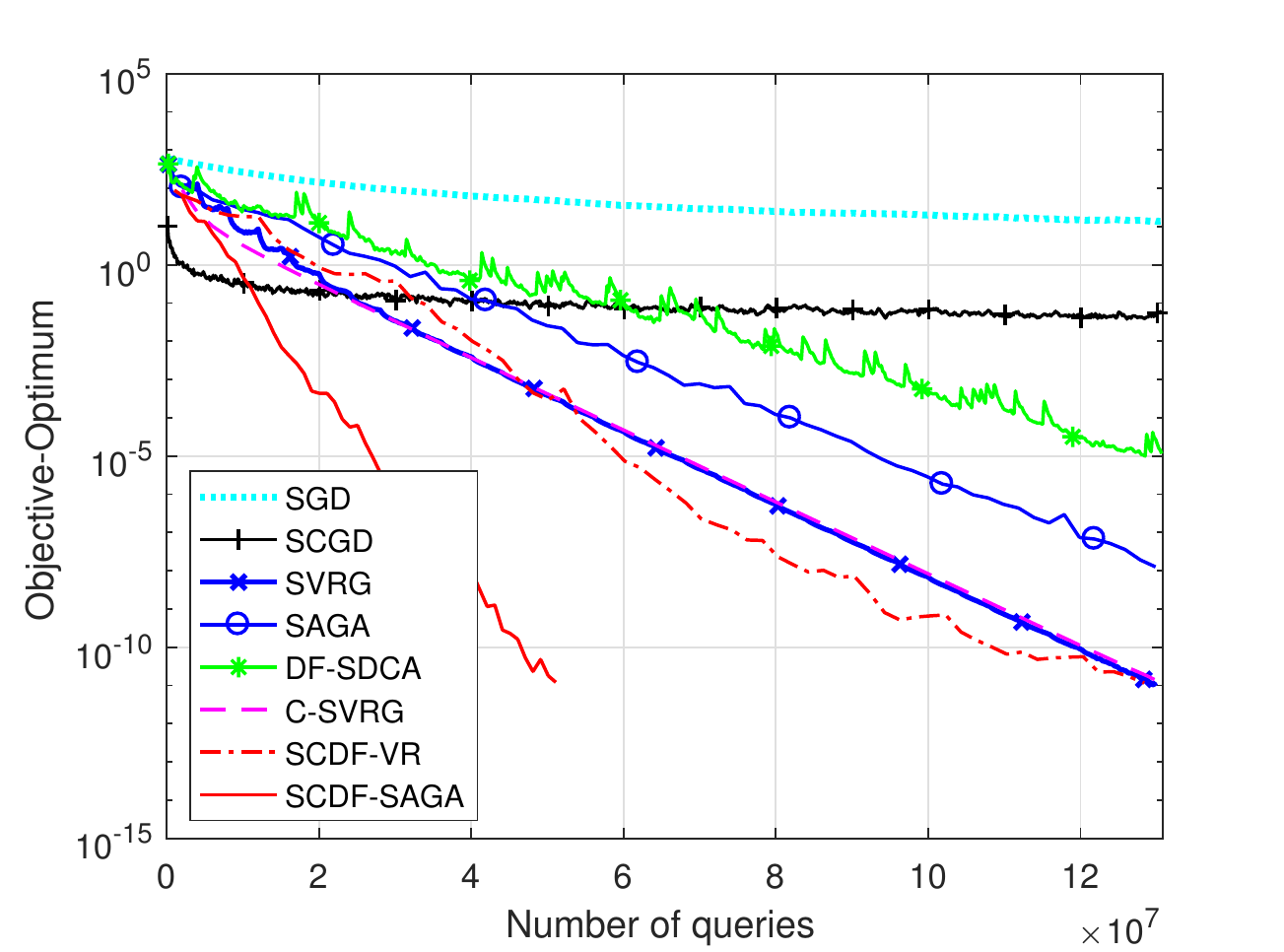} 
		\end{minipage}
	}
	\subfigure[$\kappa=50$]{
		\begin{minipage}[b]{0.32\textwidth}
			\includegraphics[width=1.0\textwidth]{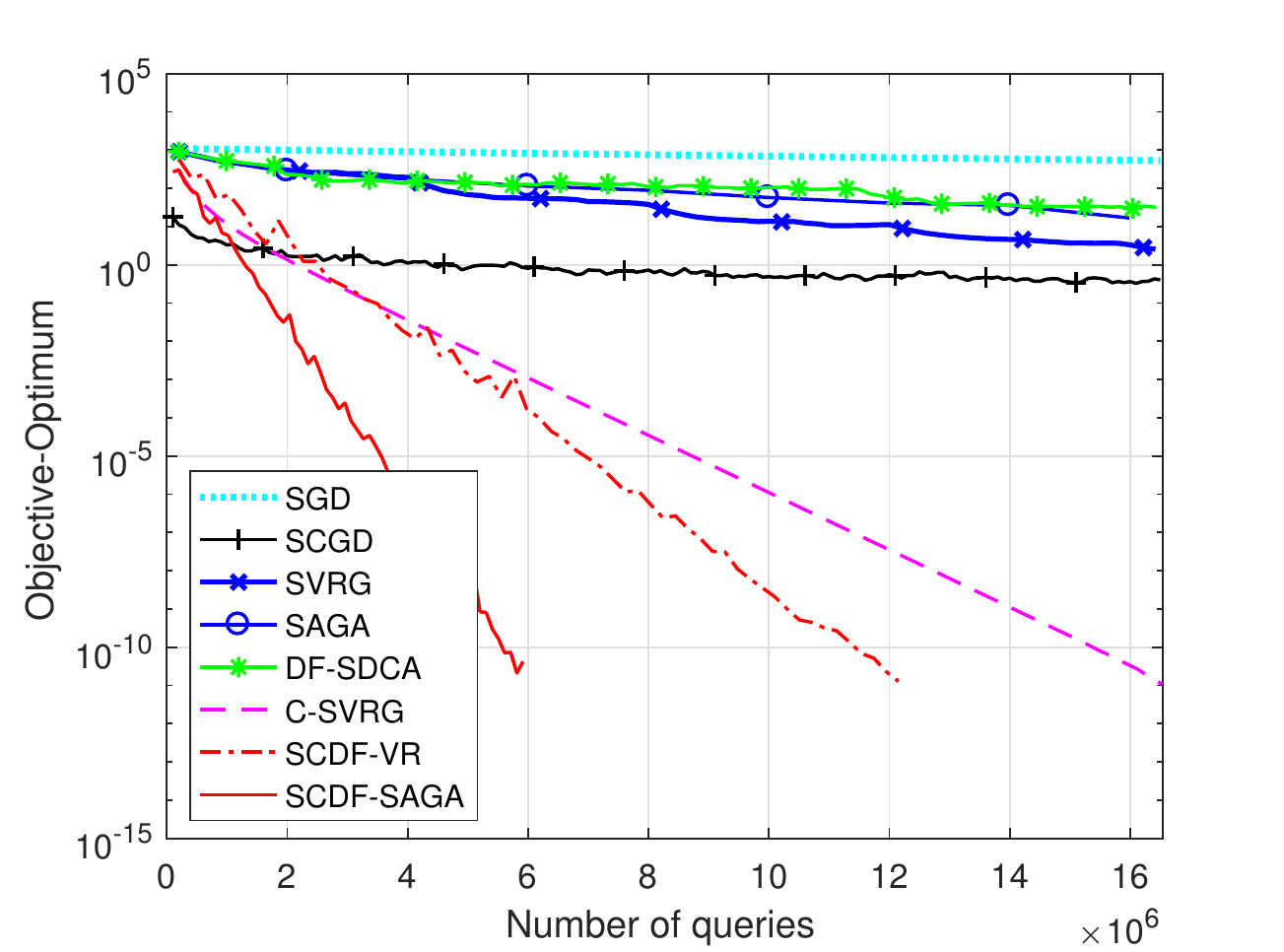}\\
			\includegraphics[width=1.0\textwidth]{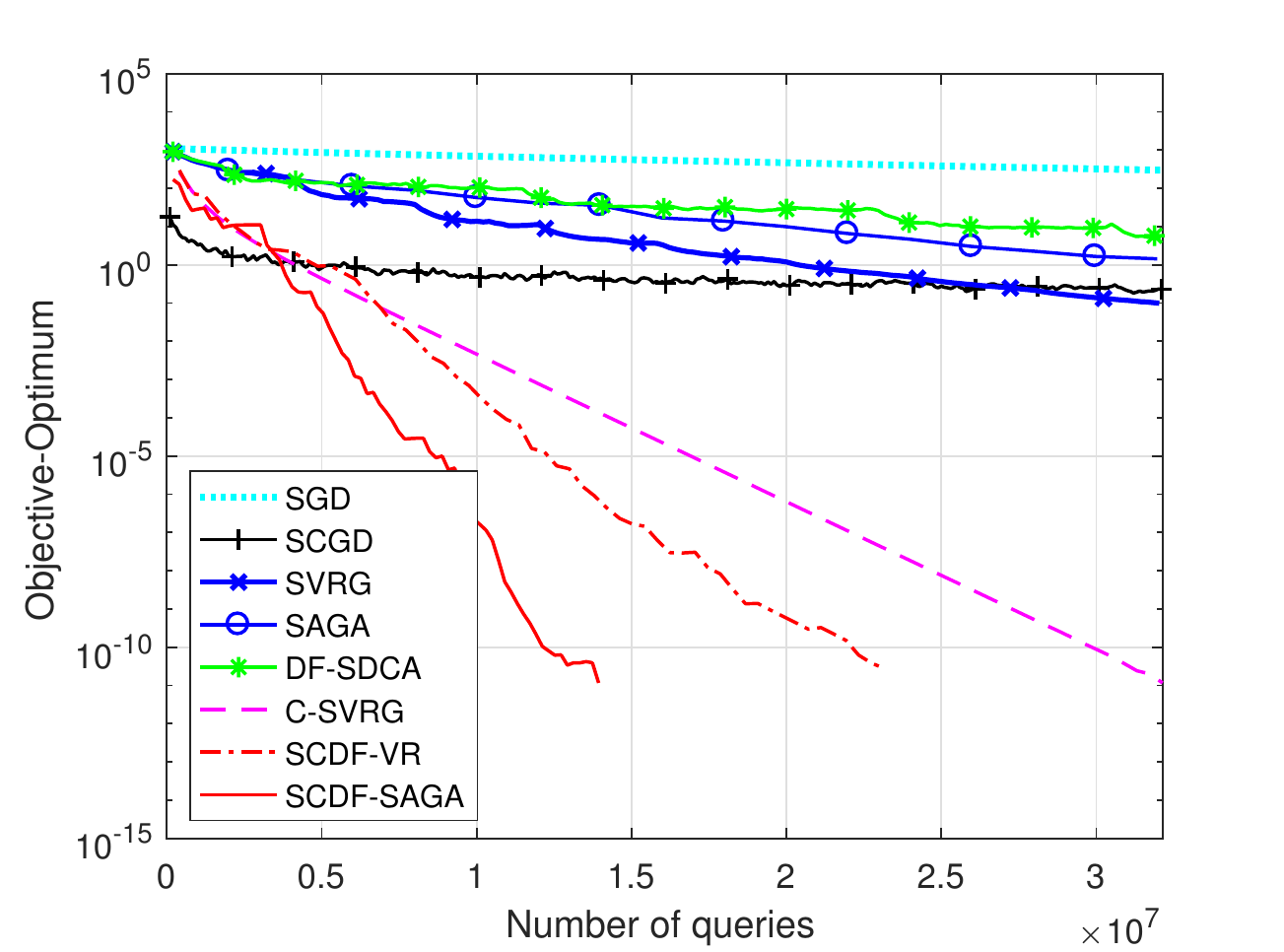}\\ 
			\includegraphics[width=1.0\textwidth]{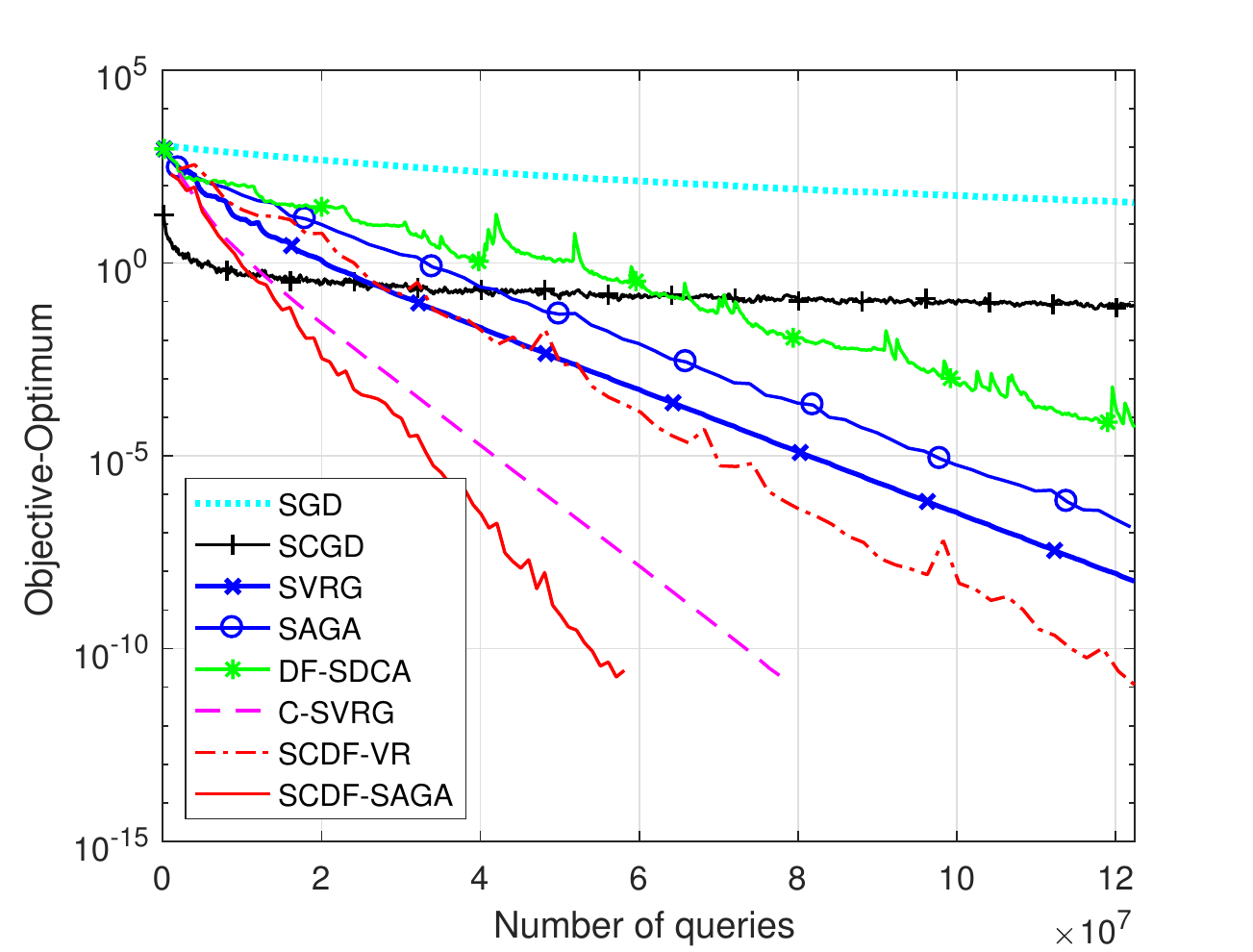} 
		\end{minipage}
	}
	\caption{As the SCDF-SAGA method also has geometric convergence in expectation, the variables $x_k$ and $\beta_k$ both convergence to the optimal solution iteratively. Because they control the upper bound of the  gradient, as indicated in Corollary, the variance of gradient decrease to zero. }
	\label{SCDF:Figure}	
\end{figure*}
	
	In this section, we  experiment with our two proposed algorithms and compare them with previous stochastic methods including SGD, SCGD, SVRG, SAGA, duality-free  SDCA (DF-SDCA) and  compositional-SVRG (C-SVRG).
	
	To verify the effectiveness of the algorithm, we use the mean-variance optimization in portfolio management:
	\begin{align*}
	\mathop {\max }\limits_{x \in {\mathbb{R}^d}} \frac{1}{n}\sum\limits_{i = 1}^n {\langle {{r_i},x} \rangle }  - \frac{1}{n}\sum\limits_{i = 1}^n ( \langle {{r_i},x} \rangle  - \frac{1}{n}\sum\limits_{i = 1}^n \langle {{r_i},x} \rangle   )^2,
	\end{align*}
	where $r_i\in\mathbb{R}^N, i\in[n]$ is the reward vector, and $x\in\mathbb{R}^N$ is the invested quantity. The goal is to maximize the objective function to obtain a large investment and reduce the investment risk. The objective function can be transformed as the composition of two finite-sum functions in \ref{SCDF:ProblemMainCompositionminimization} by the following form:
	\begin{align*}
	{G_j}( x ) =& {[ {x,\langle {{r_j},x} \rangle } ]^\mathsf{T}},\,y = \frac{1}{n}\sum\nolimits_{j = 1}^n {{G_j}( x )}  = {[ {{y_1},{y_2}} ]^\mathsf{T}},\\
	{F_i}( y ) =&  - \langle {{r_i},{y_1}} \rangle  + ( {\langle {{r_i},{y_1}} \rangle  - {y_2}} ),j,i \in [n].
	\end{align*} 
	where $y_1\in \mathbb{R}^M$ and $y_2\in\mathbb{R}$. We follow the un-regularized objective method in \cite{shalev2016sdca}, in which the term $L{\| w \|^2}/2$ is added or subtracted to the objective for DF-SDCA, SCDF-VR, and SCDF-SAGA, where  parameter $L$ can be obtained in advance and directly from the maximal eigenvalue of the Hessian matrix. We choose $n=2000$ and $N=200$ and conduct the experiment on the numerical simulations following \cite{lian2016finite}. Reward vectors $r_i$, $i\in [n]$ are generated from a random Gaussian distribution under different condition numbers of the corresponding covariance matrix, denoted  $\kappa$. We choose three different $\kappa=10,30$, and $50$. Furthermore, we give three different sample times for forming the mini-batch $\mathcal{A}$, $A=50,100$, and $500$. Figure \ref{SCDF:Figure} shows the results with different sample time $A$.  From Figure \ref{SCDF:Figure}, we can see that: 1) our proposed algorithms SCDF-SVRG and SCDF-SAGA both have linear convergence rates; and  2) SCDF-SAGA  outperforms the other algorithms.  

	\section{Conclusions}
In this paper, we  propose a new algorithm based on variance reduction technology and appli it to{ the composition of two finite-sum functions minimization problem}. Unlike most previous approaches, our work applies  duality-free SCDA to {compositional optimization} and tackles the primal and dual problems that cannot be solved directly by the primal-dual algorithm. We show  linear convergence in the situation that the estimator of the inner function is biased. Furthermore, we also show  a linear rate of convergence for the case in which the individual function is non-convex but the finite-sum function is strongly convex.

	\textbf{Appendix:}	

	\appendix

	\section{Analysis tool}
	\begin{lemma}\label{RandomVariable1}
		For the random variable $X$, we have
		\begin{align*}
		E[ {{{\| {X - E[ X ]} \|}^2}} ] { = } E[ {{X^2} - {{\| {E[ X ]} \|}^2}} ] \le E[ {{X^2}} ].
		\end{align*}
	\end{lemma}
	\begin{lemma}\label{RandomVariable2}
		For the random variable $X_1,...,X_r$, we have
		\begin{align*}
		E[ {{{\| {{X_1} + ... + {X_r}} \|}^2}} ] \le r( {E[ {{{\| {{X_1}} \|}^2}} ] + ... + [ {{{\| {{X_r}} \|}^2}} ]} ).
		\end{align*}
	\end{lemma}
	\begin{lemma}\label{LemmaInEquation} For $a$ and $b$, we have $2\langle {a,b} \rangle  \le \frac{1}{q}\| a \|^2 + q\| b \|^2,\forall q > 0$.
	\end{lemma}
	\begin{lemma}\label{LemmaAppendix} Suppose Assumption \ref{Assumption3} holds,  we have
		\begin{align*}
		\frac{1}{n}\sum\limits_{i = 1}^n {{{\| {{{(\partial G(x))}^\mathsf{T}}\nabla {F_i}(G(x)) + {{(\partial G({x^*}))}^\mathsf{T}}\nabla {F_i}(G({x^*}))} \|}^2}}  \le 2L_f\left( {P(x) - P({x^*}) - \frac{\lambda }{2}{{\| {x - {x^*}} \|}^2}} \right).
		\end{align*}
	\end{lemma}
	\begin{proof} Based on $L_F$-smoothness and convexity of $F_i$ in (\ref{InequationAssumption3}), we have
		\begin{align*}
		&\frac{1}{n}\sum\limits_{i = 1}^n {{{\| {{(\partial G(x))^\mathsf{T}}\nabla {F_i}(G(x)) - {(\partial G({x^*}))^\mathsf{T}}\nabla {F_i}(G({x^*}))} \|}^2}} \\
		\le &\frac{1}{n}\sum\limits_{i = 1}^n {2L_f( {{F_i}(G(x)) - \nabla {F_i}(G({x^*})) - \langle {{(\partial G({x^*}))^\mathsf{T}}\nabla {F_i}(G({x^*})),x - {x^*}} \rangle } )} \\
		=& 2L_f( {F(x) - F({x^*}) - \langle {F({x^*}),x - {x^*}} \rangle } )\\
		=& 2L_f( {F(x) + R(x) - F({x^*}) - R({x^*}) - R(x) + R({x^*}) - \langle {\nabla P({x^*}) + \nabla R({x^*}) - \nabla R({x^*}),x - {x^*}} \rangle } )\\
		=& 2L_f( {P(x) - P({x^*})} ) + \underbrace {2L( { - R(x) + R({x^*}) - \langle { - \nabla R({x^*}),x - {x^*}} \rangle } )}_{(App1)}\\
		=& 2L_f( {P(x) - P({x^*}) - \frac{\lambda }{2}{{\| {x - {x^*}} \|}^2}} ),
		\end{align*}
		where (\textit{App1}) is based on the smoothness of $R(x)$, that is $R(x) = \frac{1}{2}\lambda {\| x \|^2}$, the smooth constant is $\lambda$, then we have
		\begin{align*}
		- R(x) + R({x^*}) + \langle {\nabla R({x^*}),x - {x^*}} \rangle  \le&  - \frac{1}{{2L_f}}{\| {\nabla R({x^*}) - R(x)} \|^2}\\
		\le& -\frac{\lambda }{2}{\| {x - {x^*}} \|^2}.
		\end{align*}
	\end{proof}

	\section{Proof of SCDF-SVRG} 
\textbf{Proof of Lemma \ref{SCDF:LemmBoundSVRGEstimateFullGradientF}}
	\begin{proof}
		\begin{align*}
		&E[{\| {{(\partial {{\hat G}_k})^\mathsf{T}}\nabla {F_i}({{\hat G}_k}) - {(\partial {{\hat G}_k})^\mathsf{T}}\nabla {F_i}(G({x_k}))} \|^2}]\\
		\le& B_G^2E[{\| {\nabla {F_i}({{\hat G}_k}) - \nabla {F_i}(G({x_k}))} \|^2}]\\
		\le& B_G^2L_F^2E[{\| {{{\hat G}_k} - G({x_k})} \|^2}]\\
		\le& B_G^4L_F^2\frac{1}{A}E[{\| {{x_k} - {{\tilde x}_s}} \|^2}]\\
		\le& B_G^4L_F^2\frac{1}{A}E[{\| {{x_k} - {{ x}^*}} \|^2}]+B_G^4L_F^2\frac{1}{A}E[{\| { {{\tilde x}_s}} -x^*\|^2}],
		\end{align*}
		where the first and   the second  inequalities is based on the bounded Jacobian of $G$ and Lipschitz continuous gradient of $F$. The last inequality follows from Lemma \ref{LemmaSVRGBoundVarianceG}.
	\end{proof}
\textbf{Proof of Theorem \ref{SCDF:SVRG:TheoremSVRGMainConvergenceNonconvex}}
	\begin{proof}
		Based on Lemma \ref{LemmaMainBoundC}, we have
		\begin{align*}
		{C_k} - {C_{k - 1}} \le  - \eta \lambda {C_{k - 1}} + {d_2}E[ {\| {{{\tilde x}_s} - {x^*}} \|^2} ],
		\end{align*}
		where ${d_2}= 2\left( {a\eta qB_G^4L_F^2\frac{1}{A} + b\lambda \eta \left( {4B_F^2L_G^2\frac{1}{A} + 4B_G^4L_F^2} \right)} \right) + b\lambda \eta \left( {4B_F^2L_G^2 + 4B_G^4L_F^2} \right)$. Summing from k=0 to K, we obtain
		\begin{align*}
		{C_K} + \eta \lambda \sum\limits_{k = 1}^{K - 1} {{C_k}}  \le {C_0} + K{d_2}E[ {\| {{{\tilde x}_s} - {x^*}} \|^2} ].
		\end{align*}
		Since ${C_0} = {{\tilde C}_s}$ and ${{\tilde C}_{s + 1}} = \frac{1}{K}\sum\nolimits_{k = 1}^K {{C_k}} $, and $\eta \lambda  \le 1$, we have
		\begin{align*}
		\eta \lambda K{{\tilde C}_{s + 1}} = \eta \lambda \frac{1}{K}\sum\limits_{k = 1}^K {{C_k}}  \le {C_0} + K{d_2}E[ {\| {{{\tilde x}_s} - {x^*}} \|^2} ].
		\end{align*}
		The definition of ${{\tilde C}_s}$ implies that $aE[ {\| {{{\tilde x}_s} - {x^*}} \|^2} ]\le{{\tilde C}_s}  $. Therefore we have
		\begin{align*}
		\eta \lambda K{{\tilde C}_{s + 1}} \le {C_0} + \frac{{K{d_2}}}{a}{{\tilde C}_s}.
		\end{align*}
		Dividing both sides of the inequality by $\eta \lambda K$, we can obtain the linear convergence,
		\begin{align*}
		{{\tilde C}_{s + 1}} \le \left( {\frac{1}{{\eta \lambda K}} + \frac{{{d_2}}}{{a\eta \lambda }}} \right){{\tilde C}_s} \le {\left( {\frac{1}{{\eta \lambda K}} + \frac{{{d_2}}}{{a\eta \lambda }}} \right)^s}{{\tilde C}_0}.
		\end{align*}
	\end{proof}
	\textbf{Proof of Theorem \ref{SCDF:SVRG:TheoremSVRGMainConvergenceConvex}}
	\begin{proof}
		The proof process is similar to Theorem \ref{SCDF:SVRG:TheoremSVRGMainConvergenceNonconvex}, but based on different inner estimation bound from Lemma \ref{LemmaSVRGLemmaMainBoundCConvex}
	\end{proof}
	
	\textbf{Proof of Corollary \ref{CorollarySVRGGradient}}
	\begin{proof} Based on Lemma \ref{RandomVariable2}, we have
		\begin{align*}
		&E[ {\| {{( {\partial {{\hat G}_k}} )^\mathsf{T}}\nabla {F_i}( {{{\hat G}_k}} ) + \beta _i^k} \|^2} ]\\
		\le& 2E[ {\| {{( {\partial {{\hat G}_k}} )^\mathsf{T}}\nabla {F_i}( {{{\hat G}_k}} ) + \beta _i^*} \|^2} ] + 2E[ {\| {\beta _i^k - \beta _i^*} \|^2} ]\\
		\le&\left ( {4B_F^2L_G^2\frac{1}{A} + 4B_G^4L_F^2\frac{1}{A}} \right)E[ {\| {{x_k} - {{\tilde x}_s}} \|^2} ] + ( {4B_F^2L_G^2 + 4B_G^4L_F^2} )E[ {\| {{{\tilde x}_s} - {x^*}} \|^2} ] + E[ {\| {\beta _i^k - \beta _i^*} \|^2} ],
		\end{align*}
		where the first inequality follows from Lemma \ref{LemmaBoundSVRGestimateFullGradient}.
	\end{proof}
	
	\section{Proof of SCDF-SAGA}

	\textbf{Proof of Lemma \ref{LammaBoundSAGANormEsimateGradientShort}}
	\begin{proof}
		Through subtracting and adding ${{(\partial G({x_k}))^\mathsf{T}}\nabla {F_i}({{\hat G}_k})}$, we have
		\begin{align*}
		&E[ {\| {{(\partial {{\hat G}_k})^\mathsf{T}}\nabla {F_i}({{\hat G}_k}) - {(\partial G({x_k}))^\mathsf{T}}\nabla {F_i}({G_k}({x_k}))} \|^2} ]\\
		=& E[ {\| {{(\partial {{\hat G}_k})^\mathsf{T}}\nabla {F_i}({{\hat G}_k}) - {(\partial G({x_k}))^\mathsf{T}}\nabla {F_i}({{\hat G}_k}) + {(\partial G({x_k}))^\mathsf{T}}\nabla {F_i}({{\hat G}_k}) - {(\partial G({x_k}))^\mathsf{T}}\nabla {F_i}({G_k}({x_k}))} \|^2} ]\\
		\le& 2E[ {\| {{(\partial {{\hat G}_k})^\mathsf{T}}\nabla {F_i}({{\hat G}_k}) - (\partial G{({x_k})^\mathsf{T}}\nabla {F_i}({{\hat G}_k})} \|^2} ]\\
		&+ 2E[ {\| {{(\partial G({x_k}))^\mathsf{T}}\nabla {F_i}({{\hat G}_k}) - {(\partial G({x_k}))^\mathsf{T}}\nabla {F_i}({G_k}({x_k}))} \|^2} ]\\
		\le& 2B_F^2E[ {\| {\partial {{\hat G}_k} - \partial G({x_k})} \|^2} ] + 2B_G^2E[ {\| {\nabla {F_i}({{\hat G}_k}) - \nabla {F_i}(\partial G({x_k}))} \|^2} ]\\
		\le& 2B_F^2E[ {\| {\partial {{\hat G}_k} - \partial G({x_k})} \|^2} ] + 2B_G^2L_F^2E[ {\| {{{\hat G}_k} - G({x_k})} \|^2} ]\\
		\le& 2B_F^2L_G^2\frac{1}{{{A^2}}}\sum\limits_{1 \le j \le A}^{} {E[ {\| {{x_k} - \phi _{{{\cal A}_k}[j]}^k} \|_{}^2} ]}  + 2B_G^2L_F^2B_G^2\frac{1}{{{A^2}}}\sum\limits_{1 \le j \le A}^{} {E[ {\| {{x_k} - \phi _{{{\cal A}_k}[j]}^k} \|_{}^2} ]} \\
		=& \left( {2B_F^2L_G^2 + 2B_G^4L_F^2} \right)\frac{1}{{{A^2}}}\sum\limits_{1 \le j \le A}^{} {E[ {\| {{x_k} - \phi _{{{\cal A}_k}[j]}^k} \|_{}^2} ]} 
		\end{align*}
		where the first inequality follows from Lemma \ref{RandomVariable2}, The second and third are based on the bounded gradient of $F_i$ in (\ref{InequationAssumptionF1}), the bounded Jacobian of $G$ in (\ref{InequationAssumptionG1}) and  Lipschitz continuous gradient of $F$ in (\ref{InequationAssumptionF2}), The last inequality follows from the Lemma \ref{LemmaSAGABoundFunctionG} and Lemma \ref{LemmaSAGABoundGradientG}.	
	\end{proof}
	\textbf{Proof of Theorem \ref{SCDF:SAGA:TheoremSAGAMainConvergenceNonconvex}}
	\begin{proof}
		Based on Lemma \ref{LemmaBoundMainA}, Lemma \ref{LemmaBoundMainB} and Lemma \ref{LemmaBoundMainC}, let ${D_k} = aE[ {{A_k}} ] + bE[ {{B_k}} ] + cnE[ {{C_k}} ]$, we have 
		\begin{align*}
		&{D_{k + 1}} - {D_k}\\
		=& a( {E[ {{A_{k + 1}}} ] - E[ {{A_k}} ]} ) + b( {E[ {{B_{k + 1}}} ] - E[ {{B_k}} ]} ) + c( {nE[ {{C_{k + 1}}} ] - nE[ {{C_k}} ]} )\\
		\le&  - a\eta \lambda E\left[ {{A_k}} \right] - b\lambda \eta E\left[ {{B_k}} \right] - c\lambda \eta nE\left[ {{C_k}} \right]\\
		&+ \underbrace {\left( {8a\lambda \eta {R_x}\left( {B_F^2L_G^2 + B_G^4L_F^2} \right)\frac{1}{A} - a\eta \lambda  + 2\lambda \eta b\left( {B_F^2L_G^2\frac{1}{A} + B_G^4L_F^2} \right) + cA} \right)}_{{E_1}}E\left[ {{A_k}} \right]\\
		&+ \underbrace {\left( {8a\lambda \eta {R_x}\left( {B_F^2L_G^2 + B_G^4L_F^2} \right)\frac{1}{A} + 2b\lambda \eta B_F^2L_G^2\frac{1}{A} - cA + c\lambda \eta n} \right)}_{{E_2}}E\left[ {{C_k}} \right]\\
		&+ \underbrace {( {a{\eta ^2} - (1 - \lambda n\eta )b\lambda \eta } )}_{{E_3}}{\| {{(\partial {{\hat G}_k})^\mathsf{T}}\nabla {F_i}({{\hat G}_k}) + \beta _i^k} \|^2},
		\end{align*}
		by setting the last terms $E_1$, $E_2$ and $E_3$ negative. Thus,  we can obtain 
		\begin{align*}
		{D_{k + 1}} - {D_k}\le - \lambda \eta {D_k},
		\end{align*}
		In order to simply the analysis, we define ${Y_1} = {R_x}\left( {B_F^2L_G^2 + B_G^4L_F^2} \right)\frac{1}{A},{Y_2} = B_F^2L_G^2\frac{1}{A} + B_G^4L_F^2,{Y_3} = B_F^2L_G^2\frac{1}{A}$. Both $E_1$ and $E_2$ are negative, we get 
		\begin{align}\label{StepSAGAMainTheorem1}
		\frac{a}{b} \ge \frac{{2{Y_2} + \frac{A}{{A - \lambda \eta n}}2{Y_3}}}{{1 - 8\left( {1 + \frac{A}{{A - \lambda \eta n}}} \right){Y_1}}}.
		\end{align}
		To keep the  bound positive, that is $1 - 8{Y_1} - 8\frac{A}{{A - \lambda \eta }}Y \ge 0$, the sample times $A$ satisfy,
		\begin{align*}
		A \ge \frac{{\left( {\lambda \eta n + 2{R_x}\left( {B_F^2L_G^2 + B_G^4L_F^2} \right)} \right) + \sqrt {{\lambda ^2}{\eta ^2}{n^2} + 4{{\left( {{R_x}\left( {B_F^2L_G^2 + B_G^4L_F^2} \right)} \right)}^2}} }}{2}.
		\end{align*}
		Based on above condition in (\ref{StepSAGAMainTheorem1}) and $E_3\le 0$, we have
		\begin{align*}
		\frac{{2{Y_2} + \frac{A}{{A - \lambda \eta n}}2{Y_3}}}{{1 - 8\left( {1 + \frac{A}{{A - \lambda \eta n}}} \right){Y_1}}} \le \frac{{(1 - \lambda n\eta )\lambda }}{\eta },
		\end{align*}		
		Thus, we get 
		\begin{align*}					
		\eta  \le \frac{\lambda }{{2{Y_2} + \frac{A}{{A - \lambda \eta n}}2{Y_3} + {\lambda ^2}n\left( {1 - 8\left( {1 + \frac{A}{{A - \lambda \eta n}}} \right){Y_1}} \right)}}.
		\end{align*}
		Finally, we can obtain the convergence form,
		\begin{align*}
		aE\left[ {{A_{k + 1}}} \right] + bE\left[ {{B_{k + 1}}} \right] + cnE\left[ {{C_{k + 1}}} \right] \le {\left( {1 - \eta\lambda} \right)^k}\left( {aE\left[ {{A_1}} \right] + bE\left[ {{B_1}} \right] + cnE\left[ {{C_1}} \right]} \right),
		\end{align*}
	\end{proof}
	\textbf{Proof of Theorem \ref{SCDF:SAGA:TheoremSAGAMainConvergenceConvex}}
	\begin{proof}
		Based on Lemma \ref{LemmaSAGABoundMainAconvex}, Lemma \ref{LemmaBoundSAGAMainBConvex} and Lemma \ref{LemmaBoundMainC}, let ${D_k} = aE[ {{A_k}} ] + bE[ {{B_k}} ] + cnE[ {{C_k}} ]$, we have 
		\begin{align*}
		&{D_{k + 1}} - {D_k}\\
		=& a\left( {E\left[ {{A_{k + 1}}} \right] - E\left[ {{A_k}} \right]} \right) + b\left( {E\left[ {{B_{k + 1}}} \right] - E\left[ {{B_k}} \right]} \right) + c\left( {nE\left[ {{C_{k + 1}}} \right] - nE\left[ {{C_k}} \right]} \right)\\
		\le&  - a\eta \lambda E[ {{A_k}} ] - b\lambda \eta E[ {{B_k}} ] - c\lambda \eta nE[ {{C_k}} ]\\
		&+ \underbrace {\left( {8a\eta \lambda {R_x}( {B_F^2L_G^2 + B_G^4L_F^2} )\frac{1}{A} - a\eta d\lambda  + 4b\lambda \eta ( {B_F^2L_G^2 + B_G^4L_F^2} )\frac{1}{A} - 2b{L_f}{\lambda ^2}\eta  + cA} \right)}_{{E_1}}E[ {{A_k}} ]\\
		&+ \underbrace {\left( {8a\eta \lambda {R_x}( {B_F^2L_G^2 + B_G^4L_F^2} )\frac{1}{A} + 4b\lambda \eta ( {B_F^2L_G^2 + B_G^4L_F^2} )\frac{1}{A} - cA + c\lambda \eta n} \right)}_{{E_2}}E[ {{C_k}} ]\\
		&+ \underbrace {( {a{\eta ^2} - (1 - \lambda n\eta )b\lambda \eta } )}_{{E_3}}{\| {{(\partial {{\hat G}_k})^\mathsf{T}}\nabla {F_i}({{\hat G}_k}) + \beta _i^k} \|^2} + \underbrace {( { - 2a( {1 - d} )\eta  + 4b{L_f}\lambda \eta } )}_{{E_4}}(P({x_k}) - P({x^*})).
		\end{align*}
		by setting the last four terms negative,  we can obtain 
		\begin{align*}
		{D_{k + 1}} - {D_k}\le  - \lambda \eta {D_k}.
		\end{align*}
		Define $Y = \left( {B_F^2L_G^2 + B_G^4L_F^2} \right)\frac{1}{A}$ for simply analysis. Based on $E_1$ and $E_2$ that should be negative, we have
		\begin{align}\label{StepSAGAMainTheorem1Convex}
		\frac{a}{b} \ge \frac{{4Y + \frac{A}{{A - \lambda \eta n}}4Y - 2{L_f}\lambda }}{{d - 8\left( {1 + \frac{A}{{A - \lambda \eta n}}} \right){R_x}Y}},
		\end{align}
		In order to keep the bound positive, the sample times $A$ should satisfy 
		\begin{align*}
		A \ge ( {2 + \sqrt 2 } )\left( {\lambda \eta n + \frac{{16{R_x}( {B_F^2L_G^2 + B_G^4L_F^2} )}}{d}} \right).
		\end{align*}
		Based on $E_3$ and $E_4$ that should be negative, we have
		\begin{align*} \label{StepSAGAMainTheorem1Convex2}
		\frac{{2{L_f}\lambda }}{{\left( {1 - d} \right)}} \le \frac{a}{b} \le \frac{{(1 - \lambda n\eta )\lambda }}{\eta }.
		\end{align*}
		Thus, we can get the upper bound of the step
		\begin{align*}
		\eta  \le \frac{1}{{\left( {\frac{{2{L_f}\lambda }}{{\left( {1 - d} \right)}} + \lambda n} \right)}},
		\end{align*}
		where the parameter $d$  satisfy,
		\begin{align*}
		d \le \frac{{\left( {4Y + \frac{A}{{A - \lambda \eta n}}4Y - 2{L_f}\lambda } \right) + 8\left( {1 + \frac{A}{{A - \lambda \eta n}}} \right){R_x}Y2{L_F}\lambda }}{{4Y + \frac{A}{{A - \lambda \eta n}}4Y  }}.
		\end{align*}
		
		Thus we can obtain the convergence form
		\begin{align*}
		aE\left[ {{A_{k + 1}}} \right] + bE\left[ {{B_{k + 1}}} \right] + cnE\left[ {{C_{k + 1}}} \right] \le {\left( {1 - \eta\lambda} \right)^k}\left( {aE\left[ {{A_1}} \right] + bE\left[ {{B_1}} \right] + cnE\left[ {{C_1}} \right]} \right),
		\end{align*}
	\end{proof}

	Note that as the variable $x_k$ and $\beta^k$ go to the optimal solution, we can see that the variance of gradient in the update iteration is also approximating to zero. The following Corollary shows the bound of the estimated gradient variance
	
	\textbf{Proof of Corollary \ref{CorollarySAGAGradient}}

	\begin{proof} Based on the update of $x_k$, we have
		\begin{align*}
		&\frac{1}{{{\eta ^2}}}E[ {{{\| {( {{x_{k + 1}} - {x_k}} )} \|}^2}} ]\\
		=&E[ {\| {{( {\partial {{\hat G}_{k}}} )^\mathsf{T}}\nabla {F_i}( {{{\hat G}_{k }}} ) + \beta _i^k} \|^2} ]\\
		\le & 2E[ {\| {{( {\partial {{\hat G}_{k}}} )^\mathsf{T}}\nabla {F_i}( {{{\hat G}_{k }}} ) + \beta _i^*} \|^2} ]+2E[\|\beta^k-\beta^*\|^2]\\
		\le & 4\left( {B_F^2L_G^2\frac{1}{A} + B_G^4L_F^2} \right)E[ {\| {{x_k} - {x^*}} \|_{}^2} ]+ 4B_F^2L_G^2\frac{1}{{{A^2}}}\sum\limits_{1 \le j \le A}^{} {E[ {\| {\phi _{{\mathcal{A}_k}[j]}^k - {x^*}} \|_{}^2} ]}+2E[\|\beta^k-\beta^*\|^2],
		\end{align*}
		where the first and second inequalities follows from Lemma \ref{LammaBoundSAGAgradientAndOptimal} and Lemma \ref{RandomVariable2}.
	\end{proof}

	\section{Convergence Bound Analysis for SDFC-SVRG}
	\subsection{Bounding the estimation of inner function $G$}
	\begin{lemma}\label{LemmaSVRGBoundVarianceG}
		Assumption \ref{Assumption2} holds, in algorithm \ref{AlgorithmSDFCVRG1}, for the intermediated iteration of $x_k$ and $\tilde{x}_s$, and $\hat{G}_k$ defined in (\ref{SCDF:SVRG:DefinitionSVRGFunctionG}),  the variance of stochastic gradient is ,
		\begin{align*}
		E[ {\| {{{\hat G}_k} - G( {{x_{k }}} )} \|^2} ]	\le B_G^2\frac{1}{A}E[ {\| {{x_{k }} - {{\tilde x}_s}} \|^2} ],
		\end{align*}
		where $B_G$ is the parameter in (\ref{InequationAssumptionG1}).
	\end{lemma}
	\begin{proof}
		Based on the bounded Jacobian of $G_j$ and Lipschitz continuous gradient of $F_i$, $j \in [m]$, $i \in [n]$, we have
		\begin{align*}
		& E[ {\| {{{\hat G}_k} - G( {{x_{k }}} )} \|^2} ]\\
		=& E[ {\| {\frac{1}{A}\sum\limits_{1 \le j \le A}^{} {( {{G_{{{\cal A}_{k }}[j]}}( {{x_{k }}} ) - {G_{{{\cal A}_{k }}[j]}}( {{{\tilde x}_s}} )} )}  + G( {{{\tilde x}_s}} ) - G( {{x_{k }}} )} \|^2} ]\\
		\le& \frac{1}{{{A^2}}}\sum\limits_{1 \le j \le A}^{} { {E[ {\| {{G_{{{\cal A}_{k}}[j]}}( {{x_{k }}} ) - {G_{{{\cal A}_{k }}[j]}}( {{{\tilde x}_s}} ) - ( {G( {{x_{k }}} ) - G( {{{\tilde x}_s}} )} )} \|^2} ]} } \\
		\le& \frac{1}{{{A^2}}}\sum\limits_{1 \le j \le A}^{} { {E[ {\| {{G_{{{\cal A}_{k}}[j]}}( {{x_{k}}} ) - {G_{{{\cal A}_{k}}[j]}}( {{{\tilde x}_s}} )} \|^2} ]} } \\
		\le& B_G^2\frac{1}{A}E[ {\| {{x_{k }} - {{\tilde x}_s}} \|^2} ],
		\end{align*}
		where the first and second  inequalities follow from Lemma (\ref{RandomVariable1}) and (\ref{RandomVariable2}), and the third  inequality is based on the bounded Jacobian of $G$ in (\ref{InequationAssumptionG1}).
	\end{proof}
	\begin{lemma}\label{LemmaSVRGBoundVarianceGradientG}
		Assumption \ref{Assumption2} holds, in algorithm \ref{AlgorithmSDFCVRG1}, for the intermediated iteration of $x_k$ and $\tilde{x}_s$, and $\partial\hat{G}_k$ defined in (\ref{SCDF:SVRG:DefinitionSVRGEstimateG}),  the variance of stochastic gradient is ,
		\begin{align*}
		E[ {\| {\partial{{\hat G}_k} - \partial G( {{x_{k }}} )} \|^2} ]	\le L_G^2\frac{1}{A}E[ {\| {{x_{k }} - {{\tilde x}_s}} \|^2} ],
		\end{align*}
		where $B_G$ is the parameter in (\ref{InequationAssumptionG1}).
	\end{lemma}
	\begin{proof}
		Based on the bounded Jacobian of $G_j$ and Lipschitz continuous gradient of $F_i$, $j \in [m]$, $i \in [n]$, we have
		\begin{align*}
		& E[ {\| {\partial{{\hat G}_k} - \partial G( {{x_{k }}} )} \|^2} ]\\
		=& E[ {\| {\frac{1}{A}\sum\limits_{1 \le j \le A}^{} {( {\partial{G_{{{\cal A}_{k }}[j]}}( {{x_{k }}} ) - \partial{G_{{{\cal A}_{k }}[j]}}( {{{\tilde x}_s}} )} )}  + \partial G( {{{\tilde x}_s}} ) -\partial G( {{x_{k }}} )} \|^2} ]\\
		\le& \frac{1}{{{A^2}}}\sum\limits_{1 \le j \le A}^{} { {E[ {\| {\partial{G_{{{\cal A}_{k}}[j]}}( {{x_{k }}} ) - \partial{G_{{{\cal A}_{k }}[j]}}( {{{\tilde x}_s}} ) - ( {\partial G( {{x_{k }}} ) - \partial G( {{{\tilde x}_s}} )} )} \|^2} ]} } \\
		\le& \frac{1}{{{A^2}}}\sum\limits_{1 \le j \le A}^{} { {E[ {\| {\partial{G_{{{\cal A}_{k}}[j]}}( {{x_{k}}} ) - {\partial G_{{{\cal A}_{k}}[j]}}( {{{\tilde x}_s}} )} \|^2} ]} } \\
		\le& L_G^2\frac{1}{A}E[ {\| {{x_{k }} - {{\tilde x}_s}} \|^2} ],
		\end{align*}
		where the first and second  inequalities follow from Lemma (\ref{RandomVariable1}) and (\ref{RandomVariable2}), and the third  inequality is based on the bounded Jacobian of $G$ in (\ref{InequationAssumptionG1}).
	\end{proof}
	\begin{lemma}\label{LemmaSVRGBoundestimateGAndOptimal}
		Assumption \ref{Assumption2} holds, in algorithm \ref{AlgorithmSDFCVRG1}, for the intermediated iteration of $x_k$ and $\tilde{x}_s$ and $\hat{G}_k$ defined in (\ref{SCDF:SVRG:DefinitionSVRGFunctionG}),  the bound satisfies,
		\begin{align*}
		E[ {\| {{{\hat G}_{k}} - G( {{x^*}} )} \|^2} ]\le {2B_G^2\frac{1}{A}E[ {\| {{x_{k }} - {{\tilde x}_s}} \|^2} ] + 2B_G^2E[ {\| {{{\tilde x}_s} - {x^*}} \|^2} ]},
		\end{align*}
		where $B_G$ is the parameter in (\ref{InequationAssumptionG1}).
	\end{lemma}
	\begin{proof}From the definition of $\hat{G}_k$  in (\ref{SCDF:SVRG:DefinitionSVRGFunctionG}),we have
		\begin{align*}
		&E[ {\| {{{\hat G}_{k}} - G( {{x^*}} )} \|^2} ]\\
		=&E[ {\| {\frac{1}{A}\sum\limits_{1 \le j \le A}^{} {( {{G_{{{\cal A}_{k }}[j]}}( {{x_{k }}} ) - {G_{{{\cal A}_{k }}[j]}}( {{{\tilde x}_s}} )} )}  + G( {{{\tilde x}_s}} ) - G( {{x^*}} )} \|^2} ]\\
		\le&  {2E[ {\| {\frac{1}{A}\sum\limits_{1 \le j \le A}^{} {( {{G_{{{\cal A}_{k }}[j]}}( {{x_{k }}} ) - {G_{{{\cal A}_{k}}[j]}}( {{{\tilde x}_s}} )} )} } \|^2} ] + 2E[ {\| {G( {{{\tilde x}_s}} ) - G( {{x^*}} )} \|^2} ]} \\
		\le& 2\frac{1}{{{A^2}}}\sum\limits_{1 \le j \le A}^{} {E{{\| {{G_{{A_k}[j]}}({x_k}) - {G_{{A_k}[j]}}({{\tilde x}_s})} \|}^2}}+ 2E[ {\| {G( {{{\tilde x}_s}} ) - G( {{x^*}} )} \|^2} ] \\
		\le& {2B_G^2\frac{1}{A}E[ {\| {{x_{k }} - {{\tilde x}_s}} \|^2} ] + 2B_G^2E[ {\| {{{\tilde x}_s} - {x^*}} \|^2} ]},
		\end{align*}
		where the first and the second inequalities follow from Lemma \ref{RandomVariable2}, and the third inequality is based on the bounded Jacobian of $G$ in (\ref{InequationAssumptionG1}).
	\end{proof}
	\begin{lemma}\label{LemmaSVRGBoundgradientGAndOptimal}
		Assumption \ref{Assumption2} holds, in algorithm \ref{AlgorithmSDFCVRG1}, for the intermediated iteration of $x_k$ and $\tilde{x}_s$ and $\partial \hat{G}_k$ defined in (\ref{SCDF:SVRG:DefinitionSVRGEstimateG}),  the bound satisfies,
		\begin{align*}
		E[ {\| {\partial {{\hat G}_k} - \partial G( {{x^*}} )} \|^2} ]\le {2L_G^2\frac{1}{A}E[ {\| {{x_{k }} - {{\tilde x}_s}} \|^2} ] +2 L_G^2E[ {\| {{{\tilde x}_s} - {x^*}} \|^2} ]} ,
		\end{align*}
		where $L_G$ is the parameter in (\ref{InequationAssumptionG2}).
	\end{lemma}
	\begin{proof}
		\begin{align*}
		& E[ {\| {\partial {{\hat G}_k} - \partial G( {{x^*}} )} \|^2} ]\\
		=& E[ {\| {\frac{1}{A}\sum\limits_{1 \le j \le A}^{} {( {\partial {G_{{{\cal A}_{k}}[j]}}( {{x_{k }}} ) - \partial {G_{{{\cal A}_{k }}[j]}}( {{{\tilde x}_s}} )} )}  + \partial G( {{{\tilde x}_s}} ) - \partial G( {{x^*}} )} \|^2} ]\\
		\le& {2E[ {\| {\frac{1}{A}\sum\limits_{1 \le j \le A}^{} {( {\partial {G_{{{\cal A}_{k}}[j]}}( {{x_{k}}} ) - \partial {G_{{{\cal A}_{k}}[j]}}( {{{\tilde x}_s}} )} )} } \|^2} ] + 2E[ {\| {\partial G( {{{\tilde x}_s}} ) - \partial G( {{x^*}} )} \|^2} ]} \\
		\le& 2\frac{1}{{{A^2}}}\sum\limits_{1 \le j \le A}^{} {E{{\| {{\partial G_{{A_k}[j]}}({x_k}) - {\partial G_{{A_k}[j]}}({{\tilde x}_s})} \|}^2}}+ 2E[ {\| {\partial G( {{{\tilde x}_s}} ) - \partial G( {{x^*}} )} \|^2} ] \\
		\le& {2L_G^2\frac{1}{A}E[ {\| {{x_{k }} - {{\tilde x}_s}} \|^2} ] +2 L_G^2E[ {\| {{{\tilde x}_s} - {x^*}} \|^2} ]} ,
		\end{align*}
		where the first and the second inequalities follow from Lemma \ref{RandomVariable2} and the Lipschitz continuous gradient of $G$ in (\ref{InequationAssumptionG2}).
	\end{proof}
	\subsection{Bounding the estimation of  function $F$}
	\begin{lemma}\label{LemmBoundSVRGEstimateFullGradientFConvex}
		Suppose Assumption \ref{Assumption2} and \ref{Assumption3} holds, in algorithm \ref{AlgorithmSDFCVRG1}, for the intermediated iteration at $x_k$ and $\tilde{x}_s$, and $ \hat{G}_k$ and  $\partial \hat{G}_k$ defined in (\ref{SCDF:SVRG:DefinitionSVRGEstimateG}) and (\ref{SCDF:SVRG:DefinitionSVRGFunctionG}),  we have
		\begin{align*}
		&E[ \| (\partial {{\hat G}_k})^\mathsf{T}\nabla {F_i}({{\hat G}_k}) - (\partial G({x_k}))^\mathsf{T}\nabla {F_i}(G(x_k)) \|_{}^2 ]\\\le& 2\left(B_F^2L_G^2+B_G^4L_F^2\right)	\frac{1}{A}E[ {\| {{x_{k }} - x^*} \|^2} ]+2\left(B_F^2L_G^2+B_G^4L_F^2\right)\frac{1}{A}E[ {\| {{{\tilde x}_s} - x^*} \|^2} ]
		\end{align*}
		where $L_F$ and $B_G$ are the parameters in (\ref{InequationAssumptionF2}) and (\ref{InequationAssumptionG1}).
	\end{lemma}
	\begin{proof} Through subtracting and adding $(\partial G({x_k}))^\mathsf{T}\nabla {F_i}({{\hat G}_k}) $
		\begin{align*}
		&E[ \| (\partial {{\hat G}_k})^\mathsf{T}\nabla {F_i}({{\hat G}_k}) - (\partial G({x_k}))^\mathsf{T}\nabla {F_i}(G(x_k)) \|_{}^2 ]\\
		\le& E[ \| (\partial {{\hat G}_k})^\mathsf{T}\nabla {F_i}({{\hat G}_k})-(\partial G({x_k}))^\mathsf{T}\nabla {F_i}({{\hat G}_k})+(\partial G({x_k}))^\mathsf{T}\nabla {F_i}({{\hat G}_k}) - (\partial G({x_k}))^\mathsf{T}\nabla {F_i}(G(x_k)) \|_{}^2 ]\\
		\le& 2B_F^2E[ \| \partial {{\hat G}_k}-\partial G({x_k})\|_{}^2 ]+2B_G^2E[ \| \nabla {F_i}({{\hat G}_k}) - \nabla {F_i}(G(x_k)) \|_{}^2 ]\\
		\le& 2B_F^2E[ \| \partial {{\hat G}_k}-\partial G({x_k})\|_{}^2 ]+2B_G^2L_F^2E[ \| {{\hat G}_k} - G(x_k) \|_{}^2 ]\\
		\le &  2B_F^2L_G^2\frac{1}{A}E[ {\| {{x_{k }} - {{\tilde x}_s}} \|^2} ]+2B_G^2L_F^2	B_G^2\frac{1}{A}E[ {\| {{x_{k }} - {{\tilde x}_s}} \|^2} ]\\
		\le& \left(2B_F^2L_G^2\frac{1}{A}+2B_G^2L_F^2	B_G^2\frac{1}{A}\right)	E[ {\| {{x_{k }} - x^*} \|^2} ]+\left(2B_F^2L_G^2\frac{1}{A}+2B_G^2L_F^2	B_G^2\frac{1}{A}\right)	E[ {\| {{{\tilde x}_s} - x^*} \|^2} ]
		\end{align*}
		where the first and   the second  inequalities is based on the bounded Jacobian of $G$ and Lipschitz continuous gradient of $F$. The last inequality follows from Lemma \ref{LemmaSVRGBoundVarianceG}.
	\end{proof}
	
	\begin{lemma}\label{LemmaBoundSVRGestimateFullGradient}
		Suppose Assumption \ref{Assumption2} holds, in algorithm \ref{AlgorithmSDFCVRG1}, for the intermediated iteration at $\beta_k$, $ \hat{G}_k$ and  $\partial \hat{G}_k$ defined in (\ref{SCDF:SVRG:DefinitionSVRGEstimateG}) and (\ref{SCDF:SVRG:DefinitionSVRGFunctionG}), we have,
		\begin{align*}
		E[ {\| {{( {\partial {{\hat G}_{k}}} )^\mathsf{T}}\nabla {F_i}( {{{\hat G}_{k }}} ) + \beta _i^*} \|^2} ] \le 4( {B_F^2L_G^2 + B_G^4L_F^2} )\frac{1}{A}E[\| {{x_k} -x^* \|^2}] +  4( {B_F^2L_G^2 + B_G^4L_F^2} )\left(1+\frac{1}{A}\right)E[\| {\tilde x}_s - x^* \|^2],
		\end{align*}
		where  $L_G$, $L_F$, $B_G$ and $B_F$ are the parameters in (\ref{InequationAssumptionF1}) - (\ref{InequationAssumptionG3}).
	\end{lemma}
	\begin{proof}
		Based on the relationship between $\beta _i^*$ and ${( {\partial G( {{x^*}} )} )^\mathsf{T}}\nabla {F_i}( {G( {{x^*}} )} )$, we have
		\begin{align*}
		E[ {\| {{(\partial {{\hat G}_k})^\mathsf{T}}\nabla {F_i}({{\hat G}_k}) + \beta _i^*} \|_{}^2} ]
		= E[ {\| {{(\partial {{\hat G}_k})^\mathsf{T}}\nabla {F_i}({{\hat G}_k}) - {( {\partial G( {{x^*}} )} )^\mathsf{T}}\nabla {F_i}( {G( {{x^*}} )} )} \|_{}^2} ].
		\end{align*}
		Through  subtracting and adding ${{( {\partial G( {{x_k}} )} )^\mathsf{T}}\nabla {F_i}( {{{ G}_{k}}} )}$, we obtain 
		\begin{align*}
		& E[ {\| {{(\partial {{\hat G}_k})^\mathsf{T}}\nabla {F_i}({{\hat G}_k}) - {( {\partial G( {{x^*}} )} )^\mathsf{T}}\nabla {F_i}( {G( {{x^*}} )} )} \|_{}^2} ]\\
		=& E[ {\| {{(\partial {{\hat G}_k})^\mathsf{T}}\nabla {F_i}({{\hat G}_k}) - {( {\partial G( {{x^*}} )} )^\mathsf{T}}\nabla {F_i}({{\hat G}_k}) + {{( {\partial G( {{x^*}} )} )}^\mathsf{T}}\nabla {F_i}({{\hat G}_k}) - {{( {\partial G( {{x^*}} )} )}^\mathsf{T}}\nabla {F_i}( {G( {{x^*}} )} )} \|_{}^2} ]\\
		\le& 2E[ {\| {{(\partial {{\hat G}_k})^\mathsf{T}}\nabla {F_i}({{\hat G}_k}) - {{( {\partial G( {{x^*}} )} )}^\mathsf{T}}\nabla {F_i}({{\hat G}_k})} \|_{}^2} ]\\
		&+ 2E[ {\| {{( {\partial G( {{x^*}} )} )^\mathsf{T}}\nabla {F_i}({{\hat G}_k}) - {{( {\partial G( {{x^*}} )} )}^\mathsf{T}}\nabla {F_i}( {G( {{x^*}} )} )} \|_{}^2} ]\\
		\le& 2B_F^2E[ {\| {\partial {{\hat G}_k} - \partial G( {{x^*}} )} \|_{}^2} ] + 2B_G^2E[ {\| {\nabla {F_i}({{\hat G}_k}) - \nabla {F_i}( {G( {{x^*}} )} )} \|_{}^2} ]\\
		\le& 2B_F^2E[ {\| {\partial {{\hat G}_k} - \partial G( {{x^*}} )} \|_{}^2} ] + 2B_G^2L_F^2E[ {\| {{{\hat G}_k} - G( {{x^*}} )} \|_{}^2} ]\\
		\le& 4B_F^2L_G^2\frac{1}{A}E[{\| {{x_k} - {{\tilde x}_s}} \|^2}] + 4B_F^2L_G^2E[{\| {{{\tilde x}_s} - {x^*}} \|^2}] + 4B_G^4L_F^2\frac{1}{A}E[{\| {{x_k} - {{\tilde x}_s}} \|^2}] + 4B_G^4L_F^2E[{\| {{{\tilde x}_s} - {x^*}} \|^2}]\\
		=& \left( {4B_F^2L_G^2\frac{1}{A} + 4B_G^4L_F^2\frac{1}{A}} \right)E[{\| {{x_k} - {{\tilde x}_s}} \|^2}] + ( {4B_F^2L_G^2 + 4B_G^4L_F^2} )E[{\| {{{\tilde x}_s} - {x^*}} \|^2}]\\
		\le &4( {B_F^2L_G^2 + B_G^4L_F^2} )\frac{1}{A}E[\| {{x_k} -x^* \|^2}] +  4( {B_F^2L_G^2 + B_G^4L_F^2} )\left(1+\frac{1}{A}\right)E[\| {\tilde x}_s - x^* \|^2],
		\end{align*}
		where the first and fourth inequality follows from Lemma \ref{RandomVariable2}, the second and third inequalities are based on the bounded gradient of $F$ (\ref{InequationAssumptionF1}), the bounded Jacobian of $G$ in (\ref{InequationAssumptionG1}), and  Lipschitz continuous gradient of $F$ in (\ref{InequationAssumptionF2}), the fourth inequality follows from Lemma \ref{LemmaSVRGBoundestimateGAndOptimal} and \ref{LemmaSVRGBoundgradientGAndOptimal}.
	\end{proof}
	\begin{lemma}\label{LemmaBoundSVRGestimateFullGradientConvex}
		Suppose Assumption \ref{Assumption2} and \ref{Assumption3} holds, in algorithm \ref{AlgorithmSDFCVRG1}, for the intermediated iteration at $\beta_k$, $ \hat{G}_k$ and  $\partial \hat{G}_k$ defined in (\ref{SCDF:SVRG:DefinitionSVRGEstimateG}) and (\ref{SCDF:SVRG:DefinitionSVRGFunctionG}), we have,
		\begin{align*}
		&E[ {\| {{( {\partial {{\hat G}_{k}}} )^\mathsf{T}}\nabla {F_i}( {{{\hat G}_{k }}} ) + \beta _i^*} \|^2} ]\\ \le& 4\left(B_F^2L_G^2+B_G^4L_F^2\right)	\frac{1}{A}E[ {\| {{x_{k }} - x^*} \|^2} ]+4\left(B_F^2L_G^2+B_G^4L_F^2\right)\frac{1}{A}E[ {\| {{{\tilde x}_s} - x^*} \|^2} ]\\
		&+ 4L_f( P(x_k) - P({x^*})) -2 L_F\lambda ({{\| {x_k - {x^*}} \|}^2} ),
		\end{align*}
		where  $L_G$, $L_F$, $B_G$ and $B_F$ are the parameters in (\ref{InequationAssumptionF1}) - (\ref{InequationAssumptionG3}).
	\end{lemma}
	\begin{proof}
		Based on the relationship between $\beta _i^*$ and ${{( {\partial G( {{x^*}} )} )}^\mathsf{T}}\nabla {F_i}( {G( {{x^*}} )} )$, we have
		\begin{align*}
		E[ {\| {{(\partial {{\hat G}_k})^\mathsf{T}}\nabla {F_i}({{\hat G}_k}) + \beta _i^*} \|_{}^2} ]
		= E[ {\| {{(\partial {{\hat G}_k})^\mathsf{T}}\nabla {F_i}({{\hat G}_k}) - {( {\partial G( {{x^*}} )} )^\mathsf{T}}\nabla {F_i}( {G( {{x^*}} )} )} \|_{}^2} ].
		\end{align*}
		Through  subtracting and adding ${{{(\partial G({x_k}))}^\mathsf{T}}\nabla {F_i}(G(x_k))}$, we obtain 
		\begin{align*}
		& E[ {\| {{(\partial {{\hat G}_k})^\mathsf{T}}\nabla {F_i}({{\hat G}_k}) - {( {\partial G( {{x^*}} )} )^\mathsf{T}}\nabla {F_i}( {G( {{x^*}} )} )} \|_{}^2} ]\\
		=& E[ \| {{(\partial {{\hat G}_k})^\mathsf{T}}\nabla {F_i}({{\hat G}_k}) - {{(\partial G({x_k}))^\mathsf{T}}\nabla {F_i}(G(x_k))} + {{(\partial G({x_k}))^\mathsf{T}}\nabla {F_i}(G(x_k))}- ( {\partial G( {{x^*}} )} )^\mathsf{T}\nabla {F_i}( G( {{x^*}} ) )} \|_{}^2 ]\\
		\le& 2E[ {\| {{(\partial {{\hat G}_k})^\mathsf{T}}\nabla {F_i}({{\hat G}_k}) - {{(\partial G({x_k}))^\mathsf{T}}\nabla {F_i}(G(x_k))}} \|_{}^2} ]
		+ 2E[ \| {{(\partial G({x_k}))^\mathsf{T}}\nabla {F_i}(G(x_k))}- ( \partial G( {{x^*}} ) )^\mathsf{T}\nabla {F_i}( G( {{x^*}} ) ) \|_{}^2 ]\\
		\le &4\left(B_F^2L_G^2+B_G^4L_F^2\right)	\frac{1}{A}E[ {\| {{x_{k }} - x^*} \|^2} ]+4\left(B_F^2L_G^2+B_G^4L_F^2\right)\frac{1}{A}E[ {\| {{{\tilde x}_s} - x^*} \|^2} ]\\
		&+ 4L_f\left( {P(x_k) - P({x^*}) - \frac{\lambda }{2}{{\| {x_k - {x^*}} \|}^2}} \right),
		\end{align*}
		where the first inequality follow from Lemma \ref{RandomVariable2}, and the second are based on Lemma  \ref{LemmBoundSVRGEstimateFullGradientFConvex} and Lemma \ref{LemmaAppendix}. 		
	\end{proof}
	\subsection{Bound the difference of variable and the optimal solution}
	\begin{lemma}\label{LemmaMainBoundC}
		Suppose Assumption \ref{Assumption1} and \ref{Assumption2} hold, and $P(x)$ is $\lambda$-strongly convex.
		In algorithm \ref{AlgorithmSDFCVRG1}, let let ${A_k} =\| {{x_k} - {x^*}} \|^2$, ${B_k} = \frac{1}{n}\sum\nolimits_{i = 1}^n {{{\| {\beta _i^k - \beta _i^*} \|}^2}} $ and ${C_k} = aE[ {{A_k}} ] + bE[ {{B_k}} ]$, $a, b\ge 0$. As long as $A \ge 2{R_x}B_G^4L_F^2$, the step 
		\begin{align}
		\eta  \le \frac{{1 - 2{R_x}B_G^4L_F^2\frac{1}{A}}}{{4\left( {B_F^2L_G^2 + B_G^4L_F^2} \right) + n{\lambda ^2}\left( {1 - 2{R_x}B_G^4L_F^2\frac{1}{A}} \right)}},
		\end{align}
		we can obtain	
		\begin{align*}
		{C_{k+1}} - {C_{k}} \le  - \eta \lambda {C_{k}} + {d_2}E[ {\| {{{\tilde x}_s} - {x^*}} \|^2} ],
		\end{align*} 
		where the parameters $a$, $b$ and $d_2$ satisfy
		\begin{align*}
		&{d_2} = 2a\eta \lambda {R_x}B_G^4L_F^2\frac{1}{A} + 4b\lambda \eta \left( {B_F^2L_G^2 + B_G^4L_F^2} \right)\left( {1 + \frac{1}{A}} \right)\\
		&\frac{{4\left( {B_F^2L_G^2 + B_G^4L_F^2} \right)}}{{1 - 2{R_x}B_G^4L_F^2\frac{1}{A}}}\le\frac{a}{b} \le \frac{{(1 - n\lambda \eta )\lambda }}{\eta }.
		\end{align*}
	\end{lemma}
	\begin{proof} By adding bound results of Lemma \ref{BoundLemmaA} and Lemma \ref{BoundLemmaB}, we have,
		\begin{align*}
		{C_{k + 1}} =& aE[{A_{k + 1}}] + bE\left[ {{B_{k + 1}}} \right]\\
		\le& a\left( {1 - \eta \lambda } \right)E[{A_k}] + b\left( {1 - \lambda \eta } \right)E[{B_k}]\\
		&+ \underbrace {\left( {2a\eta \lambda {R_x}B_G^4L_F^2\frac{1}{A} - a\eta \lambda  + 4b\lambda \eta \left( {B_F^2L_G^2 + B_G^4L_F^2} \right)\frac{1}{A}} \right)}_{{d_1}}E[{A_k}]\\
		&+ \underbrace {\left( {2a\eta \lambda {R_x}B_G^4L_F^2\frac{1}{A} + 4b\lambda \eta \left( {B_F^2L_G^2 + B_G^4L_F^2} \right)\left( {1 + \frac{1}{A}} \right)} \right)}_{{d_2}}E[{\| {{{\tilde x}_s} - {x^*}} \|^2}]\\
		&+ \underbrace {( {a{\eta ^2} - b(1 - n\lambda \eta )\lambda \eta } )}_{{d_3}}E[{\| {{(\partial {{\hat G}_k})^\mathsf{T}}\nabla {F_i}({{\hat G}_k}) + \beta _i^k} \|^2}].
		\end{align*}
		In order to obtain ${C_{k+1}} - {C_{k}} \le  - \eta \lambda {C_{k}} + {d_2}E[ {\left\| {{{\tilde x}_s} - {x^*}} \right\|^2} ]$, we can choose the step $\eta$ such that $d_1$ and $d_3$ are both negative, that is 
		\begin{align}\label{StepSVRGMainTheorem1}
		\frac{a}{b} \ge& \frac{{4\left( {B_F^2L_G^2 + B_G^4L_F^2} \right)}}{{1 - 2{R_x}B_G^4L_F^2\frac{1}{A}}}\\
		\label{StepSVRGMainTheorem2}
		\frac{a}{b} \le& \frac{{(1 - n\lambda \eta )\lambda }}{\eta },
		\end{align}
		In order to keep $1  - 2 {R_x}B_G^4L_F^2\frac{1}{A} \ge 0$ positive, the sample times $A$ should satisfy $A \ge 2{R_x}B_G^4L_F^2$. Based on conditions (\ref{StepSVRGMainTheorem1}) and (\ref{StepSVRGMainTheorem2}), the step $\eta$ can be bounded as 
		\begin{align*}
		\eta  \le \frac{{1 - 2{R_x}B_G^4L_F^2\frac{1}{A}}}{{4\left( {B_F^2L_G^2 + B_G^4L_F^2} \right) + n{\lambda ^2}\left( {1 - 2{R_x}B_G^4L_F^2\frac{1}{A}} \right)}}.
		\end{align*}
	\end{proof}
	\begin{lemma}\label{LemmaSVRGLemmaMainBoundCConvex}
		Suppose Assumption \ref{Assumption1}, \ref{Assumption2} and \ref{Assumption3} hold, and $P(x)$ is $\lambda$-strongly convex.
		In algorithm \ref{AlgorithmSDFCVRG1}, let let ${A_k} =\| {{x_k} - {x^*}} \|^2$, ${B_k} = \frac{1}{n}\sum\nolimits_{i = 1}^n {{{\| {\beta _i^k - \beta _i^*} \|}^2}} $ and ${C_k} = aE[ {{A_k}} ] + bE[ {{B_k}} ]$, $a, b\ge 0$. As long as the sample times $A$ and  the step satisfy 
		\begin{align*}
		A \ge \frac{{2{R_x}B_G^4L_F^2}}{d},\eta  \le \frac{{1 - d}}{{2{L_f} + \lambda n\left( {1 - d} \right)}},
		\end{align*}
		we can obtain	
		\begin{align*}
		{C_{k+1}} - {C_{k}} \le  - \eta \lambda {C_{k}} + {e_2}E[ {\| {{{\tilde x}_s} - {x^*}} \|^2} ],
		\end{align*} 
		where the parameters $a$, $b$ and $e_2$ satisfy
		\begin{align*}
		&e_2={2a\eta \lambda {R_x}B_G^4L_F^2\frac{1}{A} + 4b\lambda \eta \left( {B_F^2L_G^2 + B_G^4L_F^2} \right)\frac{1}{A}}\\
		&\frac{{2\left( {2B_F^2L_G^2 + B_G^4L_F^2} \right)\frac{1}{A} - {L_f}\lambda }}{{d - 2{R_x}B_G^4L_F^2\frac{1}{A}}} \le \frac{a}{b} \le \frac{{(1 - n\lambda \eta )\lambda }}{\eta }\\
		&d \le \frac{{\left( {2B_F^2L_G^2 + B_G^4L_F^2} \right)\frac{1}{A} + \lambda {L_f}{R_x}B_G^4L_F^2\frac{1}{A}}}{{\left( {2B_F^2L_G^2 + B_G^4L_F^2} \right)\frac{1}{A} + \lambda {L_f}}}.
		\end{align*}
	\end{lemma}
	\begin{proof} By adding bound results of Lemma \ref{LemmaSVRGBoundLemmaAConvex} and Lemma \ref{LemmaSAGABoundLemmaBConvex}, we have,
		\begin{align*}
		{C_{k + 1}} =& aE[{A_{k + 1}}] + bE\left[ {{B_{k + 1}}} \right]\\
		\le& a\left( {1 - \eta \lambda } \right)E[{A_k}] + b\left( {1 - \lambda \eta } \right)E[{B_k}]\\
		&+ \underbrace {\left( {2a\eta \lambda {R_x}B_G^4L_F^2\frac{1}{A} + 4b\lambda \eta \left( {2B_F^2L_G^2 + B_G^4L_F^2} \right)\frac{1}{A} - 2b\eta \lambda^2 {L_f}  - a\eta \lambda d} \right)}_{{e_1}}E[{A_k}]\\
		&+ \underbrace {\left( {2a\eta \lambda {R_x}B_G^4L_F^2\frac{1}{A} + 4b\lambda \eta \left( {B_F^2L_G^2 + B_G^4L_F^2} \right)\frac{1}{A}} \right)}_{{e_2}}E[{\left\| {{{\tilde x}_s} - {x^*}} \right\|^2}]\\
		&\underbrace {\left( { - 2a\eta (1 - d) + 4b\eta \lambda {L_f}} \right)}_{{e_3}}\left( {P(x) - P({x^*})} \right)\\
		&+ \underbrace {\left( {a{\eta ^2} - b(1 - n\lambda \eta )\lambda \eta } \right)}_{{e_4}}E[{\left\| {{{(\partial {{\hat G}_k})}^T}\nabla {F_i}({{\hat G}_k}) + \beta _i^k} \right\|^2}],
		\end{align*}
		In order to obtain ${C_{k+1}} - {C_{k}} \le  - \eta \lambda {C_{k}} + {d_2}E[ {\left\| {{{\tilde x}_s} - {x^*}} \right\|^2} ]$, we can choose the step $\eta$ such that $e_1$, $e_2$, $e_3$ and $e_4$ are all negative, that is 
		\begin{align}\label{StepSVRGMainTheorem1Convex}
		\frac{{\lambda {L_f}}}{{(1 - d)}} \le \frac{{2\left( {2B_F^2L_G^2 + B_G^4L_F^2} \right)\frac{1}{A} - {L_f}\lambda }}{{d - 2{R_x}B_G^4L_F^2\frac{1}{A}}} \le \frac{a}{b} \le \frac{{(1 - n\lambda \eta )\lambda }}{\eta }
		\end{align}
		In order to keep ${d - 2{R_x}B_G^4L_F^2\frac{1}{A}}$ positive, the sample times $A$ should satisfy $A \ge {{2{R_x}B_G^4L_F^2} \mathord{\left/
				{\vphantom {{2{R_x}B_G^4L_F^2} d}} \right.
				\kern-\nulldelimiterspace} d}$. Based on conditions (\ref{StepSVRGMainTheorem1}), the step $\eta$ can be bounded as 
		\begin{align*}
		\eta  \le \frac{{1 - d}}{{2{L_f} + \lambda n\left( {1 - d} \right)}}.
		\end{align*}
	\end{proof}
	\begin{lemma}\label{BoundLemmaA}
		Suppose Assumption  \ref{Assumption1} and \ref{Assumption2} hold, in algorithm \ref{AlgorithmSDFCVRG1}, for the intermediated iteration at $x_k$, let ${A_k} =\| {{x_k} - {x^*}} \|^2$, define $\lambda {R_x} = \max _x \{ \| {{x^*} - x} \|^2:F(G(x)) \le F(G({x_0})) \}$, the bound of ${A_k} $ satisfies,
		\begin{align*}
		E[ {{A_{k+1}}} ] \le&  E[ {{A_{k }}} ] +2\eta \lambda R_xB_G^4L_F^2\frac{1}{A} E[ {{A_{k }}} ]+2\eta  \lambda R_xB_G^4L_F^2\frac{1}{A}E[ {\| { {{\tilde x}_s}}-x^* \|^2} ]\\
		&-2\eta  \lambda E[ {{A_{k }}} ]  + {\eta ^2}E[ {\| {{( {\partial {{\hat G}_{k}}} )^\mathsf{T}}\nabla {F_i}( {{{\hat G}_{k }}} ) + \beta _i^{k}} \|^2} ],
		\end{align*}
		where $x^*$ is the optimal solution.
	\end{lemma}
	\begin{proof}
		Based on  the update of $x_k$,we have 
		\begin{align*}
		{A_{k+1}}=& \| {{x_{k}} - \eta ( {{( {\partial {{\hat G}_{k}}} )^\mathsf{T}}\nabla {F_i}( {{{\hat G}_{k}}} ) + \beta _i^{k }} ) - {x^*}} \|^2\\
		=& \| {{x_{k}} - {x^*}} \|^2 - 2\eta \langle {{( {\partial {{\hat G}_{k }}} )^\mathsf{T}}\nabla {F_i}( {{{\hat G}_{k}}} ) + \beta _i^{k },{x_{k}} - {x^*}} \rangle  + \| {\eta ( {{( {\partial {{\hat G}_{k }}} )^\mathsf{T}}\nabla {F_i}( {{{\hat G}_{k}}} ) + \beta _i^{k }} )} \|^2.
		\end{align*}
		Taking expectation with respect to $i,j$, we get,
		\begin{align*}
		E[ {{A_{k+1}}} ] =& E[ {{A_{k}}} ]\underbrace { - 2\eta E[ {\langle {{( {\partial {{\hat G}_{k }}} )^\mathsf{T}}\nabla {F_i}( {{{\hat G}_{k }}} ) + \beta _i^{k},{x_{k}} - {x^*}} \rangle } ]}_{A1} + {\eta ^2}E[ {\| {{( {\partial {{\hat G}_{k }}} )^\mathsf{T}}\nabla {F_i}( {{{\hat G}_{k}}} ) + \beta _i^{k }} \|^2} ]\\
		\le& E[ {{A_{k }}} ] +2\eta \lambda R_xB_G^4L_F^2\frac{1}{A}E[ {\| {{x_{k}} - {x^*}} \|^2} ]+2\eta  \lambda R_xB_G^4L_F^2\frac{1}{A}E[ {\| { {{\tilde x}_s}}-x^* \|^2} ]\\
		&-2\eta  \lambda\| {{x_{k}} - {x^*}} \|^2 + + {\eta ^2}E[ {\| {{( {\partial {{\hat G}_{k}}} )^\mathsf{T}}\nabla {F_i}( {{{\hat G}_{k }}} ) + \beta _i^{k}} \|^2} ],
		\end{align*}
		where $A1$ follows from Lemma \ref{LemmaBoundA1}.
	\end{proof}
	Based on above Lemma, we can also get another form upper bound.
	\begin{lemma}\label{LemmaSVRGBoundLemmaAConvex}
		Suppose Assumption  \ref{Assumption1} and \ref{Assumption2} hold, in algorithm \ref{AlgorithmSDFCVRG1}, for the intermediated iteration at $x_k$, let ${A_k} =\| {{x_k} - {x^*}} \|^2$, define $\lambda {R_x} = \max _x \{ \| {{x^*} - x} \|^2:F(G(x)) \le F(G({x_0})) \}$, the bound of ${A_k} $ satisfies,
		\begin{align*}
		E[ {{A_{k+1}}} ] \le&  E[ {{A_{k }}} ] +2\eta \lambda R_xB_G^4L_F^2\frac{1}{A} E[ {{A_{k }}} ]+2\eta  \lambda R_xB_G^4L_F^2\frac{1}{A}E[ {\| { {{\tilde x}_s}}-x^* \|^2} ]\\
		&- 2\eta (1 - d)(P({x_k}) - P({x^*})) - \eta \lambda (1 + d){\left\| {{x_k} - {x^*}} \right\|^2} + {\eta ^2}E[ {\| {{( {\partial {{\hat G}_{k}}} )^\mathsf{T}}\nabla {F_i}( {{{\hat G}_{k }}} ) + \beta _i^{k}} \|^2} ],
		\end{align*}
		where $x^*$ is the optimal solution, and $1>d\ge0$.
	\end{lemma}
	\begin{lemma}\label{BoundLemmaB}
		Suppose Assumption \ref{Assumption2} holds, in algorithm \ref{AlgorithmSDFCVRG1}, for the intermediated iteration at $\beta^k$, let ${B_k} = \frac{1}{n}\sum\nolimits_{i = 1}^n {{{\| {\beta _i^k - \beta _i^*} \|}^2}} $, the bound of $B_k$ satisfy,
		\begin{align*}
		E[ {{B_{k+1}}} ] \le&  E[ {{B_{k}}} ]  - \lambda \eta E[ {{B_{k}}} ] + 4\lambda \eta ( {B_F^2L_G^2 + B_G^4L_F^2} )\frac{1}{A}E[\| {{x_k} -x^* \|^2}] +  4\lambda \eta( {B_F^2L_G^2 + B_G^4L_F^2} )\left(1+\frac{1}{A}\right)E[\| {\tilde x}_s - x^* \|^2]\\
		&- ( {1 - n\lambda {\eta }} )\lambda {\eta}E[ {\| {{( {\partial {{\hat G}_{k }}} )^\mathsf{T}}\nabla {F_i}( {{{\hat G}_{k}}} ) + \beta _i^{k}} \|^2} ],
		\end{align*}
		where $B_F$, $L_F$, $B_G$ and $L_G$ are the parameters in (\ref{InequationAssumptionF1}) - (\ref{InequationAssumptionG3}).
	\end{lemma}
	\begin{proof} Based on the definition of $B_k$, and the update of $\beta$, we have
		\begin{align*}
		&{B_{k+1}} - {B_{k}}\\
		=& \frac{1}{n}\sum\limits_{i = 1}^n {\| {( {\beta _i^{k} - n\lambda \eta ( {{( {\partial {{\hat G}_{k }}} )^\mathsf{T}}\nabla {F_i}( {{{\hat G}_{k }}} ) + \beta _i^{k}} )} ) - \beta _i^*} \|^2}  - \frac{1}{n}\sum\limits_{i = 1}^n {\| {\beta _i^{k} - \beta _i^*} \|^2} \\
		=& \frac{1}{n}\| {( {\beta _i^{k} - n\lambda \eta ( {{( {\partial {{\hat G}_{k}}} )^\mathsf{T}}\nabla {F_i}( {{{\hat G}_{k}}} ) + \beta _i^{k}} )} ) - \beta _i^*} \|^2 - \frac{1}{n}\| {\beta _i^{k} - \beta _i^*} \|^2\\
		=& \frac{1}{n}\| {( {\beta _i^{k} - n\lambda \eta ( {{( {\partial {{\hat G}_{k}}} )^\mathsf{T}}\nabla {F_i}( {{{\hat G}_{k}}} ) + \beta _i^{k} - \beta _i^* + \beta _i^*} )} ) - \beta _i^*} \|^2 - \frac{1}{n}\| {\beta _i^{k} - \beta _i^*} \|^2\\
		=& \frac{1}{n}\| {( {1 - n\lambda \eta } )( {\beta _i^{k} - \beta _i^*} ) + n\lambda \eta ( { - {( {\partial {{\hat G}_{k}}} )^\mathsf{T}}\nabla {F_i}( {{{\hat G}_{k}}} ) - \beta _i^*} )} \|^2 - \frac{1}{n}\| {\beta _i^{k} - \beta _i^*} \|^2\\
		=& \frac{1}{n}( {1 - n\lambda \eta } )\| {\beta _i^{k} - {\beta ^*}} \|^2 + \lambda \eta \| { - {( {\partial {{\hat G}_{k}}} )^\mathsf{T}}\nabla {F_i}( {{{\hat G}_{k }}} ) - \beta _i^*} \|^2 - ( {1 - n\lambda \eta } )\lambda \eta \| {{( {\partial {{\hat G}_{k}}} )^\mathsf{T}}\nabla {F_i}( {{{\hat G}_{k}}} ) + \beta _i^{k}} \|^2\\& - \frac{1}{n}\| {\beta _i^{k} - \beta _i^*} \|^2\\
		=&  - \lambda \eta \| {\beta _i^{k} - {\beta ^*}} \|^2 + \lambda \eta \| { - {( {\partial {{\hat G}_{k }}} )^\mathsf{T}}\nabla {F_i}( {{{\hat G}_{k - 1}}} ) - \beta _i^*} \|^2 - ( {1 - n\lambda \eta } )\lambda \eta \| {{( {\partial {{\hat G}_{k}}} )^\mathsf{T}}\nabla {F_i}( {{{\hat G}_{k}}} ) + \beta _i^{k}} \|^2.
		\end{align*}
		Taking expectation with respect to $i$ on both sides, we have
		\begin{align*}
		&E[ {{B_{k+1}}} ] - E[ {{B_{k}}} ]\\
		=&  - \lambda \eta E[ {\| {\beta _i^{k} - {\beta ^*}} \|^2} ] + \lambda \eta \underbrace {E[ {\| {{( {\partial {{\hat G}_{k}}} )^\mathsf{T}}\nabla {F_i}( {{{\hat G}_{k}}} ) + \beta _i^*} \|^2} ]}_{B1} - ( {1 - n\lambda {\eta}} )\lambda {\eta}E[ {\| {{( {\partial {{\hat G}_{k}}} )^\mathsf{T}}\nabla {F_i}( {{{\hat G}_{k}}} ) + \beta _i^{k}} \|^2} ]\\
		\le&  - \lambda \eta E[ {{B_{k}}} ] + 4\lambda \eta \left( {B_F^2L_G^2 + B_G^4L_F^2} \right)\frac{1}{A}E[\| {{x_k} -x^* \|^2}] +  4\lambda \eta\left( {B_F^2L_G^2 + B_G^4L_F^2} \right)\left(1+\frac{1}{A}\right)E[\| {\tilde x}_s - x^* \|^2]\\
		&- ( {1 - n\lambda {\eta }} )\lambda {\eta}E[ {\| {{{( {\partial {{\hat G}_{k }}} )}^\mathsf{T}}\nabla {F_i}( {{{\hat G}_{k}}} ) + \beta _i^{k}} \|^2} ],
		\end{align*}
		where (B1) follows from Lemma \ref{LemmaBoundSVRGestimateFullGradient}.
	\end{proof}
	\begin{lemma}\label{LemmaSAGABoundLemmaBConvex}
		Suppose Assumption \ref{Assumption2} and \ref{Assumption3} hold, in algorithm \ref{AlgorithmSDFCVRG1}, for the intermediated iteration at $\beta^k$, let ${B_k} = \frac{1}{n}\sum\nolimits_{i = 1}^n \| \beta _i^k - \beta _i^* \|^2 $, the bound of $B_k$ satisfy,
		\begin{align*}
		E[ {{B_{k+1}}} ] \le&  - \lambda \eta E[ {{B_{k}}} ] + 4\lambda \eta\left(B_F^2L_G^2+B_G^4L_F^2\right)	\frac{1}{A}E[ {\| {{x_{k }} - x^*} \|^2} ]+4\lambda \eta\left(B_F^2L_G^2+B_G^4L_F^2\right)\frac{1}{A}E[ {\| {{{\tilde x}_s} - x^*} \|^2} ]\\
		&+ 4 \lambda \eta L_f( P(x_k) - P({x^*})) - 2 \lambda \eta L_F\|x_k - {x^*} \|^2 )- ( {1 - n\lambda {\eta }} )\lambda {\eta}E[ {\| {{( {\partial {{\hat G}_{k }}} )^\mathsf{T}}\nabla {F_i}( {{{\hat G}_{k}}} ) + \beta _i^{k}} \|^2} ],
		\end{align*}
		where $B_F$, $L_F$, $B_G$ and $L_G$ are the parameters in (\ref{InequationAssumptionF1}) - (\ref{InequationAssumptionG3}).
	\end{lemma}
	\begin{proof} Based on the definition of $B_k$, and the update of $\beta$, we have
		\begin{align*}
		&E[ {{B_{k+1}}} ] - E[ {{B_{k}}} ]\\
		=&  - \lambda \eta E[ {\| {\beta _i^{k} - {\beta ^*}} \|^2} ] + \lambda \eta \underbrace {E[ {\| {{( {\partial {{\hat G}_{k}}} )^\mathsf{T}}\nabla {F_i}( {{{\hat G}_{k}}} ) + \beta _i^*} \|^2} ]}_{B1} - ( {1 - n\lambda {\eta}} )\lambda {\eta}E[ {\| {{( {\partial {{\hat G}_{k}}} )^\mathsf{T}}\nabla {F_i}( {{{\hat G}_{k}}} ) + \beta _i^{k}} \|^2} ]\\		
		\le&  - \lambda \eta E[ {{B_{k}}} ] + 4\lambda \eta\left(B_F^2L_G^2+B_G^4L_F^2\right)	\frac{1}{A}E[ {\| {{x_{k }} - x^*} \|^2} ]+4\lambda \eta\left(B_F^2L_G^2+B_G^4L_F^2\right)\frac{1}{A}E[ {\| {{{\tilde x}_s} - x^*} \|^2} ]\\
		&+ 4L_f\lambda\eta\left( {P(x_k) - P({x^*}) - \frac{\lambda }{2}{{\| {x_k - {x^*}} \|}^2}} \right)- ( {1 - n\lambda {\eta }} )\lambda {\eta}E[ {\| {{( {\partial {{\hat G}_{k }}} )^\mathsf{T}}\nabla {F_i}( {{{\hat G}_{k}}} ) + \beta _i^{k}} \|^2} ],
		\end{align*}
		where (B1) follows from Lemma \ref{LemmaBoundSVRGestimateFullGradientConvex}.
	\end{proof}

	\begin{lemma}\label{LemmaBoundA1}	
		Suppose Assumption  \ref{Assumption1} and \ref{Assumption2} hold, in algorithm \ref{AlgorithmSDFCVRG1}, for the intermediated iteration at $x_k$ and $\beta^k$, $ \hat{G}_k$ and  $\partial \hat{G}_k$ defined in (\ref{SCDF:SVRG:DefinitionSVRGEstimateG}) and (\ref{SCDF:SVRG:DefinitionSVRGFunctionG}), let ${A_k} =\left\| {{x_k} - {x^*}} \right\|^2$,  define $\lambda {R_x} = \max _x \{ \| {{x^*} - x} \|^2:F(G(x)) \le F(G({x_0})) \}$,  we have
		\begin{align*}
		E[ {\langle {{( {\partial {{\hat G}_{k}}} )^\mathsf{T}}\nabla {F_i}( {{{\hat G}_{k }}} ) + \beta _i^{k },{x_{k}} - {x^*}} \rangle } ] \ge  - \lambda R_xB_G^4L_F^2\frac{1}{A}E[ {\| {{x_{k}} - {x^*}} \|^2} ]- \lambda R_xB_G^4L_F^2\frac{1}{A}E[ {\| { {{\tilde x}_s}}-x^* \|^2} ]+ \lambda\| {{x_{k}} - {x^*}} \|^2,
		\end{align*}
		where $L_F$ and $B_G$ are the parameters in (\ref{InequationAssumptionF2}) and (\ref{InequationAssumptionG1}).
	\end{lemma}
	\begin{proof}
		Through subtracting and adding ${( {\partial {{\hat G}_{k}}} )^\mathsf{T}}\nabla {F_i}( {{G_{k }}} ) $, we have
		\begin{align*}
		&E[ {\langle {{( {\partial {{\hat G}_{k}}} )^\mathsf{T}}\nabla {F_i}( {{{\hat G}_{k}}} ) + \beta _i^{k},{x_{k}} - {x^*}} \rangle } ]\\
		=& E[ {\langle {{( {\partial {{\hat G}_{k}}} )^\mathsf{T}}\nabla {F_i}( {{{\hat G}_{k}}} ) - {( {\partial {{\hat G}_{k}}} )^\mathsf{T}}\nabla {F_i}( {{G_{k }}} ) + {( {\partial {{\hat G}_{k }}} )^\mathsf{T}}\nabla {F_i}( {{G_{k }}} ) + \beta _i^{k},{x_{k}} - {x^*}} \rangle } ]\\
		=& \underbrace {E[\langle {(\partial {{\hat G}_k})^\mathsf{T}}\nabla {F_i}({{\hat G}_k}) -{( {\partial {{\hat G}_{k}}} )^\mathsf{T}}\nabla {F_i}( {{G_{k }}} ) ,{x_k} - {x^*}\rangle ]}_{(A11)}+\underbrace {E[\langle {(\partial {{\hat G}_k})^\mathsf{T}}\nabla {F_i}({G_k}) + \beta _i^k,{x_k} - {x^*}\rangle ]}_{(A12)} .
		\end{align*}		
		For the bound of (A11),  we have,
		\begin{align*}
		& E[ {\langle {{( {\partial {{\hat G}_{k }}} )^\mathsf{T}}\nabla {F_i}( {{{\hat G}_{k}}} ) - {( {\partial {{\hat G}_{k}}} )^\mathsf{T}}\nabla {F_i}( {{G_{k}}} ),{x_{k}} - {x^*}} \rangle } ]\\
		\ge& -\underbrace {E[ {\| {{( {\partial {{\hat G}_{k}}} )^\mathsf{T}}\nabla {F_i}( {{{\hat G}_{k }}} ) - {( {\partial {{\hat G}_{k}}} )^\mathsf{T}}\nabla {F_i}\left( {{G_{k}}} \right)} \|^2} ]}_{A2}E[\|x_k-x^*\|^2]\\
		\ge&    - \lambda R_xB_G^4L_F^2\frac{1}{A}E[ {\| {{x_{k}} - {x^*}} \|^2} ]- \lambda R_xB_G^4L_F^2\frac{1}{A}E[ {\| { {{\tilde x}_s}}-x^* \|^2} ],
		\end{align*}
		where the first inequation is based on Lemma \ref{LemmaInEquation}, (A2) follows from Lemma \ref{SCDF:LemmBoundSVRGEstimateFullGradientF}.

		For the bound of (A12), based on the relationship between $\beta$ and $x$, that is $\frac{1}{n}\sum\nolimits_{i = 1}^n {\beta _i^{}}  = \lambda x$, we have
		\begin{align*}
		E[ {\langle {{( {\partial {{\hat G}_{k}}} )^\mathsf{T}}\nabla {F_i}( {{G_{k}}} ) + \beta _i^{k },{x_{k}} - {x^*}} \rangle } ] 
		=& \langle {{( {\partial G( {{x_{k}}} )} )^\mathsf{T}}\nabla F( {{G_{k}}} ) + \lambda {x_{k}},{x_{k}} - {x^*}} \rangle  \\
		=& \langle {\nabla P( {{x_{k}}} ),{x_{k}} - {x^*}} \rangle \\
		\ge&P( {{x_{k}}} ) - P( {{x^*}} ) + \frac{\lambda }{2}\| {{x_{k}} - {x^*}} \|^2 \\
		\ge &  \lambda\| {{x_{k}} - {x^*}} \|^2 ,
		\end{align*}
		where the first and the second inequalities are based on the $\lambda$-strongly convexity of $P(x)$. Thus, combine the bound of (A11) and (A12), we can get the result.
	\end{proof}
	\section{Convergence Bound Analyses for SDFC-SAGA}

	\subsection{Bounding the estimation of inner function $G$ and partial gradient of $G$}
	The bound on the variance of inner function $G$ and its partial gradient  $\partial{\hat G}$ is in the following two lemmas,
	\begin{lemma}\label{LemmaSAGABoundFunctionG} 
		Suppose Assumption \ref{Assumption2} holds, in algorithm \ref{AlgorithmSCDFSAGA}, for the intermediated iteration at $x_k$, and $ \hat G$ defined in (\ref{DefinitionSAGAEstimateG}), we have
		\begin{align*}
		E[ {\| {{{\hat G}_k} - G\left( {{x_k}} \right)} \|_{}^2} ] \le B_G^2\frac{1}{{{A^2}}}\sum\limits_{1 \le j \le A}^{} {E[ {\| {{x_k} - \phi _{{{\cal A}_k}[j]}^k} \|_{}^2} ]},  
		\end{align*}
		where $B_G$ is parameter of the bounded Jacobian of $G$.
	\end{lemma}
	\begin{proof}
		From the definition of ${{{\hat G}_k}}$ in (\ref{DefinitionSAGAEstimateG}), we have
		\begin{align*}
		&E[ {\| {{{\hat G}_k} - G\left( {{x_k}} \right)} \|_{}^2} ]\\
		=& E[ {\| {\frac{1}{A}\sum\limits_{1 \le j \le A}^{} {\left( {{G_{{{\cal A}_k}[j]}}({x_k}) - {G_{{{\cal A}_k}[j]}}(\phi _{{{\cal A}_k}[j]}^k)} \right)}  + \frac{1}{m}\sum\limits_{j = 1}^m {G_j(\phi _j^k)}  -  {G}({x_k})} \|_{}^2} ]\\
		\le& \frac{1}{{{A^2}}}\sum\limits_{1 \le j \le A}^{} {E[ {\| {{G_{{{\cal A}_k}[j]}}({x_k}) - {G_{{\mathcal{A}_k}[j]}}(\phi _{{{\cal A}_k}[j]}^k) + \frac{1}{m}\sum\limits_{j = 1}^m {G_j(\phi _j^k)}  - {G}({x_k})} \|_{}^2} ]} \\
		\le& \frac{1}{{{A^2}}}\sum\limits_{1 \le j \le A}^{} {E[ {\| {{G_{{{\cal A}_k}[j]}}({x_k}) - {G_{{\mathcal{A}_k}[j]}}(\phi _{{{\cal A}_k}[j]}^k)} \|_{}^2} ]} \\
		\le& B_G^2\frac{1}{{{A^2}}}\sum\limits_{1 \le j \le A}^{} {E[ {\| {{x_k} - \phi _{{{\cal A}_k}[j]}^k} \|_{}^2} ]},
		\end{align*}
		where the first and the second inequality follow from Lemma \ref{RandomVariable2} and Lemma \ref{RandomVariable1}, and the third inequality is based on the bounded Jacobian of $G$ in (\ref{InequationAssumptionG2}).
	\end{proof}
	
	\begin{lemma}\label{LemmaSAGABoundGradientG} 
		Suppose Assumption \ref{Assumption2} holds, in algorithm \ref{AlgorithmSCDFSAGA}, for the intermediated iteration at $x_k$, and $\partial \hat G_k$ defined in (\ref{DefinitionSAGAEstimateGradientG}), we have
		\begin{align*}
		E[ {\| {\partial {{\hat G}_k} - \partial G({x_k})} \|_{}^2} ] \le L_G^2\frac{1}{{{A^2}}}\sum\limits_{1 \le j \le A}^{} {E[ {\| {{x_k} - \phi _{{A_k}[j]}^k} \|^2} ]},
		\end{align*}
		where $B_G$ is parameter of bounded Jacobian of $G$.
	\end{lemma}
	\begin{proof}
		From the definition of $\partial{{{\hat G}_k}}$ in (\ref{DefinitionSAGAEstimateGradientG}), we have
		\begin{align*}
		&E[ {\| {\partial {{\hat G}_k} - \partial G({x_k})} \|_{}^2} ]\\
		=& E[ {\| {\frac{1}{A}\sum\limits_{1 \le j \le A}^{} {( {\partial {G_{{\mathcal{A}_k}[j]}}({x_k}) - \partial {G_{{A_k}[j]}}(\phi _{{\mathcal{A}_k}[j]}^k)} )}  + \frac{1}{m}\sum\limits_{j = 1}^m {\partial G_j(\phi _j^k)}  - \partial {G}({x_k})} \|^2} ]\\
		\le& \frac{1}{{{A^2}}}\sum\limits_{1 \le j \le A}^{} {E[ {\| {\partial {G_{{\mathcal{A}_k}[j]}}({x_k}) - \partial {G_{{\mathcal{A}_k}[j]}}(\phi _{{\mathcal{A}_k}[j]}^k) + \frac{1}{m}\sum\limits_{j = 1}^m {\partial G_j(\phi _j^k)}  - \partial {G}({x_k})} \|^2} ]} \\
		\le& \frac{1}{{{A^2}}}\sum\limits_{1 \le j \le A}^{} {E[ {\| {\partial {G_{{\mathcal{A}_k}[j]}}({x_k}) - \partial {G_{{\mathcal{A}_k}[j]}}(\phi _{{\mathcal{A}_k}[j]}^k)} \|^2} ]} \\
		\le& L_G^2\frac{1}{{{A^2}}}\sum\limits_{1 \le j \le A}^{} {E[ {\| {{x_k} - \phi _{{\mathcal{A}_k}[j]}^k} \|^2} ]} ,
		\end{align*}
		where the first and the second inequality follow from Lemma \ref{RandomVariable2} and Lemma \ref{RandomVariable1}, the third inequality is based on the Lipschitz continuous gradient of $G$ in (\ref{InequationAssumptionG3}).
	\end{proof}
	
	\subsection{Bounding the estimation of function $F$}
	The following two lemmas shows the upper bound between the estimated gradient of $F(G(x))$ and unbiased estimate gradient of $F(G(x))$, and between estimated gradient of $F(G(x))$ and optimal solution.
	\begin{lemma}\label{LammaBoundSAGANormEsimateGradient}
		Assume Assumption \ref{Assumption2} holds, in algorithm \ref{AlgorithmSCDFSAGA}, for the intermediated iteration at $x_k$, $ \hat G$ defined in (\ref{DefinitionSAGAEstimateG}) and  $\partial \hat G_k$ defined in (\ref{DefinitionSAGAEstimateGradientG}), the following bound satisfies,
		\begin{align*}
		&E[ {\| {{(\partial {{\hat G}_k})^\mathsf{T}}\nabla {F_i}({{\hat G}_k}) - {(\partial G({x_k}))^\mathsf{T}}\nabla {F_i}({G_k}({x_k}))} \|^2} ]\\
		\le& 4\left( {B_F^2L_G^2 + B_G^4L_F^2} \right)\frac{1}{{{A}}}{E[ {\| {{x_k} - x^*} \|_{}^2} ]} +4\left( {B_F^2L_G^2 + B_G^4L_F^2} \right)\frac{1}{{{A^2}}}\sum\limits_{1 \le j \le A}^{} {E[ {\| { \phi _{{{\cal A}_k}[j]}^k-x^*} \|_{}^2} ]},
		\end{align*}
		where  $L_F$, $L_G$, $B_F$ and $B_G$ are the parameters in (\ref{InequationAssumptionF1})- (\ref{InequationAssumptionG3}).  
	\end{lemma}
	\begin{proof}
		Based on Lemma \ref{LammaBoundSAGANormEsimateGradientShort}, we have
		\begin{align*}
		&E[ {\| {{(\partial {{\hat G}_k})^\mathsf{T}}\nabla {F_i}({{\hat G}_k}) - {(\partial G({x_k}))^\mathsf{T}}\nabla {F_i}({G_k}({x_k}))} \|^2} ]\\
		\le& \left( {2B_F^2L_G^2 + 2B_G^4L_F^2} \right)\frac{1}{{{A^2}}}\sum\limits_{1 \le j \le A}^{} {E[ {\| {{x_k} - \phi _{{{\cal A}_k}[j]}^k} \|_{}^2} ]} \\
		\le& 2\left( {2B_F^2L_G^2 + 2B_G^4L_F^2} \right)\frac{1}{{{A}}}{E[ {\| {{x_k} - x^*} \|_{}^2} ]} +2\left( {2B_F^2L_G^2 + 2B_G^4L_F^2} \right)\frac{1}{{{A^2}}}\sum\limits_{1 \le j \le A}^{} {E[ {\| { \phi _{{{\cal A}_k}[j]}^k-x^*} \|_{}^2} ]},
		\end{align*}
		where the last inequality follows from Lemma \ref{RandomVariable2}.	
	\end{proof}
	\begin{lemma}\label{LammaBoundSAGAgradientAndOptimal}
		Assume Assumption \ref{Assumption2} holds, in algorithm \ref{AlgorithmSCDFSAGA}, for the intermediated iteration at $x_k$, $ \hat G$ defined in (\ref{DefinitionSAGAEstimateG}) and  $\partial \hat G_k$ defined in (\ref{DefinitionSAGAEstimateGradientG}) and $\beta_i^*$ is the optimal dual solution, $i\in[n]$, the following bound satisfies,
		\begin{align*}
		E[ {\| {{(\partial {{\hat G}_k})^\mathsf{T}}\nabla {F_i}({{\hat G}_k}) + \beta _i^*} \|_{}^2} ] \le& 2\left( {B_F^2L_G^2\frac{1}{A} + B_G^4L_F^2} \right)E[ {\| {{x_k} - {x^*}} \|_{}^2} ]+ 2B_F^2L_G^2\frac{1}{{{A^2}}}\sum\limits_{1 \le j \le A}^{} {E[ {\| {\phi _{{\mathcal{A}_k}[j]}^k - {x^*}} \|_{}^2} ]},
		\end{align*}
		where  $L_F$, $L_G$, $B_F$ and $B_G$ are the parameters in (\ref{InequationAssumptionF1})- (\ref{InequationAssumptionG3}). 
	\end{lemma}
	\begin{proof}
		Through subtracting and adding ${{{(\partial G({x^*}))}^\mathsf{T}}\nabla {F_i}({{\hat G}_k})}$, and the relationship between $\beta^*$ and $x^*$,  we have
		\begin{align*}
		&E[ {\| {{(\partial {{\hat G}_k})^\mathsf{T}}\nabla {F_i}({{\hat G}_k}) + \beta _i^*} \|_{}^2} ]\\
		=& E[ {\| {{(\partial {{\hat G}_k})^\mathsf{T}}\nabla {F_i}({{\hat G}_k}) - {(\partial G({x^*}))^\mathsf{T}}\nabla {F_i}(G({x^*}))} \|_{}^2} ]\\
		=& E[ {\| {{(\partial {{\hat G}_k})^\mathsf{T}}\nabla {F_i}({{\hat G}_k}) - {(\partial G({x^*}))^\mathsf{T}}\nabla {F_i}({{\hat G}_k}) + {(\partial G({x^*}))^\mathsf{T}}\nabla {F_i}(G({x_k})) - {(\partial G({x^*}))^\mathsf{T}}\nabla {F_i}(G({x^*}))} \|_{}^2} ]\\
		\le& 2E[ {\| {{(\partial {{\hat G}_k})^\mathsf{T}}\nabla {F_i}({{\hat G}_k}) - {(\partial G({x^*}))^\mathsf{T}}\nabla {F_i}({{\hat G}_k})} \|_{}^2} ]+ 2E[ {\| {{(\partial G({x^*}))^\mathsf{T}}\nabla {F_i}(G({x_k})) - {(\partial G({x^*}))^\mathsf{T}}\nabla {F_i}(G({x^*}))} \|_{}^2} ]\\
		\le& 2B_F^2\underbrace {E[ {\| {\partial {{\hat G}_k} - \partial G({x^*})} \|_{}^2} ]}_{(G1)} + 2B_G^2E[ {\| {\nabla {F_i}(G({x_k})) - \nabla {F_i}(G({x^*}))} \|_{}^2} ]\\
		\le& 2B_F^2L_G^2\frac{1}{A}E[ {\| {{x_k} - {x^*}} \|_{}^2} ] + 2B_F^2L_G^2\frac{1}{{{A^2}}}\sum\limits_{1 \le j \le A}^{} {E[ {\| {\phi _{{\mathcal{A}_k}[j]}^k - {x^*}} \|_{}^2} ]}  + 2B_G^2L_F^2B_G^2E[ {\| {{x_k} - {x^*}} \|_{}^2} ]\\
		=& \left( {2B_F^2L_G^2\frac{1}{A} + 2B_G^4L_F^2} \right)E[ {\left\| {{x_k} - {x^*}} \right\|_{}^2} ] + 2B_F^2L_G^2\frac{1}{{{A^2}}}\sum\limits_{1 \le j \le A}^{} {E[ {\| {\phi _{{\mathcal{A}_k}[j]}^k - {x^*}} \|_{}^2} ]} ,
		\end{align*}
		where the first inequality is from the bounded of Jacobian of $G$ in (\ref{InequationAssumptionG1}) and the gradient of $F$ in (\ref{InequationAssumptionF1}), the second inequality is from (G1) and Jacobian bound of $G$ and Lipschitz continuous gradient of $B$. The upper bound of (G1) is derived by subtracting and adding $\frac{1}{A}\sum\nolimits_{ 1\le j \le A} {\partial {G_{{\mathcal{A}_k}[j]}}({x^*})}$, 		
		\begin{align*}
		(G1) =& E[ {\| {\partial {{\hat G}_k} - \partial G({x^*})} \|_{}^2} ]\\
		=& E[ {\| {\partial {{\hat G}_k} - \frac{1}{A}\sum\limits_{1 \le j \le A}^{} {\partial {G_{{\mathcal{A}_k}[j]}}({x^*})}  + \frac{1}{A}\sum\limits_{1 \le j \le A}^{} {\partial {G_{{\mathcal{A}_k}[j]}}({x^*})}  - \partial G({x^*})} \|_{}^2} ]\\
		=& E[ {\| {\frac{1}{A}\sum\limits_{1 \le j \le A}^{} {\partial {G_{{\mathcal{A}_k}[j]}}\left( {{x_k}} \right)}  - \frac{1}{A}\sum\limits_{1 \le j \le A}^{} {\partial {G_{{\mathcal{A}_k}[j]}}({x^*})} } \|_{}^2} ]\\
		&+ E[ {\| {\frac{1}{A}\sum\limits_{1 \le j \le A}^{} {\partial {G_{{\mathcal{A}_k}[j]}}({x^*})}  - \frac{1}{A}\sum\limits_{1 \le j \le A}^{} {\partial {G_{{\mathcal{A}_k}[j]}}(\phi _{{\mathcal{A}_k}[j]}^k)}  + \frac{1}{m}\sum\limits_{j = 1}^m {\partial {G_j}(\phi _j^k)}  - \partial G({x^*})} \|_{}^2} ]\\
		\le& \frac{1}{{{A^2}}}\sum\limits_{1 \le j \le A}^{} {E[ {\left\| {\partial {G_{{\mathcal{A}_k}[j]}}({x_k}) - \partial {G_{{\mathcal{A}_k}[j]}}({x^*})} \right\|_{}^2} ]}  + \frac{1}{{{A^2}}}\sum\limits_{1 \le j \le A}^{} {E[ {\| {\partial {G_{{\mathcal{A}_k}[j]}}({x^*}) - \partial {G_{{\mathcal{A}_k}[j]}}(\phi _{{A_k}[j]}^k)} \|_{}^2} ]} \\
		\le& L_G^2\frac{1}{A}E[ {\| {{x_k} - {x^*}} \|_{}^2} ] + L_G^2\frac{1}{{{A^2}}}\sum\limits_{1 \le j \le A}^{} {E[ {\| {\phi _{{\mathcal{A}_k}[j]}^k - {x^*}} \|_{}^2} ]}, 
		\end{align*}
		where the third equality is based on the expectation on the second term that is equal to zero,
		\begin{align*}
		&E\left[ { {\frac{1}{A}\sum\limits_{1 \le j \le A}^{} {\partial {G_{{\mathcal{A}_k}[j]}}({x^*})}  - \frac{1}{A}\sum\limits_{1 \le j \le A}^{} {\partial {G_{{\mathcal{A}_k}[j]}}(\phi _{{\mathcal{A}_k}[j]}^k)}  + \frac{1}{m}\sum\limits_{j = 1}^m {\partial {G_j}(\phi _j^k)}  - \partial G({x^*})} } \right]\\
		=&E\left[  \frac{1}{A}\sum\limits_{1 \le j \le A}^{} {\partial {G_{{\mathcal{A}_k}[j]}}({x^*})}  - \frac{1}{A}\sum\limits_{1 \le j \le A}^{} {\partial {G_{{\mathcal{A}_k}[j]}}(\phi _{{\mathcal{A}_k}[j]}^k)} \right]-\left( \partial G({x^*})-\frac{1}{m}\sum\limits_{j = 1}^m {\partial {G_j}(\phi _j^k)}  \right) =0
		\end{align*}			
		and first inequalities follow from Lemma \ref{RandomVariable2} and Lemma \ref{RandomVariable1}, the last inequality is based on the bounded Jacobian of $G$ in (\ref{InequationAssumptionG1}) and Lipschitz continuous gradient of $G$ in (\ref{InequationAssumptionG3}).
	\end{proof}
	\begin{lemma}\label{LammaBoundSAGAgradientAndOptimalConvex}
		Assume Assumption \ref{Assumption2} and Assumption \ref{Assumption3} hold, in algorithm \ref{AlgorithmSCDFSAGA}, for the intermediated iteration at $x_k$, $ \hat G$ defined in (\ref{DefinitionSAGAEstimateG}) and  $\partial \hat G_k$ defined in (\ref{DefinitionSAGAEstimateGradientG}) and $\beta_i^*$ is the optimal dual solution, $i\in[n]$, the following bound satisfies,
		\begin{align*}
		&E[ {\| {{(\partial {{\hat G}_k})^\mathsf{T}}\nabla {F_i}({{\hat G}_k}) + \beta _i^*} \|_{}^2} ]\\
		\le&4\left( {B_F^2L_G^2 + B_G^4L_F^2} \right)\frac{1}{{{A}}}{E[ {\| {{x_k} - x^*} \|_{}^2} ]} +4\left( {B_F^2L_G^2 + B_G^4L_F^2} \right)\frac{1}{{{A^2}}}\sum\limits_{1 \le j \le A}^{} {E[ {\| { \phi _{{{\cal A}_k}[j]}^k-x^*} \|_{}^2} ]}\\
		&+ 4L_f\left( {P(x) - P({x^*}) - \frac{\lambda }{2}{{\| {x - {x^*}} \|}^2}} \right),
		\end{align*}
		where  $L_F$, $L_G$, $B_F$ and $B_G$ are the parameters in (\ref{InequationAssumptionF1})- (\ref{InequationAssumptionG3}). 
		\begin{proof}
			Through subtracting and adding ${{{(\partial G({x_k}))}^\mathsf{T}}\nabla {F_i}(G(x_k))}$, we have
			\begin{align*}
			&E[ {\| {{(\partial {{\hat G}_k})^\mathsf{T}}\nabla {F_i}({{\hat G}_k}) + \beta _i^*} \|_{}^2} ]\\
			=& E[ {\| {{(\partial {{\hat G}_k})^\mathsf{T}}\nabla {F_i}({{\hat G}_k}) - {(\partial G({x^*}))^\mathsf{T}}\nabla {F_i}(G({x^*}))} \|_{}^2} ]\\
			=& E[ {\| {{(\partial {{\hat G}_k})^\mathsf{T}}\nabla {F_i}({{\hat G}_k}) - {{(\partial G({x_k}))}^\mathsf{T}}\nabla {F_i}({{ G}(x_k)}) +  {(\partial G({x_k}))^\mathsf{T}}\nabla {F_i}({{ G}(x_k)})  - {(\partial G({x^*}))^\mathsf{T}}\nabla {F_i}(G({x^*}))} \|_{}^2} ]\\
			\le& 2\underbrace {E[{{\| {{(\partial {{\hat G}_k})^\mathsf{T}}\nabla {F_i}({{\hat G}_k}) - {(\partial G({x_k}))^\mathsf{T}}\nabla {F_i}(G({x_k}))} \|^2}}]}_{(b)} + 2\underbrace {E[{{\| {{{(\partial G({x_k}))}^T}\nabla {F_i}(G({x_k})) - {(\partial G({x^*}))^\mathsf{T}}\nabla {F_i}(G({x^*}))} \|^2}}]}_{(a)}\\
			\le&2\left( {2B_F^2L_G^2 + 2B_G^4L_F^2} \right)\frac{1}{{{A}}}{E[ {\| {{x_k} - x^*} \|_{}^2} ]} +2\left( {2B_F^2L_G^2 + 2B_G^4L_F^2} \right)\frac{1}{{{A^2}}}\sum\limits_{1 \le j \le A}^{} {E[ {\| { \phi _{{{\cal A}_k}[j]}^k-x^*} \|_{}^2} ]}\\
			&+ 4L_f\left( {P(x) - P({x^*}) - \frac{\lambda }{2}{{\| {x - {x^*}} \|}^2}} \right),
			\end{align*}
			where the first inequality follow from Lemma \ref{RandomVariable2}, and the upper bound of (a) and (b) are based on Lemma  \ref{LammaBoundSAGANormEsimateGradient} and Lemma \ref{LemmaAppendix}. 		
		\end{proof}
	\end{lemma}

	\subsection{Bound the difference of variable and the optimal solution}
	\begin{lemma} \label{LemmaBoundMainA}
		Suppose Assumption \ref{Assumption1} and \ref{Assumption2} hold, in algorithm\ref{AlgorithmSCDFSAGA}, for the intermediated iteration at $x_k$, 
		let ${A_k} = \| {{x_k} - {x^*}} \|^2$, define $\lambda {R_x} = {\max _x}\{ {{{\| {{x^*} - x} \|}^2}:F(G(x)) \le F(G({x_0}))} \}$, we have
		\begin{align*}
		E[ {{A_{k + 1}}} ] - E[ {{A_k}} ]  \le & 8\eta\lambda {R_x}( {B_F^2L_G^2 + B_G^4L_F^2} )\frac{1}{{{A}}} E[ {{A_k}} ]  +  8\eta\lambda {R_x}( {B_F^2L_G^2 + B_G^4L_F^2} )\frac{1}{{{A^2}}}\sum\limits_{1 \le j \le A}^{} {E[\| {\phi _{{\mathcal{A}_k}[j]}^k - {x^*}} \|_{}^2]} 
		\\ &  -  2\eta\lambda  E[ {{A_k}} ]+ {\eta ^2}E[ {\| {{( {\partial {{\hat G}_k}} )^\mathsf{T}}\nabla {F_i}( {{{\hat G}_k}( {{x_k}} )} ) + \beta _i^k} \|^2} ],
		\end{align*}
		where  $B_F$, $L_F$, $B_G$ and $L_G$ are the parameters in (\ref{InequationAssumptionF1}) to (\ref{InequationAssumptionG3}) and  $p>0$.
	\end{lemma}
	\begin{proof}
		Let ${A_k} = \left\| {{x_k} - {x^*}} \right\|^2$, we obtain
		\begin{align*}
		{A_{k + 1}} =& \| {{x_{k + 1}} - {x^*}} \|^2\\
		=& \| {{x_k} - \eta ( {{( {\partial {{\hat G}_k}} )^\mathsf{T}}\nabla {F_i}( {{{\hat G}_k}} ) + \beta _i^k} ) - {x^*}} \|^2\\
		=& \| {{x_k} - {x^*}} \|^2 - 2\eta \langle {( {{( {\partial {{\hat G}_k}} )^\mathsf{T}}\nabla {F_i}( {{{\hat G}_k}} ) + \beta _i^k} ),{x_k} - {x^*}} \rangle  + {\eta ^2}\| {{( {\partial {{\hat G}_k}} )^\mathsf{T}}\nabla {F_i}( {{{\hat G}_k}} ) + \beta _i^k} \|^2.
		\end{align*}
		Taking expectation on above both sides, we have
		\begin{align*}
		&E[ {{A_{k + 1}}} ] - E[ {{A_k}} ]\\
		= & - 2\eta \underbrace {E[ {\langle {( {{( {\partial {{\hat G}_k}} )^\mathsf{T}}\nabla {F_i}( {{{\hat G}_k}} ) + \beta _i^k} ),{x_k} - {x^*}} \rangle } ]}_{( {A1} )} + {\eta ^2}E[ {\| {{( {\partial {{\hat G}_k}} )^\mathsf{T}}\nabla {F_i}( {{{\hat G}_k}} ) + \beta _i^k} \|^2} ]\\
		\le &  8\eta\lambda {R_x}\left( {B_F^2L_G^2 + B_G^4L_F^2} \right)\frac{1}{{{A}}}E[\left\| {{x_k} - {x^*}} \right\|_{}^2]  +  8\eta\lambda {R_x}\left( {B_F^2L_G^2 + B_G^4L_F^2} \right)\frac{1}{{{A^2}}}\sum\limits_{1 \le j \le A}^{} {E[\| {\phi _{{\mathcal{A}_k}[j]}^k - {x^*}} \|_{}^2]} 
		\\ &  - 2\eta\lambda E[ {\| {{x^*} - {x_k}} \|^2} ]+ {\eta ^2}E[ {\| {{( {\partial {{\hat G}_k}} )^\mathsf{T}}\nabla {F_i}( {{{\hat G}_k}( {{x_k}} )} ) + \beta _i^k} \|^2} ],
		\end{align*}
		where (A1) follows from Lemma \ref{LemmaBoundSAGAEstimatedGradientWithvariance}.
	\end{proof}
	\begin{lemma} \label{LemmaSAGABoundMainAconvex}
		Suppose Assumption \ref{Assumption1}, \ref{Assumption2} and \ref{Assumption3} hold, in algorithm\ref{AlgorithmSCDFSAGA}, for the intermediated iteration at $x_k$, 
		let ${A_k} = \| {{x_k} - {x^*}} \|^2$, we have
		\begin{align*}
		E[ {{A_{k + 1}}} ] - E[ {{A_k}} ]  \le&  8\eta\lambda {R_x}\left( {B_F^2L_G^2 + B_G^4L_F^2} \right)\frac{1}{{{A}}} E[ {{A_k}} ] +  8\eta\lambda {R_x}\left( {B_F^2L_G^2 + B_G^4L_F^2} \right)\frac{1}{{{A^2}}}\sum\limits_{1 \le j \le A}^{} {E[\| {\phi _{{\mathcal{A}_k}[j]}^k - {x^*}} \|_{}^2]} 
		\\ &  -2(1-d)\eta( P(x_k)-P(x^*))  -  (1+d)\lambda\eta E\| {{x^*} - {x_k}} \|^2+ {\eta ^2}E[ {\| {{( {\partial {{\hat G}_k}} )^\mathsf{T}}\nabla {F_i}( {{{\hat G}_k}( {{x_k}} )} ) + \beta _i^k} \|^2} ],
		\end{align*}
		where  $B_F$, $L_F$, $B_G$ and $L_G$ are the parameters in (\ref{InequationAssumptionF1}) to (\ref{InequationAssumptionG3}) and  $1>d\ge0$.
	\end{lemma}
	\begin{proof}
		The beginning of the proof is the same as the proof of Lemma \ref{LemmaBoundMainA},
		\begin{align*}
		&E[ {{A_{k + 1}}} ] - E[ {{A_k}} ]\\
		= & \underbrace { - 2\eta E[\langle ({(\partial {{\hat G}_k})^\mathsf{T}}\nabla {F_i}({{\hat G}_k}) + \beta _i^k),{x_k} - {x^*}\rangle ]}_{A2} + {\eta ^2}E[ {\| {{( {\partial {{\hat G}_k}} )^\mathsf{T}}\nabla {F_i}( {{{\hat G}_k}} ) + \beta _i^k} \|^2} ]\\
		\le &  8\eta\lambda {R_x}\left( {B_F^2L_G^2 + B_G^4L_F^2} \right)\frac{1}{{{A}}}E[\left\| {{x_k} - {x^*}} \right\|_{}^2]  +  8\eta\lambda {R_x}\left( {B_F^2L_G^2 + B_G^4L_F^2} \right)\frac{1}{{{A^2}}}\sum\limits_{1 \le j \le A}^{} {E[\| {\phi _{{\mathcal{A}_k}[j]}^k - {x^*}} \|_{}^2]} 
		\\ & -2(1-d)\eta( P(x_k)-P(x^*))  -  (1+d)\lambda\eta E\| {{x^*} - {x_k}} \|^2+ {\eta ^2}E[ {\| {{( {\partial {{\hat G}_k}} )^\mathsf{T}}\nabla {F_i}( {{{\hat G}_k}( {{x_k}} )} ) + \beta _i^k} \|^2} ],
		\end{align*}
		where (A2) follows from Lemma \ref{LemmaBoundSAGAEstimatedGradientWithvarianceConvex}.
	\end{proof}
	\begin{lemma}\label{LemmaBoundMainB}
		Suppose Assumption \ref{Assumption2} holds, in algorithm \ref{AlgorithmSCDFSAGA}, for the intermediated iteration at $x_k$ and $\beta^k$, let ${B_k} = \frac{1}{n}\sum\nolimits_{i = 1}^n {\| {\beta _i^k - \beta _i^*} \|^2}  $, then we have
		\begin{align*}
		&E\left[ {{B_{k + 1}}} \right] - E\left[ {{B_k}} \right]\\
		\le&  - \lambda \eta E\left[ {{B_k}} \right] + 2\lambda \eta \left( {B_F^2L_G^2\frac{1}{A} +B_G^4L_F^2} \right)E[ {\| {{x_k} - {x^*}} \|^2} ] + 2\lambda \eta B_F^2L_G^2\frac{1}{{{A^2}}}\sum\limits_{1 \le j \le A}^{} {E[ {\| {\phi _{{\mathcal{A}_k}[j]}^k - {x^*}} \|^2} ]} \\
		&- ( {1 - \lambda n\eta } )\lambda \eta \| {{(\partial {{\hat G}_k})^\mathsf{T}}\nabla {F_i}({{\hat G}_k}) + \beta _i^k} \|^2,
		\end{align*}
		where $B_F$, $L_F$, $B_G$ and $L_G$ are the parameters in (\ref{InequationAssumptionF1}) to (\ref{InequationAssumptionG3}).
	\end{lemma}
	\begin{proof}
		In algorithm \ref{AlgorithmSCDFSAGA}, for the intermediated iteration at $\beta^k$, based on the definition of $B_k$ and  update for  $\beta^{k+1}_i$, $i\in[n]$, we get
		\begin{align*}
		{B_{k + 1}} - {B_k} =& \frac{1}{n}\sum\limits_{i = 1}^n {\| {\beta _i^{k + 1} - \beta _i^*} \|^2}  - \frac{1}{n}\sum\limits_{i = 1}^n {\| {\beta _i^k - \beta _i^*} \|^2} \\
		=& \frac{1}{n}\| {\beta _i^{k + 1} - \beta _i^*} \|^2 - \frac{1}{n}\| {\beta _i^k - \beta _i^*} \|^2\\
		=& \frac{1}{n}\underbrace {\| {\beta _i^k - \lambda n\eta ( {{( {\partial {{\hat G}_k}} )^\mathsf{T}}\nabla {F_i}( {{{\hat G}_k}} ) + \beta _i^k} ) - \beta _i^*} \|_2^2}_{B1} - \frac{1}{n}\| {\beta _i^k - \beta _i^*} \|^2.
		\end{align*}
		Based on the strongly convex property in Definition (\ref{DefinitionConvex}),  $\left\| {ax + \left( {1 - a} \right)y} \right\|^2 = a\left\| x \right\|^2 + \left( {1 - a} \right)\left\| y \right\|^2 - a\left( {1 - a} \right)\left\| {x - y} \right\|^2$, $(0\le a \le 1)$, (B1) can be expressed as 
		\begin{align*}
		(B1) =& \| {\beta _i^k - \lambda n\eta ( {{(\partial {{\hat G}_k})^\mathsf{T}}\nabla {F_i}({{\hat G}_k}) + \beta _i^k} ) - \beta _i^*} \|_{}^2\\
		=& \| {\beta _i^k - \lambda n\eta ( {{(\partial {{\hat G}_k})^\mathsf{T}}\nabla {F_i}({{\hat G}_k}) + \beta _i^* + \beta _i^k - \beta _i^*} ) - \beta _i^*} \|_{}^2\\
		=& \| {( {1 - \lambda n\eta } )( {\beta _i^k - \beta _i^*} ) - \lambda n\eta ( {{(\partial {{\hat G}_k})^\mathsf{T}}\nabla {F_i}({{\hat G}_k}) + \beta _i^*} )} \|_{}^2\\
		=& \| {( {1 - \lambda n\eta } )( {\beta _i^k - \beta _i^*} ) + \lambda n\eta ( { - {(\partial {{\hat G}_k})^\mathsf{T}}\nabla {F_i}({{\hat G}_k}) - \beta _i^*} )} \|_{}^2\\
		=& ( {1 - \lambda n\eta } )\| {\beta _i^k - \beta _i^*} \|_{}^2 + \lambda n\eta \| {{(\partial {{\hat G}_k})^\mathsf{T}}\nabla {F_i}({{\hat G}_k}) + \beta _i^*} \|_{}^2 - ( {1 - \lambda n\eta } )\lambda n\eta \| {{(\partial {{\hat G}_k})^\mathsf{T}}\nabla {F_i}({{\hat G}_k}) + \beta _i^k} \|_{}^2.
		\end{align*}
		Taking expectation on both sides of above equality, we get,
		\begin{align*}
		&E[ {{B_{k + 1}}} ] - E[ {{B_k}} ]\\
		=&  - \lambda \eta E[ {{B_k}} ] + \lambda \eta \underbrace {\| {{(\partial {{\hat G}_k})^\mathsf{T}}\nabla {F_i}({{\hat G}_k}) + \beta _i^*} \|^2}_{( {B2} )} - ( {1 - \lambda n\eta } )\lambda \eta \| {{(\partial {{\hat G}_k})^\mathsf{T}}\nabla {F_i}({{\hat G}_k}) + \beta _i^k} \|^2 \\
		\le&  - \lambda \eta E\left[ {{B_k}} \right] + 2\lambda \eta \left( {B_F^2L_G^2\frac{1}{A} +B_G^4L_F^2} \right)E[ {\| {{x_k} - {x^*}} \|^2} ] + 2\lambda \eta B_F^2L_G^2\frac{1}{{{A^2}}}\sum\limits_{1 \le j \le A}^{} {E[ {\| {\phi _{{\mathcal{A}_k}[j]}^k - {x^*}} \|^2} ]} \\
		&- ( {1 - \lambda n\eta } )\lambda \eta \| {{(\partial {{\hat G}_k})^\mathsf{T}}\nabla {F_i}({{\hat G}_k}) + \beta _i^k} \|^2,
		\end{align*}
		where (B2) follows from Lemma \ref{LammaBoundSAGAgradientAndOptimal}.
	\end{proof}
	\begin{lemma}\label{LemmaBoundSAGAMainBConvex}
		Suppose Assumption \ref{Assumption2} and Assumption \ref{Assumption3} hold, in algorithm \ref{AlgorithmSCDFSAGA}, for the intermediated iteration at $x_k$ and $\beta^k$, let ${B_k} = \frac{1}{n}\sum\nolimits_{i = 1}^n {\| {\beta _i^k - \beta _i^*} \|^2}  $, then we have
		\begin{align*}
		&E\left[ {{B_{k + 1}}} \right] - E\left[ {{B_k}} \right]\\
		\le&   -\lambda \eta E\left[ {{B_k}} \right] +4\lambda \eta\left( {B_F^2L_G^2 + B_G^4L_F^2} \right)\frac{1}{{{A}}}{E[ {\| {{x_k} - x^*} \|_{}^2} ]} +4\lambda \eta\left( {B_F^2L_G^2 + B_G^4L_F^2} \right)\frac{1}{{{A^2}}}\sum\limits_{1 \le j \le A}^{} {E[ {\| { \phi _{{{\cal A}_k}[j]}^k-x^*} \|_{}^2} ]}\\
		&+ 4L_f\lambda \eta( P(x_k) - P({x^*}))  - ( {1 - \lambda n\eta } )\lambda \eta \| {{(\partial {{\hat G}_k})^\mathsf{T}}\nabla {F_i}({{\hat G}_k}) + \beta _i^k} \|^2,
		\end{align*}
		where $B_F$, $L_F$, $B_G$ and $L_G$ are the parameters in (\ref{InequationAssumptionF1}) to (\ref{InequationAssumptionG3}).
	\end{lemma}
	\begin{proof}
		The beginning proof is the same as Lemma \ref{LemmaBoundMainB}
		\begin{align*}
		&E[ {{B_{k + 1}}} ] - E[ {{B_k}} ]\\
		=&  - \lambda \eta E[ {{B_k}} ] + \lambda \eta \underbrace {\| {{(\partial {{\hat G}_k})^\mathsf{T}}\nabla {F_i}({{\hat G}_k}) + \beta _i^*} \|^2}_{( {B3} )} - ( {1 - \lambda n\eta } )\lambda \eta \| {{(\partial {{\hat G}_k})^\mathsf{T}}\nabla {F_i}({{\hat G}_k}) + \beta _i^k} \|^2\\
		\le&  - \lambda \eta E\left[ {{B_k}} \right] +4\lambda \eta\left( {B_F^2L_G^2 + B_G^4L_F^2} \right)\frac{1}{{{A}}}{E[ {\| {{x_k} - x^*} \|_{}^2} ]} +4\lambda \eta\left( {B_F^2L_G^2 + B_G^4L_F^2} \right)\frac{1}{{{A^2}}}\sum\limits_{1 \le j \le A}^{} {E[ {\| { \phi _{{{\cal A}_k}[j]}^k-x^*} \|_{}^2} ]}\\
		&+ 4L_f\lambda \eta\left( {P(x_k) - P({x^*}) - \frac{\lambda }{2}{{\| {x_k - {x^*}} \|}^2}} \right)- ( {1 - \lambda n\eta } )\lambda \eta \| {{(\partial {{\hat G}_k})^\mathsf{T}}\nabla {F_i}({{\hat G}_k}) + \beta _i^k} \|^2\\
		\le&  - \lambda \eta E\left[ {{B_k}} \right] +4\lambda \eta\left( {B_F^2L_G^2 + B_G^4L_F^2} \right)\frac{1}{{{A}}}{E[ {\| {{x_k} - x^*} \|_{}^2} ]} +4\lambda \eta\left( {B_F^2L_G^2 + B_G^4L_F^2} \right)\frac{1}{{{A^2}}}\sum\limits_{1 \le j \le A}^{} {E[ {\| { \phi _{{{\cal A}_k}[j]}^k-x^*} \|_{}^2} ]}\\
		&+ 4L_f\lambda \eta\left( {P(x_k) - P({x^*}) } \right)- ( {1 - \lambda n\eta } )\lambda \eta \| {{(\partial {{\hat G}_k})^\mathsf{T}}\nabla {F_i}({{\hat G}_k}) + \beta _i^k} \|^2
		\end{align*}
		where (B3) follows from Lemma \ref{LammaBoundSAGAgradientAndOptimalConvex}.
	\end{proof}	
	
	
	\begin{lemma}\label{LemmaBoundMainC}
		In algorithm\ref{AlgorithmSCDFSAGA}, for the intermediated iteration at $x_k$, let ${C_{k}} = \frac{1}{m}\sum\nolimits_{j = 1}^m {\| {\phi _j^k - {x^*}} \|^2} $ and ${A_k} =  {\| {{x^k} - {x^*}} \|^2} $, then we have
		\begin{align*}
		E\left[ {{C_{k + 1}}} \right] - E\left[ {{C_k}} \right] =  - \frac{A}{n}E\left[ {{C_k}} \right] + \frac{A}{n}E\left[ {{A_k}} \right],
		\end{align*}
		where $A$ is the number of sample times for forming the mini-batch $ {{\mathcal{A}_k}}$.
	\end{lemma}
	\begin{proof}
		In algorithm             \ref{AlgorithmSCDFSAGA}, at the intermediated iteration at $x_k$, for $j\in {{\mathcal{A}_k}}$, ${\phi _{_{{\mathcal{A}_k}[j]}}^{k + 1}}=x_k$, thus, we have
		\begin{align*}
		{C_{k + 1}} - {C_k}	=& \frac{1}{n}\left( {\sum\limits_{1 \le j \le A}^{} {( {E[ {\| {\phi _{_{{\mathcal{A}_k}[j]}}^{k + 1} - {x^*}} \|^2} ] - \| {\phi _{\mathcal{A}_k[j]}^k - {x^*}} \|^2} )} } \right)\\
		=& \frac{A}{n}\| {{x^k} - {x^*}} \|^2 -\frac{1}{n}\sum\limits_{1 \le j \le A} {{{\| {\phi _{{\mathcal{A}_k[j]}}^k - {x^*}} \|}^2}} .
		\end{align*}
		Taking expectation on both sides,
		\begin{align*}
		E[ {{C_{k + 1}}} ] - E[ {{C_k}} ] =  - \frac{A}{n}E[ {{C_k}} ] + \frac{A}{n}E[ {{A_k}} ],
		\end{align*}
		where $\frac{1}{n}\sum\nolimits_{1 \le j \le A} {E[\| {\phi _{{\mathcal{A}_k}[j]}^k - {x^*}} \|_{}^2]}  = \frac{1}{n}\sum\nolimits_{1 \le j \le A} {\frac{1}{m}\sum\nolimits_{j = 1}^m {\| {\phi _j^k - {x^*}} \|_{}^2} }  = \frac{A}{n}{C_k}$.
	\end{proof}
	\begin{lemma}\label{LemmaBoundSAGAEstimatedGradientWithvariance} Assume Assumption \ref{Assumption1} and \ref{Assumption2} hold, in algorithm \ref{AlgorithmSCDFSAGA}. Define $\lambda {R_x} = {\max _x}\{ {{{\| {{x^*} - x} \|}^2}:F(G(x)) \le F(G({x_0}))} \}$. The bound satisfies,
		\begin{align*}
		&-E[ {\langle {( {{( {\partial {{\hat G}_k}} )^\mathsf{T}}\nabla {F_i}( {{{\hat G}_k}} ) + \beta _i^k} ),{x_k} - {x^*}} \rangle } ]\\
		\le & \lambda {R_x}\left( {2B_F^2L_G^2 + 2B_G^4L_F^2} \right)\frac{1}{{{A}}}E[\left\| {{x_k} - {x^*}} \right\|_{}^2]  + \lambda {R_x}\left( {2B_F^2L_G^2 + 2B_G^4L_F^2} \right)\frac{1}{{{A^2}}}\sum\limits_{1 \le j \le A}^{} {E[\| {\phi _{{\mathcal{A}_k}[j]}^k - {x^*}} \|_{}^2]} 
		\\ &  - \lambda E[ {\| {{x^*} - {x_k}} \|^2} ],
		\end{align*}
		where  $B_F$, $L_F$, $B_G$ and $L_G$ are the parameters in (\ref{InequationAssumptionF1}) to (\ref{InequationAssumptionG3}), $p>0$.
	\end{lemma}
	\begin{proof}Through subtracting and adding term ${{( {\partial {G_k}} )}^\mathsf{T}}\nabla {F_i}( {{G}( {{x_k}} )} )$, we have
		\begin{align*}
		&- E[ {\langle {( {{( {\partial {{\hat G}_k}} )^\mathsf{T}}\nabla {F_i}( {{{\hat G}_k}} ) + \beta _i^k} ),{x_k} - {x^*}} \rangle } ]\\
		=& \langle {E[ {{( {\partial {{\hat G}_k}} )^\mathsf{T}}\nabla {F_i}( {{{\hat G}_k}} ) + \beta _i^k} ],{x^*} - {x_k}} \rangle \\
		=& \langle {E[ {{( {\partial {{\hat G}_k}} )^\mathsf{T}}\nabla {F_i}( {{{\hat G}_k}} ) - {( {\partial {G_k}} )^\mathsf{T}}\nabla {F_i}( {{G}( {{x_k}} )} ) + {( {\partial {G_k}} )^\mathsf{T}}\nabla {F_i}( {{G}( {{x_k}} )} ) + \beta _i^k} ],{x^*} - {x_k}} \rangle \\
		=& \langle {E[ {{( {\partial {{\hat G}_k}} )^\mathsf{T}}\nabla {F_i}( {{{\hat G}_k}} ) - {{( {\partial {G_k}} )}^\mathsf{T}}\nabla {F_i}( {{G}( {{x_k}} )} )} ],{x^*} - {x_k}} \rangle  + \underbrace {\langle {E[{( {\partial G( {{x_k}} )} )^\mathsf{T}}\nabla F_i( {{G}( {{x_k}} )} ) + \beta _i^k],{x^*} - {x_k}} \rangle }_{( {A3} )}\\
		\le& E[{\| {{(\partial {{\hat G}_k})^\mathsf{T}}\nabla {F_i}({{\hat G}_k}) - {(\partial {G_k})^\mathsf{T}}\nabla {F_i}(G({x_k}))} \|^2}]{\| {{x^*} - {x_k}} \|^2} - \lambda {\| {{x^*} - {x_k}} \|^2}\\
		\le & 4\lambda {R_x}\left( {B_F^2L_G^2 + B_G^4L_F^2} \right)\frac{1}{{{A}}}E[\left\| {{x_k} - {x^*}} \right\|_{}^2]  + 4\lambda {R_x}\left( {B_F^2L_G^2 + B_G^4L_F^2} \right)\frac{1}{{{A^2}}}\sum\limits_{1 \le j \le A}^{} {E[\| {\phi _{{\mathcal{A}_k}[j]}^k - {x^*}} \|_{}^2]} 
		\\ &  - \lambda E[ {\| {{x^*} - {x_k}} \|^2} ],
		\end{align*}
		where the first inequality follows from Cauchy-Schwarz inequality and Lemma \ref{LemmaSAGABoundGradientAndVariable} (A3), the second inequality follows from Lemma \ref{LammaBoundSAGANormEsimateGradient}. 
	\end{proof}
	\begin{lemma}\label{LemmaBoundSAGAEstimatedGradientWithvarianceConvex} Assume Assumption \ref{Assumption1}, \ref{Assumption2} and  \ref{Assumption3} hold, in algorithm \ref{AlgorithmSCDFSAGA},  the bound satisfies,
		\begin{align*}
		&-E[ {\langle {( {{( {\partial {{\hat G}_k}} )^\mathsf{T}}\nabla {F_i}( {{{\hat G}_k}} ) + \beta _i^k} ),{x_k} - {x^*}} \rangle } ]\\
		\le & 4\lambda {R_x}\left( {B_F^2L_G^2 + B_G^4L_F^2} \right)\frac{1}{{{A}}}E[\left\| {{x_k} - {x^*}} \right\|_{}^2]  + 4\lambda {R_x}\left( {B_F^2L_G^2 + B_G^4L_F^2} \right)\frac{1}{{{A^2}}}\sum\limits_{1 \le j \le A}^{} {E[\| {\phi _{{\mathcal{A}_k}[j]}^k - {x^*}} \|_{}^2]}\\&   
		-(1-d)( P(x_k)-P(x^*))  - \frac{1}{2}\lambda (1+d)\| {{x^*} - {x_k}} \|^2, 
		\end{align*}
		where  $B_F$, $L_F$, $B_G$ and $L_G$ are the parameters in (\ref{InequationAssumptionF1}) to (\ref{InequationAssumptionG3}), $p>0$,  and $A$ is the number of sample times for forming the mini-batch $ {{\mathcal{A}_k}}$.
	\end{lemma}
	\begin{proof}The beginning proof is the same as the Lemma \ref{LemmaBoundSAGAEstimatedGradientWithvariance}
		\begin{align*}
		&- E[ {\langle {( {{( {\partial {{\hat G}_k}} )^\mathsf{T}}\nabla {F_i}( {{{\hat G}_k}} ) + \beta _i^k} ),{x_k} - {x^*}} \rangle } ]\\
		\le & 4\lambda {R_x}\left( {B_F^2L_G^2 + B_G^4L_F^2} \right)\frac{1}{{{A}}}E[\left\| {{x_k} - {x^*}} \right\|_{}^2]  +4 \lambda {R_x}\left( {B_F^2L_G^2 + B_G^4L_F^2} \right)\frac{1}{{{A^2}}}\sum\limits_{1 \le j \le A}^{} {E[\| {\phi _{{\mathcal{A}_k}[j]}^k - {x^*}} \|_{}^2]}\\&   -( P(x_k)-P(x^*))  - \frac{1}{2}\lambda \| {{x^*} - {x_k}} \|^2\\
		\le & 4\lambda {R_x}\left( {B_F^2L_G^2 + B_G^4L_F^2} \right)\frac{1}{{{A}}}E[\left\| {{x_k} - {x^*}} \right\|_{}^2]  + 4\lambda {R_x}\left( {B_F^2L_G^2 + B_G^4L_F^2} \right)\frac{1}{{{A^2}}}\sum\limits_{1 \le j \le A}^{} {E[\| {\phi _{{\mathcal{A}_k}[j]}^k - {x^*}} \|_{}^2]}\\&   
		-(1-d)( P(x_k)-P(x^*))  - \frac{1}{2}\lambda (1+d)\| {{x^*} - {x_k}} \|^2,
		\end{align*}
		where  $1> d\ge0$, and the last inequality based on $d( {P( {{x_k}} ) - P( {{x^*}} )} ) \ge \frac{1}{2}\lambda d{\| {{x_k} - {x^*}} \|^2}$.
	\end{proof}
	\begin{lemma}\label{LemmaSAGABoundGradientAndVariable}
		In algorithm \ref{AlgorithmSCDFSAGA}, suppose $P(x)$ is $\lambda$-strongly convex, for the intermediated iteration at $x_k$, the bound satisfies,
		\begin{align*}
		E[\langle {(\partial G({x_k}))^\mathsf{T}}\nabla F_i(G({x_k})) + \lambda {x_k},{x^*} - {x_k}\rangle ] \le  - \lambda E[ {\| {{x_k} - {x^*}} \|^2} ].
		\end{align*}
	\end{lemma}
	\begin{proof}	
		Based on the $\lambda$-strongly convexity of function $P(x)$, we have
		\begin{align*}
		E[{(\partial G({x_k}))^\mathsf{T}}\nabla {F_i}(G({x_k})) + \beta _i^k] =& \langle {(\partial G({x_k}))^\mathsf{T}}\nabla F(G({x_k})) + \lambda {x_k},{x^*} - {x_k}\rangle \\
		=& \langle {\nabla P( {{x_k}} ),{x^*} - {x_k}} \rangle \\
		\le& P( {{x^*}} ) - P( {{x_k}} ) - \frac{\lambda }{2}\| {{x^*} - {x_k}} \|^2\\
		\le&  - \lambda \| {{x^*} - {x_k}} \|^2,
		\end{align*}	
		where $E[ {\beta _i^k} ] = \lambda {x_k}$.
	\end{proof}


\end{document}